\documentclass[final,5p,times,twocolumn]{elsarticle}
\usepackage[dvips,all,color]{xy}
\usepackage[footnotesize]{subfigure}
\usepackage[cmex10]{amsmath}
\usepackage{amsthm}
\usepackage{amssymb}
\usepackage[dvips,x11names]{xcolor}
\usepackage{dsfont}
\usepackage[linesnumberedhidden,ruled,vlined,longend]{algorithm2e}
\usepackage{mathrsfs}
\usepackage{txfonts}
\usepackage{times}
\usepackage{stmaryrd}
\usepackage{flushend}
\usepackage{psfrag}
\DeclareMathAlphabet{\mathpzc}{OT1}{pzc}{m}{it}

\newtheorem{cor}{Corollary}
\newtheorem{lem}{Lemma}
\newtheorem{prop}{Proposition}
\newtheorem{defn}{Definition}

\newtheorem{rem}{Remark}
\newtheorem{notn}{Notation}[section]

\newcommand{\C}{\mathcal{C}}
\newcommand{\nustar}{\boldsymbol{\nu^\star}}
\newcommand{\Gn}{\mathds{G}_{\textrm{\sffamily\textsc{Nav}}}}
\newcommand{\Gnm}{\mathds{G}^{\textit{MOD}}_{\textrm{\sffamily\textsc{Nav}}}}

\newcommand{\Gl}{\textrm{\sffamily\textsc{Goal}}}
\newcommand{\mm}{\nu_{\#}}

\newcommand{\obs}{\textrm{\sffamily\textsc{Obstacle}}}

\newcommand{\Crd}{\textrm{\sffamily\textsc{Card}}}

\newcommand{\eg}{\geqq_{\textbf{\texttt{(Elementwise)}}}}
\newcommand{\egg}{>_{\textbf{\texttt{(Elementwise)}}}}

\newcommand{\Q}{\mathscr{P}}

\newcommand{\myar}{\ar@[|(2.5)]}
\newcommand{\myarT}{\ar@[|(3.5)]}
\newcommand{\myarL}{\ar@[|(1.5)]}

\newcommand{\red}{\color{red}}

\xyoption{arc}
\newcommand{\Mblue}{\color{DeepSkyBlue4}}

\newcommand{\DGreen}{\color{Green4}}
\newcommand{\BRed}{\color{OrangeRed3}}
\newcommand{\DST}{\mathds{D}}

\begin{document}
\begin{frontmatter}
 \title{Formal-language-theoretic Optimal Path Planning  \\
For Accommodation of  Amortized  Uncertainties and  Dynamic Effects\tnoteref{thnks}}
\tnotetext[thanks]{This work has been supported in part by the U.S. Army Research Laboratory 
and the U.S. Army Research Office under Grant No. W911NF-07-1-0376 and by the Office of Naval Research under Grant No. N00014-08-1-380}
\author[psu]{I. Chattopadhyay \corref{cor1}}
\author[psu]{A. Cascone}
\author[psu]{ A. Ray}

\address[psu]{Department of Mechanical Engineering, The Pennsylvnaia State University, University Park, 16802}
\cortext[cor1]{Corresponding Author}
\begin{abstract}
We report  a 
globally-optimal approach to robotic path planning under uncertainty, based 
on the theory of quantitative measures of formal languages. 
A significant generalization to the language-measure-theoretic path planning algorithm $\nustar$ is presented that explicitly
accounts for average dynamic uncertainties and estimation errors in plan execution. 
The notion of the navigation automaton is generalized to include probabilistic  uncontrollable transitions, which account for
 uncertainties by modeling and planning for   probabilistic  deviations from the computed policy in the course of execution. 
The planning problem is solved by casting it in the form of a performance maximization problem for 
probabilistic finite state automata. In essence we  solve the following optimization problem: Compute the navigation policy which maximizes the probability of reaching the goal, while simultaneously
minimizing the probability of hitting an obstacle.
Key novelties of the proposed approach include the modeling of uncertainties using the concept of 
uncontrollable transitions, and the solution of the ensuing optimization problem using a highly efficient search-free combinatorial approach to 
maximize quantitative measures of probabilistic regular languages. 
Applicability of the algorithm in various models of robot navigation has been shown with experimental validation  on a two-wheeled mobile robotic platform (SEGWAY RMP 200) in a
laboratory environment.
\end{abstract}

\begin{keyword}


Language Measure  \sep
Probabilistic Finite State Machines  \sep
Robotics  \sep
Path Planning  \sep
Supervisory Control 
\end{keyword}

\end{frontmatter}

% \renewcommand{\rmdefault}{phv}

% \pagestyle{empty}
% \maketitle
%##############################################################
\allowdisplaybreaks{
%##############################################################

% \begin{keywords}
% Language Measure; Probabilistic Finite State Machines; Robotics; Path Planning; Supervisory Control
% \end{keywords}
%##############################################################
%##############################################################
%
% \vspace{-10pt}
% \pagestyle{empty}
\section{Introduction \& Motivation}
%##############################################################
%##############################################################
\vspace{0pt}  
The objective of this paper is to report  a 
globally-optimal approach to path planning under uncertainty, based 
on the theory of quantitative measures of formal languages.
The field of trajectory and
motion planning is enormous, with applications in such diverse
areas as industrial robots, mobile robot navigation,
spacecraft reentry, video games and even drug design. Many of
the basic concepts are presented in ~\cite{Lat91} and in  recent comprehensive surveys~\cite{Lav06}. In the context of planning for mobile
robots and manipulators much of the literature on path and
motion planning is concerned with finding collision-free
trajectories~\cite{KK92}. A great deal of the complexity in
these problems arises from the topology of the robot's
configuration space, called the $\C$-Space. Various analytical techniques, such as wavefront expansion~\cite{BK02}
and cellular decomposition, have been reported in recent literature~\cite{Lp87}, which partition the $\C$-Space into a finite number
of regions with the objective of
reducing the motion planning problem as identification of a
sequence of neighboring cells between the initial and final (i.e., goal)
regions. Graph-theoretic search-based techniques have been used somewhat
successfully in many wheeled ground robot path planning
problems and have been used for some UAV planning problems,
typically radar evasion~\cite{AniHamHu03}. These approaches typically
suffer from complexity issues arising from expensive searches, particularly in complicated configuration spaces.  To circumvent the complexity associated
with graph-based planning, sampling based planning
methods~\cite{BLL90} such as probabilistic roadmaps have been proposed.
However, sampling based  approaches are only probabilistically
complete (i.e. if a feasible solution exists it will be found,
given enough time) but there is no guarantee of finding a
solution within a specified time, and more often than not, global route optimality is not guaranteed. Distinct from 
these general approaches, there exist reported techniques that explicitly make use of  physical aspects of specific  problems for planning $e.g.$ use of vertical wind 
component for generating optimal trajectories for UAVs~\cite{Lan08}, feasible collision-free trajectory generation for cable driven platforms~\cite{LOZRC09}, and the recently reported approach employing angular processing~\cite{Ort09}.
% Due to the 
\subsection{Potential Field-based Planning Methodology} 
Among reported deterministic approaches, methods based on 
artificial potential fields have been extensively investigated, often referred to cumulatively as potential field methods (PFM).
The idea of imaginary forces acting
on a robot were suggested by several authors inclding \cite{AH83} and \cite{Kh85}. In these approaches
obstacles exert repulsive forces onto the robot, while the target applies an attractive force to the robot. The
resultant of all forces determines the subsequent direction and speed of travel. One of the
reasons for the popularity of this method is its simplicity and elegance. Simple PFMs can be implemented
quickly and initially provide acceptable results without requiring many refinements. \cite{Kr84} has suggested a
generalized potential field method that combines global and local path planning. Potential field based techniques have been also successfully employed in 
multi-robot co-operative planning scenarios~\cite{KGZ07,SHK08}, where other techniques prove to be inefficient and impractical.

While the potential field principle is particularly attractive because of its
elegance and simplicity, substantial shortcomings have been identified  that are inherent
to this principle. The interested reader is referred to \cite{BK91-1} for a systematic criticism of PFM-based planning,
 where the authors cite the underlying differential equation based analysis as the source of the problems, and the fact that it combines the robot and
 the environment into one unified system. 
Key problems  inherent to PFMs, independent of the particular implementation, are:
\begin{enumerate}
 \item \textbf{Trap situations due to local minima:} Perhaps the best-known  problem with PFMs are possible trap-situations \cite{AH83,Ti90}, which occur when when the robot runs into
a dead end, due to the existence of a local extrema in the potential field. 
% It has been proposed  that trap-situations can be resolved
% by heuristic or global recovery.
Trap-situations can be remedied with heuristic recovery rules, which are likely to result in  non-optimal paths. 
\item \textbf{No passage between closely spaced obstacles:} A  severe problem with PFMs
occurs when the robot attempts to travel through narrow corridors thereby  experiencing repulsive
forces simultaneously from opposite sides, leading to wavy trajectories, no passage etc.
\item \textbf{Oscillations in the presence of obstacles:} Presence of high obstacle clutter often leads to unstable motion, due to the 
complexity of the resultant potential.
\item \textbf{Effect of past obstacles:} Even after the robot has already passed an obstacle, the latter keeps affecting the 
robot motion for a significant period of time (until the repulsive potential dies down).
\end{enumerate}
These disadvantages become more apparent when the PFM-based methods are implemented in high-speed real-time systems; simulations and slow speed experiments
often conceal the issues; probably contributing to the widespread popularity of potential planners.
\subsection{The $\nustar$ Planning Algorithm}
Recently, the authors reported 
a novel path planning algorithm $\nustar$~\cite{CMR08}, 
 that models the navigation
problem in the framework of Probabilistic Finite State
Automata (PFSA) and computes optimal plans via optimization of
the PFSA from a strictly control-theoretic viewpoint. 
$\nustar$ uses cellular decomposition of the workspace, and assumes 
that the blocked grid locations can be easily estimated, upon which the planner
computes an optimal navigation gradient that is used to obtain the routes.
This navigation gradient is computed by optimizing the quantitative measure of the probabilistic formal language
generated by the associated navigation automaton.
The key advantages can be enumerated as:
\begin{enumerate}
\item \textbf{$\nustar$ is fundamentally distinct
from a search:} The search problem is replaced by a sequential
solution of sparse linear systems. On completion of cellular decomposition,
$\nustar$ optimizes the resultant PFSA via a iterative
sequence of combinatorial operations which elementwise
maximizes the language measure vector~\cite{CR06}\cite{CR07}.
Note that although $\nustar$ involves probabilistic reasoning,
the final waypoint sequence obtained is deterministic.
\item \textbf{Computational efficiency:} The
intensive step in $\nustar$ is a special sparse matrix inversion to compute the language
measure. The time complexity of each iteration step can be
shown to be linear in problem size implying significant numerical advantage over search-based
methods for high-dimensional problems.
\item  \textbf{Global monotonicity:} The solution
iterations are globally monotonic, $i.e$, each iteration yields a better approximation to the final optimal solution.
The final waypoint sequence is generated essentially by
following the measure gradient which has a unique  maxima  at the
goal. %
\item \textbf{Global Optimality:} It can be shown that trap-situations are a 
mathematical impossibility for $\nustar$.
\end{enumerate}
The optimal navigation gradient produced by $\nustar$ is reminiscent  of potential field methods
\cite{BLL90}. However, $\nustar$ automatically generates, \textit{and optimizes} this
 gradient; no ad-hoc potential function is necessary.
\subsection{Focus of Current Work \& Key Contributions}
The key focus of this paper is extension of the $\nustar$ planning algorithm to optimally handle 
execution uncertainties. It is well recognized by domain experts that merely coming up with a navigation plan is not sufficient; the computed plan 
must be executed in the real world by the mobile robot, which often cannot be done exactly and precisely due to measurement noise in 
the exteroceptive sensors,  imperfect actuations, and external disturbances.
The idea of planning under uncertainties is not particularly new, and good surveys of reported methodologies exist~\cite{JTCL05}. 
   In chronological order, the main family of reported approaches can be enumerated  as follows:
\begin{itemize}
 \item Pre-image Back-chaining \cite{LP84,ll92,fm98} where the plan is synthesized by computing a set of configurations 
from which the robot can possibly reach the goal, and then propgating this \textit{preimage} recursively backward or \textit{back-chaining}, a problem solving approach
originally proposed in \cite{Ni80}.
\item Approach based on sensory uncertainty fields (SUF) \cite{tl92,tk96,vt98,rt99}
 computed over the collision-free subset of the robot's configuration space, which reflects expected uncertainty (distribution of possible errors)
in the sensed configuration that would be computed
by matching the sensory data against a known  environment
model ($e.g.$ landmark locations). A planner
then makes use of the computed SUF to generate paths
that minimize expected errors.
\item  Sensor-based planning approaches \cite{as94,bsa95,kbsc97}, which consider explicit uncertainty models of various motion primitives to compute a feasible robust plan
composed of sensor-based motion commands  in polygonal environments, with significant emphasis on wall-following schemes.
\item  Information space based approach  using the Bellman principle of stochastic dynamic programming \cite{BF95}, which introduced key concepts such as 
setting up the problem in a probabilistic framework, and demanding that the optimal plan maximize the probability of reaching the goal. However, the 
main drawback was the exponential dependence on the dimension of the computed information space.
\item  The set-membership approach \cite{PS95}
which performs a local
search,  trying to deform a path into one that respects uncertainty
constraints imposed by arbitrarily shaped uncertainty sets. Each hard constraint is turned into a soft penalty function,
and the gradient descent algorithm is employed, hoping convergence to an admissible solution.
\item  Probabilistic approaches based on disjunctive linear programming \cite{blw06,b06}, with emphasis on UAV applications.
The key limitation is the inability to take into account exteroceptive sensors, and also the assumption that dead-reckoning is independent of the 
path executed. Later extensions of this approach use  particle representations of the distributions, implying wider applicability.
\item  Adaptation of  search strategies in extended spaces \cite{lf00,LG03,gs05,gs07}, which consider 
the classical search problem in configuration spaces augmented with uncertainty information.

\item Approach based on Stochastic Motion Roadmaps (SRM)~\cite{ASG07}, which combines  sampling-based roadmap representation of the configuration
space, with  the theory of Markov Decision Processes, to yield models that can be subsequently 
optimized via value-iteration based infinite horizon dynamic programming, leading to plans that maximize the probability of reaching the goal.
\end{itemize}

The current work adds a new member to the family of existing approaches  to address globally  optimal path planning under uncertainties.
The key novelty of this paper is the association of \textit{uncertainty with the notion of uncontrollability in a controlled system}. The navigation automaton introduced in ~\cite{CMR08} is augmented with
uncontrollable transitions which essentially captures the possibility that the agent may execute actuation sequences (or motion primitives) that 
are not coincident with the planned moves. The planning objective is simple: \textit{Maximize the probability of reaching the goal, while simultaneously minimizing the probability of hitting any obstacle.} Note that, in this respect, we are essentially solving the exact same problem investigated by ~\cite{ASG07}. However our solution approach is very different.
Instead of using value iteration based dynamic programming, we use the theory of language-measure-theoretic optimization of probabilistic finite state automata~\cite{CR07}.
Unlike the SRM approach, the proposed algorithm does not require the use of local auxiliary planners, and also needs to make no assumptions on the structure of the configuration space
to guarantee iterative convergence. The use of arbitrary penalties for reducing the weight on longer paths is also unnecessary, which makes the proposed $\nustar$ under uncertainties completely free 
from heuristics. We show that all the key advantages that $\nustar$ has over the state-of-art carries over to this more general case; namely that of significantly better
computational efficiency, simplicity of implementation, and achieving global optimality via monotonic sequence of \textit{search-free combinatorial iterative improvements}, with guaranteed polynomial convergence. 
The proposed approach thus solves the inherently non-convex optimization~\cite{SS08} by mapping the physical specification to  an optimal control problem for probabilistic finite state machines (the navigation automata), which admits efficient combinatorial solutions via the language-measure-theoretic approach. 
The source of  many uncertainties, namely modeling uncertainty,
disturbances, and uncertain localization, is averaged over (or amortized) for adequate representation in the automaton framework. This may be viewed as a source of approximation in the proposed approach; however we show in simulation and in actual experimentation that the amortization is indeed a good approach to reduce planning complexity and results in highly robust planning decisions.
% 
%    The problem of planning optimal paths under uncertainties is challenging for two principal reasons.
% First, as noted by [3], the optimization problem is inherently
% non-convex. Second, there are a number of sources of
% uncertainty in the problem, such as modeling uncertainty,
% disturbances, and uncertain localization. 
%  
% 
%    All the algorithms described in the following chapters can be consid-
% ered specializations of the same generic search algorithm ~\cite{Lav06}, enriched by a dominance relation $\trianglerighteq$, which is used to discard
% useless nodes in the serach space.
% 
% 
% We present a significant generalization of the $\nustar$ algorithm; the average dynamic
% uncertainty in plan execution is integrated with the planning
% process, resulting in plans that are highly robust, which
% take into account the amortized effect of physical dynamic limitations of individual
% robotic platforms and possibly different operating conditions
% and execution parameters.
% 
% 
% 
Thus the modified language-measure-theoretic approach presented in this paper, potentially lays the framework  for seamless integration of 
 data-driven and  physics-based models with the high-level decision processes; this is a crucial advantage, and goes to address a key issue in autonomous robotics, $e.g.$, 
in a path-planning scenario with mobile robots, the optimal
path may be very different for different speeds, platform capabilities and mission specifications.
Previously reported approaches to handle these effects using exact differential models of platform dynamics results in overtly complex solutions that do not respond well
to modeling uncertainties,  and more importantly to possibly non-stationary environmental dynamics and evolving mission contexts. Thus the measure-theoretic approach enables the development of true \textbf{Cyber-Physical algorithms} for
control of autonomous systems; algorithms that operate in the logical domain while optimally integrating, and responding to, physical information in the planning process.

The rest of the paper is organized in seven sections. Section~\ref{BriefReview} briefly explains the
language-theoretic models considered in this paper, reviews
the language-measure-theoretic optimal control of
probabilistic finite state machines and presents the necessary details
of the reported $\nustar$ algorithm.
Section~\ref{secmodel} presents the modifications to the 
navigation model to incorporate the effects of dynamic uncertainties within
the framework of probabilistic automata. Section~\ref{secoptplan} presents the pertinent
theoretical results and establishes the main planning algorithm.
Section~\ref{secuncertain} develops a formulation to identify the key amortized uncertainty parameters of the PFSA-based navigation model
from an  observed dynamical response of a given platform. The proposed algorithm is summarized with pertinent comments in 
Section~\ref{secsummalgo}. The theoretical development is verified in high-fidelity simulations on different navigation models and 
validated in experimental runs
on the SEGWAY RMP 200 in section~\ref{secvv}.
The paper is summarized and concluded in Section~\ref{secsummary}
with recommendations for future work. 
% 
%##############################################################################
%##############################################################################
%##############################################################################
\section{Preliminaries: Language Measure-theoretic Optimization Of Probabilistic Automata}\label{BriefReview}
This section summarizes the signed real measure of regular
languages; the details are reported in~\cite{R05}.
Let $G_i \equiv \langle Q,\Sigma,\delta,q_{i},Q_{m}\rangle$ be a
trim (i.e., accessible and co-accessible) finite-state automaton
model that represents the discrete-event dynamics of a physical
plant, where $Q=\{q_k: k\in \mathcal{I}_Q\}$ is the set\index{set}
of states and $\mathcal{I}_Q\equiv\{1, 2, \cdots, n\}$ is the
index set of states; the automaton starts with the initial state
$q_{i}$; the alphabet of events is $\Sigma=\{\sigma_k: k\in
\mathcal{I}_\Sigma\}$, having $\Sigma \bigcap \mathcal{I}_Q =
\emptyset$ and $\mathcal{I}_\Sigma\equiv\{1, 2, \cdots, \ell\}$ is
the index set of events; $\delta:Q\times\Sigma \rightarrow Q$ is
the (possibly partial) function of state transitions; and
$Q_{m}\equiv \{q_{m_1},q_{m_2},\cdots ,q_{m_l}\} \subseteq Q$ is
the set of marked (i.e., accepted) states with $q_{m_k}=q_j$ for
some $j \in \mathcal{I}_Q$. Let $\Sigma^{*}$  be the Kleene closure of $\Sigma$, i.e., the set\index{set} of all finite-length strings made of the events
belonging to $\Sigma$ as well as the empty string $\epsilon$ that
is viewed as the identity  of the monoid $\Sigma^{*}$ under the
operation of string concatenation, i.e., $\epsilon s=s=s\epsilon$.
The state transition map $\delta$ is recursively extended to its reflexive and transitive closure  $\delta:Q \times \Sigma^* \rightarrow Q$ by defining
$\forall q_j \in Q, \ \delta(q_j,\epsilon) = q_j $ and $
\forall q_j \in Q, \sigma \in \Sigma,  s \in \Sigma^\star, \ \delta(q_i,\sigma s) = \delta(\delta(q_i,\sigma),s)
$
%\vspace{-10pt}
%
\begin{defn}\label{Lgen}
The language $L(q_i)$ generated by a DFSA $G$ initialized at the
state $q_i\in Q$ is defined as:
%\begin{equation} \label{LG-q}
$
    L(q_i) = \{s \in \Sigma^* \ | \ \delta^*(q_i, s) \in Q \}
$
The language $L_m(q_i)$ marked by the DFSA $G$ initialized at the
state $q_i\in Q$ is defined as:
%\begin{equation} \label{LmG-q}
$
    L_m(q_i) = \{s \in \Sigma^* \ | \ \delta^*(q_i, s) \in Q_m \}
$
\end{defn}
\begin{defn}
For every $q_j\in Q$, let $L(q_i, q_j)$ denote the set of all
strings that, starting from the state $q_i$, terminate at the
state $q_j$, i.e.,
$
L_{i,j} = \{ s \in \Sigma^* \ | \ \delta^*(q_i, s) = q_j \in Q \}
$
\end{defn}
The formal language measure is first defined for terminating
plants~\cite{G92,G92-2} with sub-stochastic event generation
probabilities $i.e.$ the event generation probabilities at
each state summing to strictly less than unity.
\begin{defn} \label{pitildefn}
The event generation
probabilities are specified by the function $\tilde {\pi }:\Sigma^\star
\times \,Q\to [0,\,1]$ such that $\forall q_j \in Q,\forall
\sigma _k \in \Sigma , \forall s\in \Sigma^\star ,$
\begin{enumerate}
\item[(1)] $\tilde{\pi}({\sigma_k},q_j) \triangleq
\tilde{\pi}_{jk} \in [0, 1)$; \ $\sum_{k} \tilde{\pi}_{jk} = 1
 - \theta, \ \mathrm{with} \  \theta \in (0,1) $; \item[(2)] $\tilde{\pi}(\sigma,q_j) = 0$ if
$\delta(q_j,\sigma)$ is undefined; $\
\tilde{\pi}(\epsilon,q_j) = 1$; \item[(3)]
$\tilde{\pi}({\sigma_k s},q_j)= \tilde{\pi}({\sigma_k}, q_j)\
\tilde{\pi}(s, \delta(q_j,\sigma_k))$.
\end{enumerate}
The $n\times \ell$ event cost matrix is defined as:
$
\mathbf{\widetilde{\Pi}}\vert_{ij}=\tilde{\pi}(q_i,\sigma_j)
$
\end{defn}\vspace{0pt}
\begin{defn} \label{pifn}
The state transition probability $\pi: Q \times Q \rightarrow [0, 1)$,
of the DFSA\index{DFSA} $G_i$ is defined as follows:
$
\forall q_i, q_j \in Q, \pi_{ij} =
    \displaystyle \sum_{\sigma\in\Sigma \ \mathrm{s.t.} \  \delta(q_i,\sigma)=q_j } \tilde{\pi}(\sigma, q_i)
$
The $n\times n$ state transition probability matrix is
defined as
$
\mathbf{\Pi}\vert_{jk} = \pi(q_i,q_j)
$
\end{defn}\vspace{0pt}
The set $Q_m$ of  marked states is partitioned into $Q_m^+$ and
$Q_m^-$, i.e.,  $Q_m = Q_m^+ \cup Q_m^-$ and $Q_m^+ \cap Q_m^- =
\emptyset$, where $Q_m^+$ contains all \textit{good} marked states
that we desire to reach, and $Q_m^-$ contains all \textit{bad}
marked states that we want to avoid, although it may not always be
possible to completely avoid the \textit{bad} states while
attempting to reach the \textit{good} states. To characterize
this, each marked state is assigned a real value based on the
designer's perception of its impact on the system performance.
\begin{defn} \label{charfn}
The characteristic function $\chi:Q \rightarrow [-1, 1]$ that
assigns a signed real weight to state-based sublanguages
$L(q_i,q)$ is defined as:
\begin{equation}\label{chi}
    \forall q \in Q, \quad \chi(q) \in \left\lbrace
        \begin{array}{cc}
            [-1, 0), & q \in Q_m^-\\
            \{ 0 \}, & q \notin Q_m\\
            \rm{(0, 1]}, & \it{q} \in Q_m^+
        \end{array}
    \right.
\end{equation}
The state weighting vector, denoted by $\boldsymbol{\chi} =
[\chi_1 \ \chi_2 \ \cdots \ \chi_{n}]^T$, where $\chi_j\equiv
\chi(q_j)$ $\forall j \in \mathcal{I}_Q$, is called the
$\boldsymbol{\chi}$-vector. The $j$-th element $\chi_j$ of
$\boldsymbol{\chi}$-vector is the weight assigned to the
corresponding terminal state $q_j$.
\end{defn}
In general, the marked language $L_m(q_i)$ consists of both good
and bad event strings that, starting from the initial state $q_i$,
lead to $Q_m^+$ and $Q_m^-$ respectively. Any event string
belonging to the language $L^0 = L(q_i) - L_m(q_i)$ leads to one
of the non-marked states belonging to $Q - Q_m$ and $L^0$ does not
contain any one of the good or bad strings. Based on the
equivalence classes defined in the Myhill-Nerode Theorem, the
regular languages $L(q_i)$ and $L_m(q_i)$ can be expressed as:
%\begin{equation} \label{LGq}
$
L(q_i) = \bigcup_{q_k \in Q} L_{i,k}
$ and
%\begin{equation} \label{LmGq}
$
L_m(q_i) = \bigcup_{q_k \in Q_m} L_{i,k} =  L_m^+ \cup L_m^-
%\end{equation}
$ where the sublanguage $L_{i,k}\subseteq G_i$ having the initial
state $q_i$ is uniquely labelled by the terminal state $q_k, k \in
\mathcal{I}_Q$ and $L_{i,j} \cap L_{i,k} = \emptyset$ $\forall j
\neq k$; and $L_m^+\equiv\bigcup_{q_k \in Q_m^+} L_{i,k}$ and
$L_m^-\equiv\bigcup_{q_k \in Q_m^-} L_{i,k}$ are good and bad
sublanguages of $L_m(q_i)$, respectively. Then, $L^0 =
\bigcup_{q_k \notin Q_m} L_{i,k}$ and $L(q_i) = L^0 \cup L_m^+
\cup L_m^-$.\vspace{0pt}

A signed real measure $\mu^i:{2^{L(q_i)}} \rightarrow
\mathbb{R}\equiv(-\infty,+\infty)$ is constructed on the
$\sigma$-algebra $2^{L(q_i)}$ for any $i \in  \mathcal{I}_Q$;
interested readers are referred to~\cite{R05} for the
details of measure-theoretic definitions and results. With the
choice of this $\sigma$-algebra, every singleton set made of an
event string $s \in L(q_i)$ is a measurable set. By Hahn
Decomposition Theorem~\cite{R88}, each of these measurable sets
qualifies itself to have a numerical value based on the above
state-based decomposition of $L(q_i)$ into $L^0$(null),
$L^{+}$(positive), and $L^-$(negative) sublanguages.\vspace{-0pt}
\begin{defn} \label{measurefn}
Let $\omega \in L(q_i, q_j)\subseteq 2^{L(q_i)}$. The signed
real measure $\mu^i$ of every singleton string set $ \{ \omega
\} $ is defined as:
$
\mu^i(\{\omega \})\coloneqq\tilde{\pi}(q_i,\omega)\chi(q_j)
$.
The signed real measure of a sublanguage $L_{i,j} \subseteq
L(q_i)$ is defined as:
%\begin{equation}\label{measureEqn}
$\mu_{i,j} \coloneqq \mu^i(L(q_i, q_j)) = \left( \sum_{\omega\in
L(q_i, q_j)} \tilde{\pi}( q_i,\omega)\right)\chi_j
$
\end{defn}
Therefore, the signed real measure of the language  of a DFSA
$G_i$ initialized at $q_i \in Q$, is defined as
%\begin{equation} \label{fullObsMeasure}
$
\mu_i\coloneqq \mu^i(L(q_i)) = \sum_{j\in \mathcal{I}_Q}
\mu^i(L_{i,j})
$. It is shown in \cite{R05} that the language measure
can be expressed as
%\begin{equation} \label{algbrameasure}
$
\mu_i = \sum_{j\in \mathcal{I}_Q} \mathcal{\pi}_{ij} \mu_j +
\chi_i
$. The language measure vector, denoted as \mbox{\boldmath $\mu$} =
$[\mu_1 \ \mu_2 \ \cdots \ \mu_n]^{T}$, is called the
\mbox{\boldmath $\mu$}-vector.  In vector form, we have
$
\mbox{\boldmath$\mu$} = \mathbf{\Pi}\mathbf{\mu}+\boldsymbol{\chi}
$
whose solution is given by
%\begin{equation} \label{fullObservation}
$
\mbox{\boldmath$\mu$} = (\mathbf{I} - \mathbf{\Pi})^{-1}
\boldsymbol{\chi}
$
The inverse  exists for
terminating plant models~\cite{G92} because
$\boldsymbol{\Pi}$ is a contraction
operator~\cite{R05} due to the strict inequality
$ \sum_j \Pi_{ij} < 1$. The residual $\theta_i = 1 - \sum_j
\Pi_{ij}$ is  referred to as the termination probability for
state $q_i \in Q$. We extend the analysis to non-terminating
plants with stochastic transition probability matrices ($i.e.$
with $\theta_i = 0,\ \forall q_i \in Q$) by renormalizing the
language measure~\cite{CR06} with respect to the uniform
termination probability of a limiting terminating model as
described next.

Let $\widetilde{\Pi}$ and $\Pi$ be the stochastic event
generation and transition probability matrices for a
non-terminating plant $G_i= \langle
Q,\Sigma,\delta,q_{i},Q_{m}\rangle$. We consider the
terminating plant $G_i(\theta)$ with the same DFSA structure $
\langle Q,\Sigma,\delta,q_{i},Q_{m}\rangle$ such that the
event generation probability matrix is given by
$(1-\theta)\widetilde{\Pi}$ with $\theta \in (0,1)$ implying
that the state transition probability matrix is  $(1 -
\theta)\Pi$.
\begin{defn}[Renormalized Measure] \label{defrenormmeas}
The renormalized measure {\small $\nu^i_\theta :
2^{L(q_i(\theta))} \rightarrow [-1,1]$} for the
$\theta$-parametrized terminating plant $G_i(\theta)$ is
defined as: {\small \begin{gather} \forall \omega \in
L(q_i(\theta)), \ \nu^i_\theta(\{\omega\}) = \theta
\mu^i(\{\omega\})
\end{gather}}
The corresponding matrix form is given by
%\begin{gather} \label{renormalizedMeasure}
$
\boldsymbol{\nu_\theta} = \theta \ \boldsymbol{\mu}=
\theta \ [I-(1-\theta)\Pi]^{-1} \boldsymbol{\chi} \
\mathrm{with} \ \theta \in (0,1)
$. We note that the vector representation allows for the following notational simplification
$
  \nu_\theta^i(L(q_i(\theta))) = \boldsymbol{\nu_\theta} \big \vert_i
$
The renormalized measure for the non-terminating plant $G_i$ is defined to be $\lim_{\theta \rightarrow o^+}\nu_\theta^i$.
\end{defn}
%##############################################################
\vspace{-12pt}
%##############################################################
\subsection{Event-driven Supervision of  PFSA}\label{formulationOptimal}
Plant models considered in this paper are
\textit{deterministic} finite state automata (plant) with
well-defined event occurrence \textit{probabilities}. In other
words, the occurrence of events is probabilistic, but the
state at which the plant ends up, \textit{given a particular
event has occurred}, is deterministic. Since no emphasis is placed on the initial state and marked states are completely determined by $\chi$, the models can be completely specified by a
sextuple as:
%\begin{gather}\label{sextuple}
$G=(Q,\Sigma,\delta,\widetilde{\boldsymbol{\Pi}},\chi,\mathscr{C})
$
\vspace{-1pt}
\begin{defn}[Control Philosophy]\label{contapp}
If $q_i \xrightarrow[\sigma]{} q_k$,
and the  event $\sigma$ is disabled at state
$q_i$, then the  supervisory action is to
prevent the plant from making a transition to the state $q_k$, by
 forcing it to stay at the original state
$q_i$.  Thus disabling any transition $\sigma$ at a given state $q$ results
in   deletion of the original transition and appearance of the self-loop $\delta(q,\sigma) = q$ with the occurrence
probability of $\sigma$ from the state $q$ remaining unchanged
in the supervised and unsupervised plants. For a given plant, transitions that can be disabled in the
sense of Definition~\ref{contapp} are defined to be
\textit{controllable} transitions. The set of
controllable transitions in a plant is denoted $\mathscr{C}$.
\textit{Note controllability is state-based.}
\end{defn}\vspace{0pt}
%\begin{defn}\label{optifull-contdef}\textbf{(Controllable Transitions)}
%
%\end{defn}
%
\vspace{-10pt}
%##############################################################
%##############################################################
%##############################################################
\subsection[Optimal Supervision Problem]{Optimal Supervision Problem: Formulation \& Solution}
A supervisor disables a subset of the set $\mathscr{C}$ of
controllable transitions and hence there is a bijection
between the set of all possible supervision policies and the
power set $2^{\mathscr{C}}$. That is, there exists $2^{\vert
\mathscr{C} \vert}$ possible supervisors and each supervisor
is uniquely identifiable with a subset of $\mathscr{C}$ and
the language measure $\nu$ allows a quantitative comparison of
different policies.
\vspace{-3pt}
\begin{defn}\label{superior}
For an unsupervised  plant
$G=(Q,\Sigma,\delta,\widetilde{\Pi},\chi,\mathscr{C})$, let
$G^{\dag}$ and $G^{\ddag}$ be the supervised plants with sets
of disabled transitions, $\mathscr{D}^{\dag}\subseteq
\mathscr{C}$ and $\mathscr{D}^{\ddag}\subseteq \mathscr{C}$,
respectively, whose measures are $\boldsymbol{\nu}^{\dag}$ and
$\boldsymbol{\nu}^{\ddag}$. Then, the supervisor that disables
$\mathscr{D}^{\dag}$ is defined to be superior to the
supervisor that disables $\mathscr{D}^{\ddag}$  if
$\boldsymbol{\nu}^{\dag} \eg \boldsymbol{\nu}^{\ddag}$ and
strictly superior if $\boldsymbol{\nu}^{\dag} \egg
\boldsymbol{\nu}^{\ddag}$.
\end{defn}
\vspace{-1pt}
\begin{defn}[Optimal Supervision Problem]\label{pdef}
Given a (non-terminating) plant
$G=(Q,\Sigma,\delta,\widetilde{\Pi},\chi,\mathscr{C})$, the
problem is to compute a supervisor that disables a subset
$\mathscr{D}^\star \subseteq \mathscr{C}$, such that $
\boldsymbol{\nu}^{\star} \eg \boldsymbol{\nu}^{\dag} \ \
\forall \mathscr{D}^{\dag} \subseteq \mathscr{C} $ where
$\boldsymbol{\nu}^{\star}$ and $\boldsymbol{\nu}^{\dag}$ are
the measure vectors of the supervised plants $G^{\star}$ and
$G^{\dag}$ under $\mathscr{D}^{\star}$ and $\mathscr{D}^\dag$,
respectively.
\end{defn}
\begin{rem}\label{remtheta}
The solution to the optimal supervision problem  is obtained
in \cite{CR07,C-PhD} by designing an optimal policy for a
\textit{terminating} plant~\cite{G92-2} with a
sub-stochastic transition probability matrix
$(1-\theta)\widetilde{\Pi}$ with $\theta \in (0,1)$. To ensure
that the computed optimal policy coincides with the one for
$\theta = 0$, the suggested algorithm chooses a \textit{small}
value for $\theta$ in each iteration step of the design
algorithm. However, choosing $\theta$ too small may cause
numerical problems in convergence. Algorithms reported in \cite{CR07,C-PhD}
computes how small a $\theta$ is actually required, $i.e.$,
computes the critical lower bound $\theta_\star$, thus
solving the optimal supervision
problem for a generic PFSA.
It is further shown that the solution obtained is optimal and unique and can be computed by an effective algorithm.
\end{rem}
\begin{defn}\label{defthetamin}
Following Remark~\ref{remtheta}, we note that
algorithms reported in \cite{CR07,C-PhD} compute a lower bound for the
critical termination probability for each iteration of
 such that the disabling/enabling
decisions for the terminating plant coincide with the given
non-terminating  model. We define
$
\theta_{min} = \min_{k} \theta^{[k]}_\star
$
where $\theta^{[k]}_\star$ is the termination probability
computed  in the $k^{th}$
iteration.
\end{defn}
\vspace{-3pt}
\begin{defn}\label{defnustarmeas}
If $G$ and $G^\star$ are the unsupervised and supervised PFSA
respectively then we denote the renormalized measure of the
terminating plant $G^\star(\theta_{min})$ as
$\nu_{\#}^i:2^{L(q_i)} \rightarrow [-1,1]$ (See
Definition~\ref{defrenormmeas}). Hence, in vector notation we
have:
$
 \boldsymbol{\nu_{\#}} = \theta_{min}  [I-(1-\theta_{min})\Pi^\#]^{-1} \boldsymbol{\chi}
$
where $\Pi^\#$ is the transition probability matrix of the supervised plant $G^\star$,  we note that
$\boldsymbol{\nu_{\#}} = \nu^{[K]}$ where $K$ is the total
number of iterations required for convergence.
\end{defn}
For the sake of completeness, the algorithmic approach is shown in Algorithms~\ref{Algorithm02} and \ref{Algorithm01}.
%############################################################################
%##############################################################################%##############################################################################
% \appendix[Pertinent Algorithms For Language Measure-theoretic PFSA Optimization]\label{appen}
% 
\begin{algorithm}[!ht]
 \footnotesize 
%  \SetNoline
  \SetKwData{Left}{left}
  \SetKwData{This}{this}
  \SetKwData{Up}{up}
  \SetKwFunction{Union}{Union}
  \SetKwFunction{FindCompress}{FindCompress}
  \SetKwInOut{Input}{input}
  \SetKwInOut{Output}{output}
  \SetKw{Tr}{true}
   \SetKw{Tf}{false}
  \caption{Computation of Optimal Supervisor}\label{Algorithm02}
\Input{$\mathbf{P}, \ \boldsymbol{\chi}, \ \mathscr{C}$}
\Output{Optimal set of disabled transitions
$\mathscr{D}^{\star}$} \Begin{ Set
$\mathscr{D}^{[0]}=\emptyset$ \tcc*[r]{Initial disabling set}
Set $\widetilde{\Pi}^{[0]}=\widetilde{\Pi} $ \tcc*[r]{Initial
event prob. matrix}
%Set
%$\mathbf{P}^{[0]}=\mathbf{P}$\tcc*[r]{Initial tran. prob. matrix}
Set $\theta^{[0]}_{\star} = 0.99$, Set $k \ = \ 1$ , Set $
\texttt{Terminate} = $ \Tf \;
 \While{($ \texttt{Terminate}$ == \Tf)}{
 Compute $\theta_{\star}^{[k]}$\tcc*[r]{Algorithm~\ref{Algorithm01}}
 Set $\widetilde{\Pi}^{[k]} = \frac{1-\theta_{\star}^{[k]}}{1-\theta_{\star}^{[k-1]}}\widetilde{\Pi}^{[k-1]}$\;
% Compute
%$\mathbf{P}^{[k]}$ \tcc*[r]{$k^{th}$ Tran. Prob. matrix}
 Compute $\boldsymbol{\nu}^{[k]}$ \;
 \For{$j = 1$ \textbf{to} $n$}{
\For{$i = 1$ \textbf{to} $n$}{
 Disable  all controllable
  $q_i \xrightarrow[]{\sigma} q_j$ s.t.
$\boldsymbol{\nu}^{[k]}_j < \boldsymbol{\nu}^{[k]}_i $ \;
 Enable all controllable   $q_i \xrightarrow[]{\sigma} q_j$
 s.t.
$\boldsymbol{\nu}^{[k]}_j \geqq \boldsymbol{\nu}^{[k]}_i $ \;
} } Collect all disabled transitions in $\mathscr{D}^{[k]}$\;
\eIf{$\mathscr{D}^{[k]} ==
\mathscr{D}^{[k-1]}$}{$\texttt{Terminate} =$ \Tr \;}{$k \ = \
k \ + \ 1$ \;}
 }
 $\mathscr{D}^{\star} \ = \ \mathscr{D}^{[k]}$
 \tcc*[r]{Optimal disabling set}
  }\vspace{0pt}
\end{algorithm}
%##############################################################################
\begin{algorithm}[!hb]
%  \restylealgo{linesnumbered,ruled,noline,noend}
 \footnotesize 
%  \SetNoline
% \footnotesize \SetLine
%  \SetKwData{Left}{left}
%  \SetKwData{This}{this}
%  \SetKwData{Up}{up}
%  \SetKwFunction{Union}{Union}
%  \SetKwFunction{FindCompress}{FindCompress}
  \SetKwInOut{Input}{input}
  \SetKwInOut{Output}{output}
  \SetKw{Tr}{true}
   \SetKw{Tf}{false}
  \caption{Computation of the Critical Lower Bound $\theta_{\star}$ }\label{Algorithm01}
\Input{$\mathbf{P}, \ \boldsymbol{\chi}$}
\Output{$\theta_{\star}$} \Begin{ Set $\theta_{\star} =
1$, Set $\theta_{curr} = 0$\; Compute $\Q$ , $M_0$ ,
 $M_1 $,
  $M_2 $\;
 \For{$j = 1$ \textbf{to} $n$}{
\For{$i = 1$ \textbf{to} $n$}{ \eIf{$\left (\Q
\boldsymbol{\chi} \right )_i - \left (\Q \boldsymbol{\chi}
\right )_j \neq 0$}{$\theta_{curr} = \frac{1}{8M_2}\big \vert
\left (\Q \boldsymbol{\chi} \right )_i - \left (\Q
\boldsymbol{\chi} \right )_j \big \vert $}{
\For {$r=0$ \textbf{to} $n$}{ \eIf{$\left
(M_0\boldsymbol{\chi} \right )_i \neq \left
(M_0\boldsymbol{\chi} \right )_j$}{\textbf{Break}\;}{\If
{$\left ( M_0 M_1^r \boldsymbol{\chi} \right )_i \neq \left (
M_0 M_1^r \boldsymbol{\chi} \right )_j$}{\textbf{Break}\;} } }
\eIf{$r==0$}{$\theta_{curr} = \frac{\vert \left \{ (M_0 -\Q
)\boldsymbol{\chi} \right \}_i - \left \{ (M_0 -\Q
)\boldsymbol{\chi} \right \}_j \vert}{8M_2} $\;} {\eIf { $r >
0$
 \textbf{AND} $ r \leq n $ }{ $ \theta_{curr} = \frac{\vert
\left (M_0 M_1\boldsymbol{\chi} \right )_i - \left (M_0 M_1
\boldsymbol{\chi} \right )_j \vert }{2^{r+3}M_2} $ \;}{ $
\theta_{curr} = 1 $ \;}}}
$\theta_{\star}$ = $\mathrm{min} (
\theta_{\star} , \theta_{curr} ) $ \;
 } } }
\end{algorithm}
%
% 
% 
%############################################################################
%############################################################################
\vspace{-10pt}
\subsection{Problem Formulation: A PFSA Model of Autonomous Navigation}\label{sectionformulation}
%##############################################################################
%
%##############################################################################
%##############################################################################
\begin{figure}[!ht]
\center
\begin{minipage}{3.5in}
\center
\includegraphics[width=3.5in]{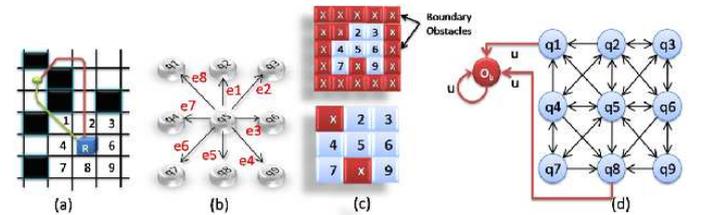}
\end{minipage}
%#################
\caption{\textbf{(a)} shows the vehicle (marked "R")  with the
obstacle positions shown as black squares. The Green4 dot
identifies the goal \textbf{(b)} shows the finite state
representation of the possible one-step moves from the current
position. \textbf{(d)} shows uncontrollable transitions "u"
from states corresponding to blocked grid locations to
"$q_{\circleddash}$"} \label{Fig02} \vspace{-10pt}
\end{figure}
%##############################################################################
%##############################################################################
%
We consider a 2D workspace for the mobile agents. This
restriction on workspace dimensionality  serves to simplify
the exposition and can be easily relaxed. To set up the
problem, the workspace is first discretized into a finite grid
and hence the approach developed in this paper falls under the
generic category of discrete planning. The underlying  theory
does not require the grid to be regular; however for the sake
of clarity we shall present the formulation under the
assumption of a regular grid. The obstacles are represented as
blocked-off grid locations in the discretized workspace. We
specify a particular location as the fixed goal and consider
the problem of finding optimal and feasible paths from
arbitrary initial grid locations in the workspace.
Figure~\ref{Fig02}(a) illustrates the basic problem setup.
We further assume that at any given time instant the robot
occupies one particular location ($i.e.$ a particular square
in Figure~\ref{Fig02}(a)). As shown in Figure~\ref{Fig02}, the
robot has eight possible moves from any interior location. The
boundaries are handled by removing the  moves that take the
robot out of the workspace. The possible moves  are modeled as
controllable transitions between grid locations since the
robot can "choose" to execute a particular move from the
available set. We note that the number of possible moves ($8$
in this case) depends on the chosen fidelity of discretization
of the robot motion and also on the intrinsic vehicle
dynamics. The complexity results presented in this paper only
assumes that the number of available moves is significantly
smaller compared to the number of grid squares, $i.e.$, the
discretized position states.  Specification of inter-grid
transitions in this manner allows us to generate a finite
state automaton (FSA) description of the navigation problem.
Each square in the discretized workspace is modeled as a FSA
state with the controllable transitions defining the
corresponding state transition map. The formal description of
the model is as follows:

Let $\Gn =(Q,\Sigma,\delta,\tilde{\Pi},\chi)$ be a Probabilistic Finite State
Automaton (PFSA).
%with $Q$ being the set of states,
%$\Sigma$ being the alphabet and $\delta : Q\times \Sigma
%\rightarrow Q$ being the state transition map~\cite{HMU01}. $\tilde{\Pi
%The initial state and accepting states are unspecified.
The state set $Q$ consists of states that correspond to grid locations and one extra state denoted by $q_\circleddash$. The necessity of this special state $q_\circleddash$ is explained in the sequel. The grid squares are numbered in a pre-determined scheme such
that each $q_i \in Q\setminus \{q_\circleddash\}$ denotes a specific square in the
discretized workspace. The particular numbering scheme chosen
is irrelevant. In the absence of dynamic uncertainties and
state estimation errors, the alphabet contains one
uncontrollable event $i.e.$ $\Sigma = \Sigma_C \bigcup \{ u\}$
such that  $\Sigma_C$ is the set of controllable events
corresponding to the possible moves of the robot. The
uncontrollable event $u$ is defined from each of the blocked
states and leads to  $q_\circleddash$ which is a deadlock state. All other
transitions ($i.e.$ moves) are removed from the blocked
states. Thus, if a robot moves into a blocked state, it
uncontrollably transitions to the deadlock state $q_\circleddash$ which
is physically interpreted to be a collision. We further assume
that the robot fails to recover from collisions which is
reflected by making $q_{\circleddash}$ a deadlock state. We note that $q_\circleddash$
does not correspond to any physical grid location. The set of blocked grid locations along with the obstacle state $q_\circleddash$ is denoted as $Q_\obs \subseteqq Q$. Figure~\ref{Fig02} illustrates the navigation automaton for a
nine state discretized workspace with two blocked squares.
Note that the only outgoing transition from the blocked states
$q_1$ and $q_8$ is $u$.
Next we augment the navigation FSA by specifying event
generation probabilities defined by the map $\tilde{\pi} : Q
\times \Sigma \rightarrow [0,1]$  and the characteristic
state-weight vector specified as  $\chi:Q \rightarrow [-1,1]$.
%
%The event probabilities satisfy $ \forall q_i \in Q, \ \sum_j
%\tilde{\pi}(q_i,\sigma_j) = 1 $.
%
The characteristic
state-weight vector~\cite{CR07} assigns scalar weights to the PFSA states to capture the desirability of
ending up in each state.
%In the general analysis presented in
%\cite{CR07}, relatively negative weights are assigned to "bad"
%states and positive weights are assigned to "good" states.
\begin{defn}\label{defcharold}
The characteristic weights are specified for the navigation
automaton as  follows: {\small \begin{align} \chi(q_i) = \left
\{
\begin{array}{cl}
-1 & \mathrm{if} \ q_i \equiv q_{\circleddash} \\
1 & \mathrm{if} \ q_i \ \mathrm{is \ the \ goal} \\
0 & \mathrm{otherwise}
\end{array}
\right.
\end{align}}
\end{defn}
In the absence of dynamic constraints and state estimation uncertainties, the robot can "choose" the particular controllable
transition to execute at any grid location. Hence we assume that the
probability of generation of controllable events is uniform
over the set of moves defined at any particular state.
\begin{defn}\label{defpitildenav}
Since there is no
uncontrollable events defined at any of the unblocked states
and no controllable events defined at any of the blocked
states, we have the following consistent specification of
event generation probabilities: $\forall q_i \in Q,\sigma_j
\in \Sigma$, %%
 {\small\begin{align}
\tilde{\pi}(q_i,\sigma_j) & = \left \{ \begin{array}{cl}
\frac{1}{\mathrm{No. \ of \ controllable \ events \ at \ } q_i}, & \mathrm{if} \ \sigma_j \in \Sigma_C \\
1, & \mathrm{otherwise}
\end{array}
\right. \notag
%\\
%& =  \left
%\{ \begin{array}{cl}
%\frac{1}{\Crd(\mathscr{M})}, & \mathrm{if} \ \sigma_j \in \Sigma_C \\
%1, & \mathrm{otherwise}
%\end{array}
%\right.
\end{align}}
\end{defn}
The boundaries  are  handled by
"surrounding" the workspace with blocked position states shown
as "boundary obstacles" in the upper part of
Figure~\ref{Fig02}(c).
\begin{defn}\label{defmoves}
The  navigation model id defined to have identical
connectivity as far as controllable transitions are concerned
implying that every controllable transition or move ($i.e.$
every element of $\Sigma_C$) is defined from each of the
unblocked states.
\end{defn}
%%
%%##############################################################################
%%##############################################################################
%%##############################################################################
\vspace{-10pt}
\subsection{Decision-theoretic Optimization of PFSA}\label{sectionsolution}
The above-described probabilistic finite state automaton
(PFSA) based navigation model allows us to compute optimally
feasible path plans via the language-measure-theoretic
optimization algorithm~\cite{CR07} described in
Section~\ref{BriefReview}. Keeping in line with nomenclature in
the path-planning literature, we refer to the
language-measure-theoretic algorithm as $\nustar$ in the
sequel. For the unsupervised model, the robot is free to
execute any one of the defined controllable events from any
given grid location (See Figure~\ref{Fig02}(b)). The optimization
algorithm selectively disables controllable transitions to
ensure that the formal measure vector of the navigation
automaton is elementwise maximized. Physically, this implies
that the supervised robot is constrained to choose among only
the enabled moves at each state such that the probability of
collision is minimized with the probability of reaching the
goal simultaneously maximized.
\textit{Although $\nustar$ is based on optimization of
probabilistic finite state machines, it is shown that an
optimal and feasible path plan can be obtained that is
executable in a purely deterministic sense.}

Let $\Gn$ be the unsupervised navigation automaton and
$\Gn^\star$ be the optimally supervised PFSA obtained by
$\nustar$. We note that $\nu_{\#}^i$ is the renormalized
measure of the terminating plant $\Gn^\star(\theta_{min})$
with substochastic event generation probability matrix
$\widetilde{\Pi}^{\theta_{min}} =
(1-\theta_{min})\widetilde{\Pi}$. Denoting the event
generating function (See Definition~\ref{pitildefn}) for
$\Gn^\star$ and $\Gn^\star(\theta_{min})$ as $\tilde{\pi}:Q
\times \Sigma \rightarrow Q$ and $\tilde{\pi}^{\theta_{min}}:Q
\times \Sigma \rightarrow Q$ respectively, we have
{\small
\begin{subequations}
\begin{gather}
\tilde{\pi}^{\theta_{min}}(q_i,\epsilon) = 1\\
%\forall q_i \in Q, \ \sum_{\sigma \in \Sigma}\tilde{\pi}^{\theta_{min}}(q_i,\sigma) = 1 - \theta_{min}\\
\forall q_i \in Q, \sigma_j \in \Sigma, \ \tilde{\pi}^{\theta_{min}}(q_i,\sigma_j) = (1-\theta_{min})\tilde{\pi}(q_i,\sigma_j)
\end{gather}
\end{subequations}
}
\vspace{-15pt}
\begin{notn}
For notational simplicity, we use
\begin{gather*}
\nu_{\#}^i (L(q_i)) = \nu_\#(q_i) = \boldsymbol{\nu_\#} \vert_i \\
\mathrm{where} \ \boldsymbol{\nu_\#} = \theta_{min}  [I-(1-\theta_{min})\Pi^\#]^{-1} \boldsymbol{\chi}
\end{gather*}
\end{notn}
\vspace{0pt}
\begin{defn}[$\nustar$-path]\label{defnustar} A $\nustar$-path $\rho(q_i,q_j)$
 from state $q_i \in Q$ to state $q_j \in Q$ is defined to be
an ordered set of PFSA states $\rho = \{q_{r_1}, \cdots ,
q_{r_M}\}$ with $ q_{r_s} \in Q, \ \forall s \in
\{1,\cdots,M\}, M \leq \Crd(Q)$ such that \vspace{-10pt}
{\small\begin{subequations}
\begin{gather}
q_{r_1} = q_i \\
q_{r_M} = q_j \\
\forall i,j \in \{1,\cdots,M\},\ q_{r_i} \neq q_{r_j} \\
\forall s \in \{1, \cdots , M\},\forall t \leqq s, \ \mm (q_{r_t}) \leqq \mm  (q_{r_s})
\end{gather}
\end{subequations}
}
\end{defn}
\vspace{0pt}
We reproduce without proof the following key results pertaining to $\nustar$- planning as reported
in \cite{CMR08}.
\begin{lem}\label{lemmpathseq}
There exists an enabled sequence of transitions from state
$q_i \in Q \setminus Q_\obs $ to $q_j \in Q \setminus
\{q_\circleddash\}$ in $\Gn^\star$ if and only if there exists
a $\nustar$-path $\rho(q_i,q_j)$ in $\Gn^\star$.
\end{lem}
\vspace{0pt}
\begin{prop}\label{propobstacleavoidance}
For the optimally supervised navigation automaton $\Gn^\star$, we have
\begin{gather*}
\forall q_i \in Q \setminus Q_\obs, \ L(q_i) \subseteqq \Sigma^\star_C
\end{gather*}
\end{prop}
\begin{cor}\label{corobstacleavoidance}
\textbf{(Obstacle Avoidance:) }
There exists no $\nustar$-path from any unblocked state to any blocked state in the optimally supervised navigation automaton $\Gn^\star$.
\end{cor}
\begin{prop}[Existence of $\nustar$-paths]\label{propexisnustar}
There exists a $\nustar$-path $\rho(q_i,q_\Gl )$ from any state $q_i \in Q$ to the goal $q_\Gl \in Q$ if and only if $\mm (q_i)  > 0$.
\end{prop}
\vspace{0pt}
\begin{cor}\label{corollarynolocalmax}\textbf{(Absence of Local Maxima:)}
If there exists a $\nustar$-path from $q_i \in Q$ to $q_j \in Q$ and  a $\nustar$-path from $q_i$ to $q_\Gl$ then there exists a $\nustar$-path from $q_j$ to $q_\Gl$, $i.e.$,
\begin{gather*}
\forall q_i,q_j \in Q \bigg ( \exists \rho_1(q_i,q_\Gl) \bigwedge \exists \rho_2(q_i,q_j) \Rightarrow \exists \rho(q_j,q_\Gl) \bigg )
\end{gather*}
\end{cor}
\vspace{0pt}
%###############################################################################
%###############################################################################
%###############################################################################
\vspace{-13pt}
\subsection{Optimal Tradeoff between Computed
Path Length \& Availability Of Alternate Routes}
% \subsection{Robustness to Map Uncertainty}
\begin{figure}[!htb]
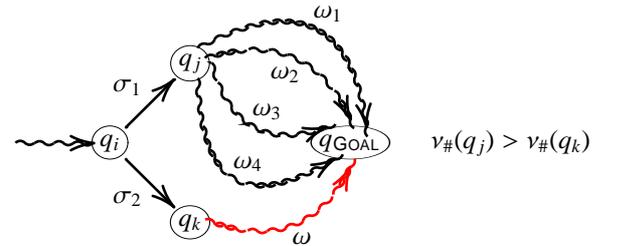

\centering
\begin{minipage}[h!]{3.25in}
\hspace{20pt} \xy 0;/r.25pc/:
(0,0)*+[o][F-]{\txt{$q_i$}}; (10,10)*+[o][F-]{\txt{$q_j$}};
(10,-10)*+[o][F-]{\txt{$q_k$}};
(30,0)*+[o][F-]{\txt{$q_\Gl$}};
{\myar@{{}{~}{>}}(-12,0)*{};(-2,0)*{}};
{\myar@{{}{-}{>}}^{\txt{$\sigma_1$}}(2,2)*{};(8,8)*{}};
{\myar@{{}{-}{>}}_{\txt{$\sigma_2$}}(2,-2)*{};(8,-8)*{}};
{\myar@[red]@{{}{~}{>}}@/_1pc/_{\txt{$\omega$}}(13,-10)*{};(30,-3)*{}};
{\myar@{{}{~}{>}}@/_1pc/^{\txt{$\omega_3$}}(13,9)*{};(27,2)*{}};
{\myar@{{}{~}{>}}@/_2pc/^{\txt{$\omega_4$}}(11,7)*{};(28,-2)*{}};
{\myar@{{}{~}{>}}@/^1pc/_{\txt{$\omega_2$}}(13,11)*{};(30,3)*{}};
{\myar@{{}{~}{>}}@/^2pc/^{\txt{$\omega_1$}}(12,12.5)*{};(32,2)*{}};
(50,0)*{\txt{$\mm(q_j) > \mm(q_k) $}};%
\endxy
\caption{Tradeoff between path-length and robustness under
dynamic uncertainty: $\sigma_2\omega$ is the shortest path to
$q_\Gl$ from $q_i$; but the $\nustar$ plan may be
$\sigma_1\omega_1$ due to the availability of larger number of feasible
paths through $q_j$.}\label{fig03a}
\end{minipage}
\end{figure}
Majority of reported path planning algorithms consider
minimization of the computed feasible path length as the sole
optimization objective. 
% Mobile robotic platforms however
% suffer from  varying degrees of dynamic and parametric
% uncertainties, implying that path length minimization is  of
% lesser practical importance to computing  plans that are
% robust under sensor noise, imperfect actuation and possibly
% accumulating odometry errors.  Even with sophisticated signal
% processing techniques such errors cannot be eliminated.
However, the $\nustar$ algorithm can be shown to achieve an optimal
trade-off between path lengths and availability of feasible
alternate routes. If $\omega$
is the shortest path to goal from state $q_k$, then the
shortest path from state $q_i$ (with $q_i \xrightarrow{\sigma_2} q_k$) is given by $\sigma_2\omega$.
However, a larger number of feasible paths may be  available from
state $q_j$ (with $q_i \xrightarrow{\sigma_1} q_j$) which may result in the optimal $\nustar$ plan to
be $\sigma_1\omega_1$. Mathematically, each feasible path from
state $q_j$ has a positive measure which may sum to be greater
than the measure of the single path $\omega$ from state $q_k$.
The condition $ \mm(q_j)> \mm(q_k)$ would then imply that the
next state from $q_i$ would be computed to be $q_j$ and not
$q_k$. Physically it can be
interpreted  that the mobile gent is better off going to $q_j$
since the goal remains reachable even if one or more paths
become unavailable. The key results~\cite{CMR08} are as follows:
\begin{lem}\label{lemmstringmeas}
For the optimally supervised navigation automaton $\Gn^\star$, we have
$\forall q_i\in Q \setminus Q_\obs$,
\begin{gather*}
\forall \omega \in L(q_i),
\ \nu_{\#}^i (\{\omega \}) =
 \theta_{min} \bigg (\frac{ 1
 - \theta_{min}}{\Crd(\Sigma_C)}
 \bigg )^{\vert \omega \vert } \chi(\delta^\#(q_i,\omega))
\end{gather*}
\end{lem}
\vspace{0pt}
\begin{prop}\label{proptradeoff}
For $q_i \in Q \setminus Q_\obs$, let $q_i
\xrightarrow{\sigma_1} q_j \rightarrow \cdots \rightarrow
q_\Gl$ be the shortest path to the goal. If there exists $q_k
\in Q\setminus Q_\obs$ with $q_i \xrightarrow{\sigma_2} q_k$
for some $\sigma_2 \in \Sigma_C$ such that
$
%
%\begin{gather*}
\mm(q_k) > \mm(q_j)
$,
%
%\end{gather*}
%
then the number of distinct paths to goal from state $q_k$ is at least $\Crd(\Sigma_C) +1$.
\end{prop}
The lower bound computed in Proposition~\ref{proptradeoff} is
not tight and  if the alternate paths are longer or if there
are multiple 'shortest' paths then the number of alternate
routes required is significantly higher. Detailed
examples can be easily presented to
illustrate situation where $\nustar$ opts for a longer but
more robust plan.
\section{Generalizing The Navigation Automaton To Accommodate Uncertain Execution}\label{secmodel}
In this paper, we modify the PFSA-based navigation model to explicitly
reflect uncertainties arising from imperfect localization and  the dynamic response of the platform to
navigation commands. These  effects  manifest
as uncontrollable transitions in the navigation automaton as illustrated in Figure~\ref{Fig0333}. Note, while
in absence of uncertainties and dynamic effects, one can disable transitions perfectly, in the modified model, such disabling is only partial.
Choosing the probabilities of the uncontrollable transitions correctly allows the model to incorporate
physical movement errors and sensing noise in an amortized fashion. 

A sample run with a SEGWAY RMP at NRSL is shown in Figure~\ref{figplan}. Note that the robot is 
unable to follow the plan exactly due to cellular discretization and dynamic effects. Such effects can be conceptually modeled
by decomposing trajectory fragments into sequential combinations
of controllable and uncontrollable inter-cellular moves as illustrated in Figure~\ref{Fig0333}(c). We do not need to actually decompose
trajectories, it is merely a conceptual construct that gives us a theoretical basis 
for computing the  probabilities of uncontrollable transitions from observed robot dynamics (as described later in Section~\ref{secuncertain}, and therefore
incorporate the amortized effect of uncertainties in the navigation automaton.
%##############################################
%##############################################
%##############################################
\begin{figure}[!ht]
\centering
\includegraphics[width=2in]{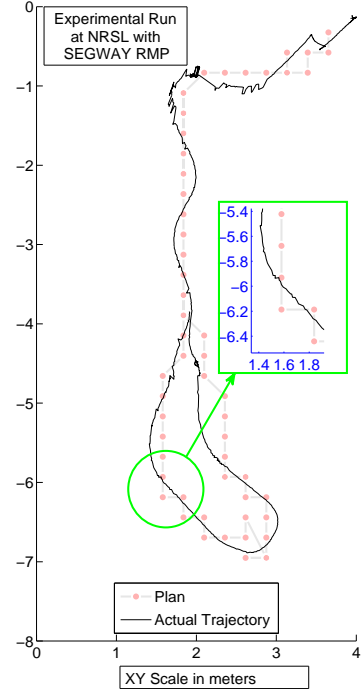}
\caption{Plan execution with SEGWAY RMP at NRSL, Pennstate}\label{figplan}
\end{figure}
%##############################################
%##############################################
%##############################################
%##############################################
\begin{figure}[!ht]
\subfigure[]{
\xy 0;/r.2pc/:
%states
(0,0)*+[o][F-]{\txt{\textbf{C}}};
(20,0)*+[o][F-]{\txt{\textbf{R}}};
(-20,0)*+[o][F-]{\txt{\textbf{L}}};
(0,20)*+[o][F-]{\txt{\textbf{F}}};
(14,14)*+[o][F-]{\txt{\textbf{RF}}};
(0,-20)*+[o][F-]{\txt{\textbf{B}}};
(-14,14)*+[o][F-]{\txt{\textbf{LF}}};
(14,-14)*+[o][F-]{\txt{\textbf{RB}}};
(-14,-14)*+[o][F-]{\txt{\textbf{LB}}};
%arrrows
{\myar@{->}_{\txt{$e_1$}}(0,2)*{};(0,18)*{}};
{\myar@{->}_{\txt{$e_5$}}(0,-2)*{};(0,-18)*{}};
{\myar@{->}_{\txt{$e_3$}}(2,0)*{};(18,0)*{}};
{\myar@{->}_{\txt{$e_7$}}(-2,0)*{};(-18,0)*{}};
{\myar@{->}_{\txt{$e_2$}}(2,2)*{};(12,12)*{}};
{\myar@{->}_{\txt{$e_4$}}(2,-2)*{};(12,-12)*{}};
{\myar@{->}_{\txt{$e_6$}}(-2,-2)*{};(-12,-12)*{}};
{\myar@{->}_{\txt{$e_8$}}(-2,2)*{};(-12,12)*{}};
\endxy
}
%loops
\subfigure[]{
\xy 0;/r.2pc/:
%states
(0,0)*+[o][F-]{\txt{\textbf{C}}};
(20,0)*+[o][F-]{\txt{\red\textbf{R}}};
(-20,0)*+[o][F-]{\txt{\textbf{L}}};
(0,20)*+[o][F-]{\txt{\textbf{F}}}; (14,14)*+[o][F-]{\txt{\red
\textbf{RF}}}; (0,-20)*+[o][F-]{\txt{\red\textbf{B}}};
(-14,14)*+[o][F-]{\txt{\textbf{LF}}};
(14,-14)*+[o][F-]{\txt{\red\textbf{RB}}};
(-14,-14)*+[o][F-]{\txt{\red\textbf{LB}}};
%arrrows
{\myar@{->}_{\txt{$e_1$}}(0,2)*{};(0,18)*{}};
{\myar@{->}_{\txt{$e_7$}}(-2,0)*{};(-18,0)*{}};
{\myar@{->}_{\txt{$e_8$}}(-2,2)*{};(-12,12)*{}};
%loops
(1,2)*{}="X1"; (1.75,-2)*{}="X2"; "X1";"X2"
**[thicker][red]\crv{(3,7) & (15.5,0) & (3,-7)};
{\myar@[red]@{{}{}{>}}(2.75,-3.5)*{};(1.25,-2)*{}};
% (0,-30)*{\txt{(b)}};
(5,-10)*{\txt{$e_2,e_3,e_4,$\\ $
e_5,e_6$}};\endxy
}\\
\subfigure[]{
\psfrag{J0}[tl]{$J_0$}
\psfrag{J1}{}
\psfrag{J2}{}
\psfrag{J3}{}
\psfrag{J4}[tc]{$J_4$}
\psfrag{J5}{}
\psfrag{J6}{}
\psfrag{J7}[tl]{$J_7$}
\psfrag{J8}{}
\psfrag{T}[cc]{\small \red Trajectory}
\psfrag{C}[cc]{\small \txt{Controllable\\ Move}}
\psfrag{U}[cc]{\small \txt{Uncontrollable \\ Move}}
\includegraphics[width=1.25in]{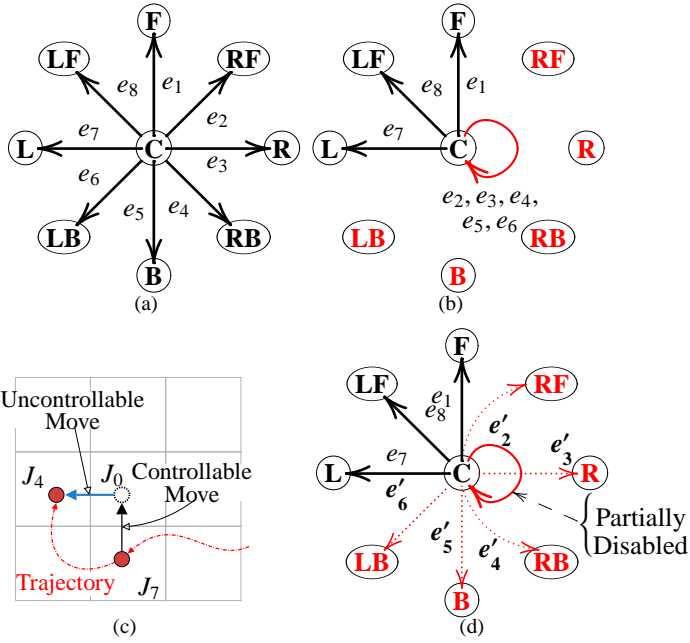}
}
\subfigure[]{
\xy 0;/r.2pc/:
(0,20)*{\xy
%states
(0,0)*+[o][F-]{\txt{\textbf{C}}};
(20,0)*+[o][F-]{\txt{\red\textbf{R}}};
(-20,0)*+[o][F-]{\txt{\textbf{L}}};
(0,20)*+[o][F-]{\txt{\textbf{F}}}; (14,14)*+[o][F-]{\txt{\red
\textbf{RF}}}; (0,-20)*+[o][F-]{\txt{\red\textbf{B}}};
(-14,14)*+[o][F-]{\txt{\textbf{LF}}};
(14,-14)*+[o][F-]{\txt{\red\textbf{RB}}};
(-14,-14)*+[o][F-]{\txt{\red\textbf{LB}}};
%arrrows
% {\ar@[red]@{->}_{\txt{$e_1$}}(0,2)*{};(0,18)*{}};
{\myarL@[red]@{.>}_{\txt{$\boldsymbol{e_5'}$}}(0,-2)*{};(0,-18)*{}};
{\myarL@[red]@{.>}^{\txt{$\phantom{XXXX}\boldsymbol{e_3'}$}}(2,0)*{};(18,0)*{}};
% {\ar@[red]@{->}_{\txt{$e_7$}}(-2,0)*{};(-18,0)*{}};
{\myarL@/^0.5pc/@[red]@{.>}_{\txt{$\phantom{}\mspace{0mu}\boldsymbol{e_2'}$\\$\phantom{X}$}}(0,2)*{};(10,14)*{}};
{\myarL@/_0.5pc/@[red]@{.>}_{\txt{$\mspace{100mu}\phantom{XXX}\boldsymbol{e_4'}$}}(0,-2)*{};(12,-12)*{}};
{\myarL@[red]@{.>}_{\txt{$\boldsymbol{e_6'}$}}(-2,-2)*{};(-12,-12)*{}};
% {\ar@[red]@{->}_{\txt{$e_8$}}(-2,2)*{};(-12,12)*{}};
%arrrows
{\myar@{->}^{\txt{$e_1$\\$\phantom{X}$}}(0,2)*{};(0,18)*{}};
{\myar@{->}_{\txt{$e_7$}}(-2,0)*{};(-18,0)*{}};
{\myar@{->}_{\txt{$e_8$}}(-2,2)*{};(-12,12)*{}};
%loops
(1,2)*{}="X1"; (1.75,-2)*{}="X2"; "X1";"X2"
**[thicker][red]\crv{(3,7) & (15.5,0) & (3,-7)};
{\myar@[red]@{{}{}{>}}(2.75,-3.5)*{};(1.25,-2)*{}};
{\ar@{-->}(18,-7.5)*{};(8,-3)*{}};
% (16,-10)*{}="X1"; (8,-3)*{}="X2"; "X1";"X2"
% **[thicker]\crv{(13,-5) & (5,-5)};
% (8,-3)*{\medbullet};
(27,-8)*{\left \{ \textrm{\txt{Partially \\ Disabled }} \right. };
\endxy};\endxy
}
% \end{minipage}
\caption{(a) shows available moves from the current state (C)
in unsupervised navigation automaton. (b) shows the enabled
moves in the optimally supervised PFSA with no dynamic uncertainty, (c) illustrates the case with dynamic uncertainty, so that
the robot can still uncontrollably (and hence unwillingly) make the disabled transitions, albeit with a small probability, $i.e.$,
probability of transitions $e_2',e_3',e_4'$ etc. is small. (d) illustrates the concept of using uncontrollable transitions
to model dynamical response for a 2D circular robot: $J_0$ is the target cell from $J_7$, while the actual trajectory of the 
robot (shown in dotted line) ends up in $J_4$. We can model this trajectory fragment as first executing a controllable move to $J_0$ and then
uncontrollably moving to $J_4$.}\label{Fig0333}
\vspace{0pt}
\end{figure}
%##############################################
%##############################################
%##############################################
\subsection{The Modified Navigation Automaton}
 The modified navigation automaton $\Gnm = (Q,\Sigma,\delta,\widetilde{\Pi},\chi)$ is defined similar to the
formulation in Section~\ref{sectionformulation}, with the exception that the alphabet $\Sigma$ is defined as follows:
\begin{gather}
 \Sigma = \Sigma_C \cup \Sigma_{UC} \cup \{ u \}
\end{gather}
where $\Sigma_C$ is the set of controllable moves from any unblocked navigation state (as before), while $\Sigma_{UC}$ is the 
set of uncontrollable transitions that can occur as an effect of the platform dynamics and oather uncertainty effects. We assume that for each $\sigma \in \Sigma_C$, we have 
a corresponding event $\sigma_u$ in $\Sigma_{UC}$, such that both $\sigma$ and $\sigma_u$ represent the same physical move
from a given navigation state; but while $\sigma$ is controllable and may be disabled, $\sigma_u$ is uncontrollable. 
 Although for 2D circular robots we have: $\Crd(\Sigma_C) = \Crd(\Sigma_{UC})$, in general, there can exist uncontrollable moves reflecting estimation errors that cannot be 
realized via a single controllable move. For example, for planar rectangular robots with a non-zero minimum turn radius, there can be an uncontrollable shift in the heading without any change in the 
$xy$-positional coordinates, which may reflect errors in heading estimation, but such a move cannot be executed via controllable transitions due to the restriction on the minimum turn radius. We will discuss these issues in more details in the sequel.

\begin{defn}\label{defgamma}
The coefficient of dynamic deviation $\gamma(\Gnm)$ is defined as follows:
% \begin{gather}
%  \gamma(\Gnm) = 1 - \frac{1}{\Crd(Q)} \sum_{q_i \in Q}\sum_{\sigma_u \in \Sigma_{UC}} \tilde{\pi}(q_i,\sigma_u)
% \end{gather}
\begin{gather}
 \gamma(\Gnm) = 1 - \max_{q_i \in Q}\sum_{\sigma_u \in \Sigma_{UC}} \tilde{\pi}(q_i,\sigma_u)
\end{gather}
\end{defn}
\begin{defn}\label{defpitildenavMOD}
The event generation probabilities for $\Gnm$ is defined as follows: $\forall q_i \in Q\setminus\{q_\Gl\},\sigma_j
\in \Sigma$, 
% \begin{align}
% \tilde{\pi}(q_i,\sigma_j) & = \left \{ \begin{array}{cl}
% \frac{\gamma(\Gnm)}{\mathrm{No. \ of \ controllable \ events \ at \ } q_i}, & \mathrm{if} \ \sigma_j \in \Sigma_C \\
% \tilde{\pi}_{AV}(\sigma_j), & \mathrm{if} \ \sigma_j \in \Sigma_{UC} \\
% 1, & \mathrm{otherwise}
% \end{array}
% \right. \notag
% \end{align}
% and for the goal, we define as before:
% \begin{align}
% \tilde{\pi}(q_\Gl,\sigma_j) & = \left \{ \begin{array}{cl}
% \frac{1}{\mathrm{No. \ of \ controllable \ events \ at \ } q_i}, & \mathrm{if} \ \sigma_j \in \Sigma_C \\
% 1, & \mathrm{otherwise}
% \end{array}
% \right. \notag
% \end{align}
\begin{align}
\tilde{\pi}(q_i,\sigma_j) & = \left \{ \begin{array}{cl}
\frac{ 1 - \sum_{\sigma_u \in \Sigma_{UC}} \tilde{\pi}(q_i,\sigma_u) }{\mathrm{No. \ of \ controllable \ events \ at \ } q_i}, & \mathrm{if} \ \sigma_j \in \Sigma_C \\
\tilde{\pi}(q_i,\sigma_j), & \mathrm{if} \ \sigma_j \in \Sigma_{UC} \\
1, & \mathrm{otherwise}
\end{array}
\right. \notag
\end{align}
and for the goal, we define as before:
\begin{align}
\tilde{\pi}(q_\Gl,\sigma_j) & = \left \{ \begin{array}{cl}
\frac{1}{\mathrm{No. \ of \ controllable \ events \ at \ } q_i}, & \mathrm{if} \ \sigma_j \in \Sigma_C \\
1, & \mathrm{otherwise}
\end{array}
\right. \notag
\end{align}
Note that we assume there is no uncontrollability at the goal. This assumption is made for technical reasons clarified in the sequel
and also to reflect the fact that once we reach the goal, we terminate the mission and hence such effects can be neglected.
\vspace{3pt}
\end{defn}
We note the following:
\begin{itemize}
 \item In the idealized case where we assume platform dynamics is completely absent, we have $\tilde{\pi}(q_i,\sigma_u)=0, \forall q_i \in Q,\forall \sigma_u \in \Sigma_{UC}$ implying that $\gamma(\Gnm)=1$, while in practice, we expect $\gamma(\Gnm)<1$.
\item In Definition~\ref{defgamma}, we allowed for the possibility of $\tilde{\pi}(q_i,\sigma_u)$ being dependent on the particular navigation states $q_i \in Q$.
A significantly simpler approach would be to redefine the probability of the
uncontrollable events $\tilde{\pi}(q_i,\sigma_u)$ as follows:
\begin{gather}
 \tilde{\pi}_{AV}(\sigma_u) = \frac{1}{\Crd(Q)} \sum_{q_i \in Q}\tilde{\pi}(q_i,\sigma_u)
\end{gather}
where $\tilde{\pi}_{AV}(\sigma_u)$ is the average probability of the uncontrollable event $\sigma_u$ being generated. 
\end{itemize}
The averaging of the probabilities of uncontrollable transitions is 
justified in situations where we can assume that  the dynamic response of the platform is not dependent on the location of the platform in the workspace.
In this simplified case, the event generation probabilities for $\Gnm$ can be stated as: $\forall q_i \in Q\setminus\{q_\Gl\},\sigma_j
\in \Sigma$, 
\begin{align}
\tilde{\pi}(q_i,\sigma_j) & = \left \{ \begin{array}{cl}
\frac{\gamma(\Gnm)}{\mathrm{No. \ of \ controllable \ events \ at \ } q_i}, & \mathrm{if} \ \sigma_j \in \Sigma_C \\
\tilde{\pi}_{AV}(\sigma_j), & \mathrm{if} \ \sigma_j \in \Sigma_{UC} \\
1, & \mathrm{otherwise}
\end{array}
\right. \notag
\end{align}
% and for the goal, we define as before:
% \begin{align}
% \tilde{\pi}(q_\Gl,\sigma_j) & = \left \{ \begin{array}{cl}
% \frac{1}{\mathrm{No. \ of \ controllable \ events \ at \ } q_i}, & \mathrm{if} \ \sigma_j \in \Sigma_C \\
% 1, & \mathrm{otherwise}
% \end{array}
% \right. \notag
% \end{align}
% 
% 
% A mobile robot operating in a hilly region with varying inclination is an example where the dynamic response is dependent on the navigation states. For simplicity, however, 
% we limit our analysis in this paper to cases where we can use the above assumption. 
The key difficulty is allowing the aforementioned dependence on states  is not the decision optimization that would follow, but the complexity of  identifying the 
probabilities; averaging results in significant simplification as shown in the sequel. Thus, even if we cannot realistically 
average out the uncontrollable transition probabilities over the entire state space, we could decompose the workspace to identify subregions where such an assumption is locally valid.
In this paper, we do not address formal approaches to such decomposition, and will generally assume that the afore-mentioned averaging is valid throughout the workspace; the explicit 
identification of the sub-regions is more a matter of implementation specifics, and has little to do with the details of the planning algorithm presented here, and hence  will be discussed elsewhere.
% 
% Note that in specifying the event generation probabilities, we assume that $\tilde{\pi}_{AV}(\sigma_j)$ is already known.
In Section~\ref{secuncertain}, we will address the computation of the probabilities of uncontrollable transitions from observed dynamics.
First, we will establish the main planning algorithm as a solution to the performance optimization of the navigation automaton in the next section.
\section{Optimal Planning Via Decision Optimization Under Dynamic Effects}\label{secoptplan}
The modified model $\Gnm$ can be optimized via the measure-theoretic technique in a straightforward manner, using
the $\nustar$-algorithm reported in \cite{CMR08}. The presence of uncontrollable transitions in $\Gnm$ poses no problem (as far as the automaton optimization is concerned), since
the underlying measure-theoretic optimization is already capable of handling such effects~\cite{CR07}.
However the presence of uncontrollable transitions weakens  some of the theoretical results  obtained
in \cite{CMR08} pertaining to navigation, specifically the absence of local maxima. We show that this causes the $\nustar$ planner to lose some of its  crucial advantages, and therefore must be explicitly 
addressed via a recursive decomposition of the planning problem.
% Furthermore, large problem sizes give rise
% to critical issues due to partial controllability of transitions in presence of dynamic uncertainty, which would be addressed in the next section.
\begin{prop}[Weaker Version of Proposition~\ref{propexisnustar}]\label{propweak}
There exists a $\nustar$-path $\rho(q_i,q_\Gl )$ from any state $q_i \in Q$ to the goal $q_\Gl \in Q$ if $\mm (q_i)  > 0$.
\end{prop}
\begin{proof}
We note that $\mm (q_i)  > 0$ implies that there necessarily exists at least one string $\omega$
of positive measure initiating from $q_i$ and hence there exists at least one string
 that
terminates on $q_\Gl$. The proof then follows from the definition of $\nustar$-paths (See Definition~\ref{defnustar}).
\end{proof}
\begin{rem}
 Comparing with Proposition~\ref{propexisnustar}, we note that the \textit{only if} part of the result is lost in the modified case.
\end{rem}

\begin{rem}
We note that under the modified model, $\mm(q_i)<0$ needs to be interpreted somewhat differently. In absence of any dynamic uncertainty, 
$\mm(q_i)<0$ implies that no path to goal exists. However, due to weakening of Proposition~\ref{propobstacleavoidance} (See Proposition~\ref{propweak}), 
% and in the light 
% of Proposition~\ref{propp1p2},
$\mm(q_i)<0$ implies that the measure  of the set of strings  reaching the goal is smaller to that of the set of strings  hitting an obstacle from the state $q_i$.
\end{rem}
The $\nustar$-planning algorithm is based on several abstract concepts such as 
the navigation automaton and the formal measure of symbolic strings. 
It is important to realize that in spite of the somewhat elaborate framework presented here, $\nustar$-optimization is
free from heuristics, which is often not the case with competing approaches.
In this light, the next proposition is critically important as it elucidates 
this concrete physical connection. 
%##################################################################
%##################################################################
%##################################################################
\begin{figure}[!ht]
 \centering
\psfrag{N}[l]{\sffamily \small No Connecting Path}
\psfrag{G}[l][c]{\small \bf  Goal}
\psfrag{A}[c]{\small \bf A}
\psfrag{B}[c]{\small \bf B}
\psfrag{O}[l]{\small \bf Obstacles}
\includegraphics[width=2in,height=1.75in]{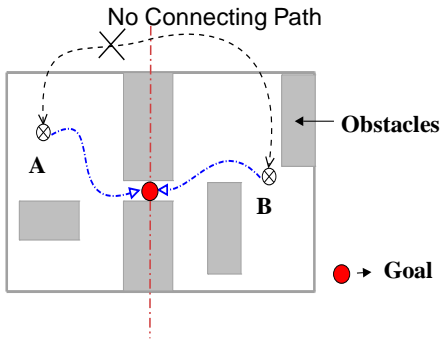}
\caption{Absence of uncontrollable transitions at the goal imply that there is no path in the optimally disabled navigation automaton from
point A to point B (or vice versa), since all controllable transitions will necessarily be disabled at the goal. It follows that the stationary probability vector
may be different depending on whether one starts left or right to the goal. However, note that that  any two points on the same side have a path (possibly made of uncontrollable transitions)
between them; implying that the stationary probability vector will be identical if either of them is chosen as the start locations. }\label{figreducible}
\end{figure}
%##################################################################
%##################################################################
%##################################################################
\begin{prop}\label{propoptinterpret}
 Given that a feasible path exists from the starting state to the goal, the $\nustar$ planning algorithm under non-trivial dynamic uncertainty ($i.e.$ with $\gamma(\Gnm) < 1$) maximizes 
the probability $\wp_\Gl$ of reaching the goal while simultaneously minimizing the
probability $\wp_\circleddash$ of hitting an obstacle.
\end{prop}
\begin{proof}
 Let $\wp$ be the stationary probability vector for the stochastic transition probability matrix corresponding to the 
navigation automaton $\Gnm$, for a starting state from which a feasible path to goal exists. 
(Note that $\wp$ may depend on the starting state; Figure~\ref{figreducible} illustrates one such example. However, once we fix a particular starting state, the stationary vector 
$\wp$ is uniquely determined).
The selective disabling of controllable events modifies the transition matrix and in effect
alters $\wp$, such that $\wp^T \chi$ is maximized~\cite{CR07}, where $\boldsymbol{\chi}$ is the characteristic weight vector, $i.e.$ , $\boldsymbol{\chi}_i = \chi(q_i)$.
Recalling that $\chi(q_\Gl) = 1,\chi(q_\circleddash) = -1$ and $\chi(q_i) = 0$ if $q_i$ is neither the goal nor the abstract obstacle state $q_\circleddash$, 
we conclude that the optimization, in effect, maximizes the quantity:
\begin{gather}
 \psi = \wp_\Gl - \wp_\circleddash
\end{gather}
Also, note that the optimized navigation automaton has only two dump states, namely the goal $q_\Gl$ and the abstract obstacle state $q_\circleddash$. 
That the goal $q_\Gl$ is in fact a dump state is ensured by not having uncontrollable transitions at the goal (See Definition~\ref{defpitildenavMOD}). Hence
we must have 
\begin{gather}
 \forall q_i \in Q\setminus \{q_\Gl,q_\circleddash\}, \ \wp_i = 0
\end{gather}
implying that 
\begin{subequations}
\begin{gather}
 \wp_\Gl + \wp_\circleddash = 1 \\
\Longrightarrow \psi = 2\wp_\Gl - 1 = 1 - 2\wp_\circleddash
\end{gather}
\end{subequations}
Hence it follows that the optimization maximizes $\wp_\Gl$ and simultaneously minimizes $\wp_\circleddash$.
\end{proof}
\vspace{3pt}
\begin{rem}\label{remvalid}
It is easy to see that Proposition~\ref{propoptinterpret} remains valid if $\chi(q_\Gl) = \chi_\Gl > 1$. In fact, the 
result remains valid  as long as the characteristic weight of the goal is positive and the characteristic weight 
of the abstract obstacle state $q_\circleddash$ is negative.
\end{rem}
\subsection{Recursive Problem Decomposition For Maxima Elimination}
Weakening of Proposition~\ref{propobstacleavoidance} (See Proposition~\ref{propweak})
has the crucial consequence that Corollary~\ref{corollarynolocalmax} is no longer valid.
Local maxima can occur under the modified model. This is a serious problem
for autonomous planning and must be remedied.
The problem becomes critically important when applied to solution of 
mazes; larger the number of obstables, higher is the chance of ending up in 
a local maxima.
While elimination of local maxima  is notoriously difficult for 
potential based planning approaches, $\nustar$ can be
modified with ease into a recursive scheme that yields maxima-free plans 
 in models with non-zero dynamic 
effects ($i.e.$ with $\gamma(\Gnm) < 1$). 

% The procedure is depicted in a flowchart in Figure~\ref{figflow}.
It will be shown in the sequel that for successful execution of the algorithm, we may need to assign a larger than unity 
characteristic weight $\chi_\Gl$ to the goal $q_\Gl$. A sufficient lower bound for $\chi_\Gl$, with possible dependence on the
recursion step, is given in
Proposition~\ref{proprecurs}.
The basic recursion scheme can be described as follows (Also see the flowchart illustration in Algorithm~\ref{figflow}):
\begin{enumerate}
 \item In the first step ($i.e.$, at recursion step $k=1$) we execute $\nustar$-optimization on the given navigation automaton $\Gnm$
and obtain the measure vector $\boldsymbol{\nu_\#}^{[k]}$.
\item We denote the set of states with strictly positive measure as $Q_k$ ($k$ denotes the recursion step), $i.e.$,
\begin{gather}\label{eqqk}
 Q_k = \{  q_i \in Q : \boldsymbol{\nu_\#}^{[k]}\vert_i > 0 \}
\end{gather}
\item If $Q_k = Q_{k-1}$, the recursion terminates; else
we update the characteristic weights as follows:
\begin{gather}
 \forall q_i \in Q_k, \ \chi(q_i) = \chi_\Gl^{[k]}
\end{gather}
and continue the recursion by going back to the first step
and incrementing the step number $k$.
\end{enumerate}

%##################################################################
%##################################################################
%##################################################################
\begin{algorithm}[!ht]
\restylealgo{plain}
\centering
\psfrag{P}[c]{\sffamily \bf \small \txt{\color{OrangeRed4} Problem: \\ $\Gnm$}}
\psfrag{E}[c]{\sffamily \bf \small \color{Green4} Execute $\nustar$}
\psfrag{C}[c]{\sffamily \bf \small $Q_k = Q_{k-1} ?$}
\psfrag{T}[c][c]{\sffamily \bf \small Terminate}
\psfrag{K}[l][c]{\sffamily \bf \small Set $k = k+1$}
\psfrag{S}[c]{\sffamily \bf \small \txt{Save vector \\  $\boldsymbol{\nu_\#}^{[k]}$}}
\psfrag{G}[c]{\sffamily \bf \small \txt{ 1. Set: $\forall q_i \in Q_k$\\ $\chi(q_i) = \chi_\Gl^{[k]} $\\ $\phantom{-}$2. Eliminate Uncont.
\\ transitions from \\all $q_i \in Q_k$}}
\psfrag{Y}[c][c]{\sffamily \bf \small \color{OrangeRed4} No}
\psfrag{N}[c][c]{\sffamily \bf \small \Mblue Yes}
\psfrag{Q}[cc]{\sffamily \bf \small \txt{$\phantom{XXXXXXXXX}$Define: \\$\phantom{XXXXXXXXXXX}$ $Q_k=\{  q_i \in Q : \boldsymbol{\nu_\#}^{[k]}\vert_i > 0 \}$}}
\psfrag{I}[l]{\sffamily \bf \small \txt{Initialize: \\ $k = 0, \ Q_0 = \{q_\Gl\}$}}
 \includegraphics[width=3in]{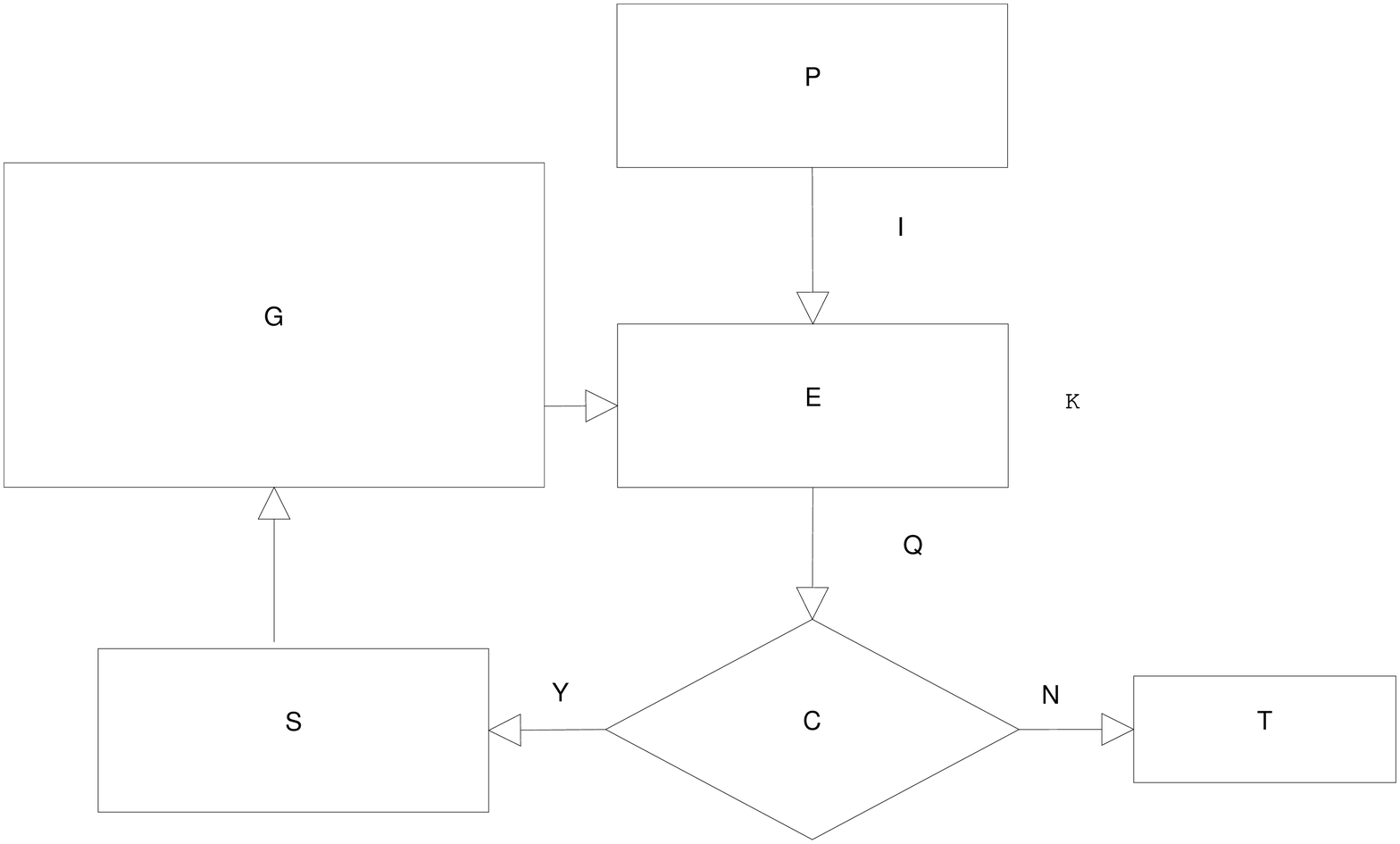}
\caption{Flowchart for  recursive $\nustar$-planning}\label{figflow}
% \Begin{}
\end{algorithm}
%##################################################################
%##################################################################
%##################################################################
\begin{prop}\label{proprecurs}
If $\theta_{min}^{[k]}$ is the critical termination probability (See Definition~\ref{defthetamin}) for the $\nustar$-optimization in the $k^{th}$ recursion step of 
Algorithm~\ref{figflow}, then
the following condition
% \begin{gather}\label{eqcond1}
% \chi_\Gl^{[k]} > \frac{\theta_{min}^{[k]}\Crd(\Sigma_C)}{\gamma\left (1-\theta_{min}^{[k]}\right ) \left ( 1 - (1-\theta_{min}^{[k]})(1 - \gamma) \right )}
% \end{gather}
\begin{gather}\label{eqcond1}
\chi_\Gl^{[k]} > \frac{\Crd(\Sigma_C)}{1-\theta_{min}^{[k]}}\left (  \frac{1}{\gamma} - 1 \right )
\end{gather}
is sufficient to guarantee that the following statements are true:
\begin{enumerate}
\item If there exists a state $q_i \in Q\setminus Q_{k}$ from which 
at least one state $q_\ell \in Q_k$
% the goal $q_\Gl$ 
is reachable in one hop, then
$\boldsymbol{\nu_\#}^{[k]}\vert_i > 0$.
 \item The recursion terminates in at most $\Crd(Q)$ steps.
\item For the $k^{th}$ recursion step, either $Q_k \supsetneqq Q_{k-1}$ or  no feasible path exists to $q_\Gl$ from any state $q_i \in Q \setminus Q_{k-1}$.
\end{enumerate}
\end{prop}
%##################################################################
%##################################################################
%##################################################################
\begin{figure}[!ht]
 \centering
\psfrag{G}[c][c]{\sffamily \bf \small $q_\Gl$}
\psfrag{q}[c][c]{\sffamily \bf \small $q_i$}
\psfrag{p1}[c][t\trianglerighteq]{\sffamily \bf \small $1-\theta_{min}$}
\psfrag{p2}[c][c]{\sffamily \bf \small $ (1-\gamma)(1-\theta_{min})$}
\psfrag{p3}[c][r]{\sffamily \bf \small $(1-\theta_{min})\frac{\gamma}{\Crd(\Sigma_C)}$}
\psfrag{p4}[c][c]{\sffamily \bf \small }
\psfrag{a}[t][c]{\sffamily \bf \small $q_j$}
\psfrag{b}[t][c]{\sffamily \bf \small $q_k$}
\psfrag{O}[c][b]{\sffamily \bf \small $q_\circleddash$ }
\includegraphics[width=2.25in]{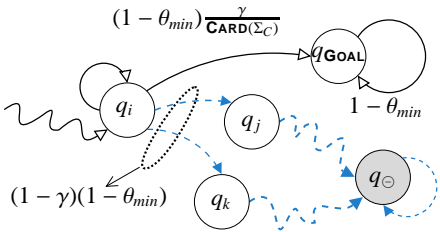}
\caption{Illustration for Proposition~\ref{proprecurs}. Uncontrollable events and strings are shown in dashed.}
\label{figrecur}
\end{figure}
%##################################################################
%##################################################################
%##################################################################
\begin{proof}
\textbf{Statement 1:}\\
We first consider the first recursion step, $i.e.$, the case where $k=0$ and $Q_k = \{q_\Gl\}$ (See Algorithm~\ref{figflow}). 
We note that the goal $q_\Gl$ achieves the maximum measure on account of the fact that only $q_\Gl$ has a positive 
characteristic weight, $i.e.$,  we have
\begin{gather}
 \forall q_i \in Q, \ \boldsymbol{\nu_\#}^{[1]}\vert_\Gl \geqq \boldsymbol{\nu_\#}^{[1]}\vert_i
\end{gather}
It follows that all controllable transitions from the goal will be disabled in the 
optimized navigation automaton obtained at the end of the first recursion step (See Definition~\ref{contapp} and Algorithms~\ref{Algorithm02} \& \ref{Algorithm01}), 
which in turn implies that the non-renormalized measure of the goal (at the end of the first recursion step)
is given by $\chi_\Gl \frac{1}{\theta_{min}}$.

The Hahn Decomposition Theorem~\cite{R88}, allows us to write:
\begin{gather}
 \boldsymbol{\nu_\#}\vert_i = \nu_\# (L^+(q_i))+\nu_\# (L^-(q_i)) 
\end{gather}
 where $L^+(q_i),L^-(q_i)$ are the sets of strings initiating from state $q_i$ that have positive and negative measures respectively. 

Let $q_i \in Q\setminus \{q_\Gl\}$ such that $q_\Gl$
is reachable from $q_i$ in one hop.
We note that since it is possible to reach the goal in one hop from $q_i$, we have:
\begin{gather}\label{eqpos}
 \nu_\#(L^+(q_i)) \geqq \theta_{min} \times \frac{\gamma(1-\theta_{min})}{\Crd(\Sigma_C)} \times \frac{\chi_\Gl}{\theta_{min}} 
\end{gather}
where the first term arises due to renormalization (See Definition~\ref{defrenormmeas}), 
the second term denotes the probability of the transition leading to the goal and  the third term is the non-renormalized measure of the goal itself (as argued above). 
Since it is obvious that the goal achieves the maximum measure, the transition to the goal will obviously be enabled in the optimized automaton, which
justifies the second term. It is clear that there are many more strings of positive measure
($e.g.$ arising due to the self loops at the state $q_i$ that correspond to the disabled controllable events that do not transition to the goal from $q_i$)
which are not considered in the above inequality (which contributes to making the left hand side even larger); 
therefore guaranteeing the correctness of the lower bound stated in Eq.\eqref{eqpos}.

Next, we compute a lower bound for $\nu_\#(L^-(q_i))$. To that effect, we consider an automaton $G'$ identical to 
the navigation automaton at hand in ever respect, but the fact that the $q_\Gl$ has zero characteristic.
We denote the state corresponding to $q_i$ in this hypothesized automaton as $q_i'$, and the set of al states in $G'$ as $Q'$.
We claim that, after a measure-theoretic optimization ($i.e.$ after applying Algorithms~\ref{Algorithm02} and \ref{Algorithm01}), the measure of $q_i'$, denoted as $\nu^\star(q_i')$, 
satisfies:
\begin{gather}\label{eqclaim}
 \nu^\star(q_i') \geqq -(1 - \gamma)
\end{gather}
To prove the claim in Eq.~\eqref{eqclaim}, we first note that denoting the renormalized measure vector for $G'$ before any optimization as 
$\boldsymbol{\nu'}$, the characteristic vector as $\chi'$ and  for any  termination probability $\theta \in (0,1)$, we have:
\begin{multline}\label{eq20}
 \vert \vert \boldsymbol{\nu'} \vert \vert_\infty = \vert \vert \theta [ \mathbb{I} - (1-\theta)\Pi ]^{-1} \chi' \vert \vert_\infty \\
\leqq \vert \vert \theta [ \mathbb{I} - (1-\theta)\Pi ]^{-1} \vert \vert_\infty \times 1 = 1
\end{multline}
which follows from the following facts:
\begin{enumerate}
\item  For all $\theta \in (0,1]$, $\theta [ \mathbb{I} - (1-\theta)\Pi ]^{-1}$ is a row-stochastic matrix  and therefore has unity infinity norm~\cite{CR06}
\item   $\vert \vert \chi' \vert \vert_\infty = 1$, since all entries of $\chi'$ are $0$ except for the state corresponding to the obstacle state in the
navigation automaton, which has a characteristic of $-1$.
\end{enumerate}
Since the only non-zero characteristic is $-1$, it follows that no state in $G'$ can have a positive measure and we conclude from Eq.~\eqref{eq20} that:
\begin{gather}
\forall  q_j' \in Q', \ \nu(q_j') \in [-1,0]
\end{gather}
Note that  $q_i'$ is not blocked itself (since we chose $q_i$ such that a feasible 1-hop path to the goal exists from $q_i$). 
Next, we subject $G'$ to the measure-theoretic optimization (See Algorithms~\ref{Algorithm02} \& \ref{Algorithm01}), which
disables all controllable transitions to the blocked states. In order to compute a lower bound on the optimized measure 
for the state $q_i'$, (denoted by $\nu^\star(q_i')$ ), we consider the worst case scenario where all neighboring states that can be reached from $q_i'$
in single hops are blocked. Denoting the set of all such neighboring states of $q_i$ by  $\mathcal{N}(q_i')$, we have:
\begin{gather}\label{eqlb}
 \nu^\star(q_i') = \sum_{q_j' \in \mathcal{N}(q_i')} \Pi^u_{ij}\nu(q_j') \geqq -1 \times \sum_{q_j' \in \mathcal{N}(q_i')} \Pi^u_{ij} = -1\times(1-\gamma)
\end{gather}
where $\Pi^u_{ij}$ is the probability of the uncontrollable transition from $q_i'$ to the neighboring state $q_j'$. 
Note that we can write Eq.~\eqref{eqlb} in the worst case scenario where each state in $\mathcal{N}(q_i')$ is blocked, since all controllable transitions from $q_i'$ will be disabled in the
optimized plant under such a scenario, and only the uncontrollable transitions will remain enabled; and the  probabilities of all uncontrollable transitions defined at state $q_i$ sums to $1-\gamma$.
It is obvious that the lower bound computed in Eq.~\eqref{eqlb}
also reflects a lower bound for $\nu_\#(L^-(q_i))$, since addition of state(s) with positive characteristic or eliminating obstacles cannot possibly 
make strings more negative. Furthermore, recalling that the goal $q_\Gl$ is actually reachable from state $q_i$ by a single hop, it follows that not all neighbors of $q_i$ in the 
navigation automaton are blocked, and hence we have the strict inequality:
\begin{gather}\label{eqneg}
 \nu_\#(L^-(q_i)) >  -(1 - \gamma)
\end{gather}
Combining Eqns. \eqref{eqpos} and \eqref{eqneg}, we note that the following condition is sufficient for guaranteeing $\boldsymbol{\nu_\#}\vert_i > 0$.
\begin{gather}
 \frac{\gamma(1-\theta_{min})}{\Crd(\Sigma_C)} \times \chi_\Gl > 1-\gamma
\end{gather}
which after a straightforward calculation yields the bound stated in Eq.~\ref{eqcond1}, and the Statement 1 is proved for the first recursion step, $i.e.$ for $Q_k = \{q_\Gl\}$.
%#####################################################################################
%#####################################################################################
%#####################################################################################
\begin{figure}[!ht]
\centering
\psfrag{qi1}[c][c]{\sffamily \bf  $q_i$}
\psfrag{qi2}[c][c]{\sffamily \bf  $q_i$}
\psfrag{u}[c][c]{\sffamily \bf  $\sigma$}
\psfrag{s1}[c][c]{\sffamily \bf  $\sigma_1$}
\psfrag{s2}[c][c]{\sffamily \bf  $\sigma_2$}
\psfrag{s21}[c][c]{\sffamily \bf $\sigma_1,\sigma_2$}
\psfrag{Q}[c][c]{\sffamily \bf \color{OrangeRed4} $\boldsymbol{Q_k}$}
\psfrag{qj3}[c][c]{\sffamily \bf  $q_{j_1}$}
\psfrag{qj4}[c][c]{\sffamily \bf  $q_{j_2}$}
\psfrag{qj5}[c][c]{\sffamily \bf  $q_{j_3}$}
\psfrag{qj2}[c][c]{\sffamily \bf  $q_\Gl$}
 \includegraphics[width=2.4in]{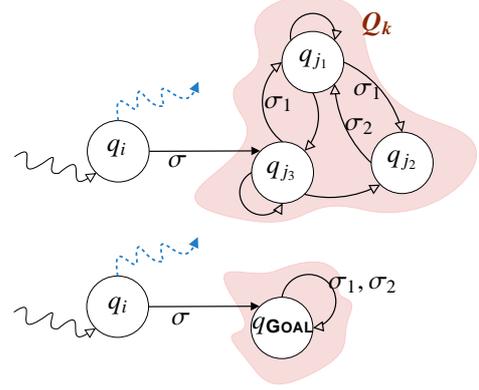}
\caption{Illustration for Statement 1 of Proposition~\ref{proprecurs}. Note that even if $Q_k$ has multiple states, $q_{j_1},q_{j_2},q_{j_3}$,
the measure of any string (say $\sigma\sigma_1\sigma_1\sigma_2$) from $q_i$ is the same as  if $q_i$ was directly connected to the goal $q_\Gl$
with all controllable events disabled at $q_\Gl$. The bottom plate illustrates this by showing the hypothetical scenario where $q_i$ is connected to $q_\Gl$ by $\sigma$ and
$\sigma_1,\sigma_2$ are controllable events disabled at $q_\Gl$. Note that for this argument to work, we must eliminate uncontrollable transitions from all states in $Q_k$.}\label{figredux}
\end{figure}
%#####################################################################################
%#####################################################################################
%#####################################################################################

%#####################################################################################
%#####################################################################################
%#####################################################################################
To extend the argument to later recursion steps of Algorithm~\ref{figflow}, $i.e.$, for  $k > 0$, we argue as follows. Let $Q_k \supsetneqq \{q_\Gl\}$ and we have eliminated all uncontrollable transitions from all $q_j \in Q_k$ (as required in Algorithm~\ref{figflow}).
Further, let $q_i \in Q\setminus Q_k$ such that it is possible to reach some $q_j \in Q$ in a single controllable hop, $i.e.$
\begin{gather}
 q_i \xrightarrow[controllable]{\sigma} q_j, \ q_j \in Q_k
\end{gather}
We first claim that
\begin{gather}\label{eqgrt}
\forall q_j \in Q_k, q_r \notin Q_k, \  \boldsymbol{\nu_\#}^{[k]}\vert_j >  \boldsymbol{\nu_\#}^{[k]}\vert_r
\end{gather}
% The argument for  Eq. ~\eqref{eqgrt} is as follows:
% Assume if possible that after $\nustar$-optimization, we have an enabled transition:
% \begin{gather}
%  q_j \xrightarrow[controllable]{\sigma} q_r, \ \mathrm{where} \  q_j \in Q_k, q_r \notin Q_k
% \end{gather}
% This would result in a contradiction since it is easy to see 
which immediately follows from the fact
that 
the optimal configuration (of transitions from states in $Q_k$) at the end of the $\nustar$-optimization at the $k^{th}$ step would be to have all controllable transitions
from states $q_j \in Q_k$ enabled if and only if the transition goes to some state in $Q_k$, since in that case every string initiating from $q_j$ terminates on a state having characteristic $\chi_\Gl$ (since there is no uncontrollability from states within $Q_k$ by construction), whereas if a transition $q_j \xrightarrow[controllable]{\sigma} q_r, \ \mathrm{where} \  q_j \in Q_k, q_r \notin Q_k$ allows strings which end up in zero-characteristic states and also
(via uncontrollable transitions) on negative-characteristic states.

Eq. ~\eqref{eqgrt} implies that no enabled string exits $Q_k$.
It therefore follows that every string $\sigma\omega$ starting from the state $q_i$, with $\omega \in \Sigma_C^\star$ 
and $\delta(q_i,\sigma) \in Q_k$ ($i.e.$, $\sigma$ leads to some state within $Q_k$)
has exactly the same measure as if $q_i$ is directly connected to $q_\Gl$ 
and all controllable transitions are disabled at $q_\Gl$ (See Figure~\ref{figredux} for an illustration). This conceptual reduction implies 
that Eq.~\eqref{eqpos} is valid when $Q_k \supsetneqq \{ q_\Gl\}$ since the lower bound for $\nu_\#(L^+(q_i))$ can be computed exactly
as already done for the case with $Q_k=\{q_\Gl\}$. The  argument for obtaining the lower bound for $\nu_\#(L^-(q_i))$
 is the same as before, thus completing the proof for Statement 1 for 
all recursion steps of Algorithm~\ref{figflow} .\\
\textbf{Statement 2:}\\
Let $Q_R \subsetneqq Q$ be the set of states from which a feasible path to the goal exists.
If $\Crd(Q_R)= 1$, then we must have $Q_R=\{q_\Gl\}$ and the recursion terminates in one step.
In general, for the $k^{th}$ recursion step, let $\Crd(Q_k)< \Crd(Q_R)$. Since there exists at least one state, not in $Q_k$, 
from which a feasible path to the goal exists,
it follows that there exists at least one state $q_j$ from which it is possible to reach a state in $Q_k$ in one hop.
Using Statement 1, we can then conclude:
\begin{gather}
Q_{k+1} \neq Q_k \Rightarrow \Crd(Q_{k+1}) \geqq \Crd(Q_k) +1 \notag \\
\Rightarrow \Crd(Q_{k+1}) \geqq k+1
\end{gather}
which immediately implies that the recursion must terminate in at most $\Crd(Q)$ steps.\\
\textbf{Statement 3:}\\
Follows immediately from the argument used for proving Statement 2.
\end{proof}
\begin{rem}
 The generality of Eq.~\eqref{eqcond1} is remarkable. Note that the lower bound is  not directly dependent on the exact structure of the navigation automaton; what only matters is the
number of controllable moves available at each state, the coefficient of dynamic deviation $\gamma(\Gnm)$ and the critical termination probability 
$\theta_{min}$.  Although the exact automaton structure and the probability distribution of the uncontrollable transitions are not directly important, their  effect  
 enters, in a somewhat non-trivial fashion, through the value of the critical 
termination probability. The reader might want to review Algorithm~\ref{Algorithm01} (See also \cite{CR06,CR07}) which computes the critical termination probability in each step of the $\nustar$-optimization for a better elucidation of the aforementioned connection between the structure of the navigation automaton and $\chi_\Gl$.
\end{rem}

The dependencies of the acceptable lower bound for $\chi_\Gl$ with the coefficient of dynamic deviation $\gamma(\Gnm)$, as computed in Proposition~\ref{proprecurs}, is illustrated in Figures~\ref{figEE001}(a) and (b). The key points to be noted are:
\begin{enumerate}
 \item As $\gamma(\Gnm) \rightarrow 0^+$, $\chi_\Gl \rightarrow +\infty$; which reflects the physical fact that if no events are controllable, then we cannot optimize the
mission plan no matter how large $\chi_\Gl$ is chosen.
\item As $\gamma(\Gnm) \rightarrow 1$, $\chi_\Gl \rightarrow 0$; which implies that in the absence of dynamic effects any positive value of $\chi_\Gl$ suffices. This reiterates the
result obtained with $\chi_\Gl = 1$ in \cite{CMR08}.
\item As the number of available controllable moves increases (See Figure~\ref{figEE001}(a)), we need a larger value of $\chi_\Gl$; similarly if the
critical termination probability $\theta_{min}$ is large, then the value of $\chi_\Gl$ required is also large (See Figure~\ref{figEE001}(b)).
\item The functional relationships in Figures~\ref{figEE001}(a) and (b) establish the fact  that for relatively smaller number of controllable moves, a 
large value of $\gamma(\Gnm)$ and a small termination probability, a constant value of $\chi_\Gl=1$ may be sufficient.
\end{enumerate}

%##################################################################
%##################################################################
%##################################################################
\begin{figure}[!ht]
\centering
% \begin{minipage}{3in}
% \centering
% \psfrag{c  }[r]{$\mathop{log}_{10}(\chi_\Gl)$}
% \psfrag{g}[r]{$\gamma(\Gnm)\mspace{10mu}$}
% \psfrag{T1}[l]{$\theta_{min} = 0.001$}
% \psfrag{T2}[b][l]{$\theta_{min} = 0.1$}
% \psfrag{C = 100}[b][l]{\small$\phantom{xxx}\Crd(\Sigma_C)=100$}
% \psfrag{C = 4}[t][c]{\small$\phantom{xxx}\Crd(\Sigma_C)=4$}
% \subfigure[$\theta_{min}=0.01$]{\includegraphics[width=2.75in]{chig30}}
% \subfigure[$\Crd(\Sigma_C)=8$]{\includegraphics[width=2.75in]{chig40}}
% % \subfigure[$\Crd(\Sigma_C)=8$]{\includegraphics[width=3in]{chig5}}
% \vspace{4pt}
% \end{minipage}
\includegraphics[width=2.75in]{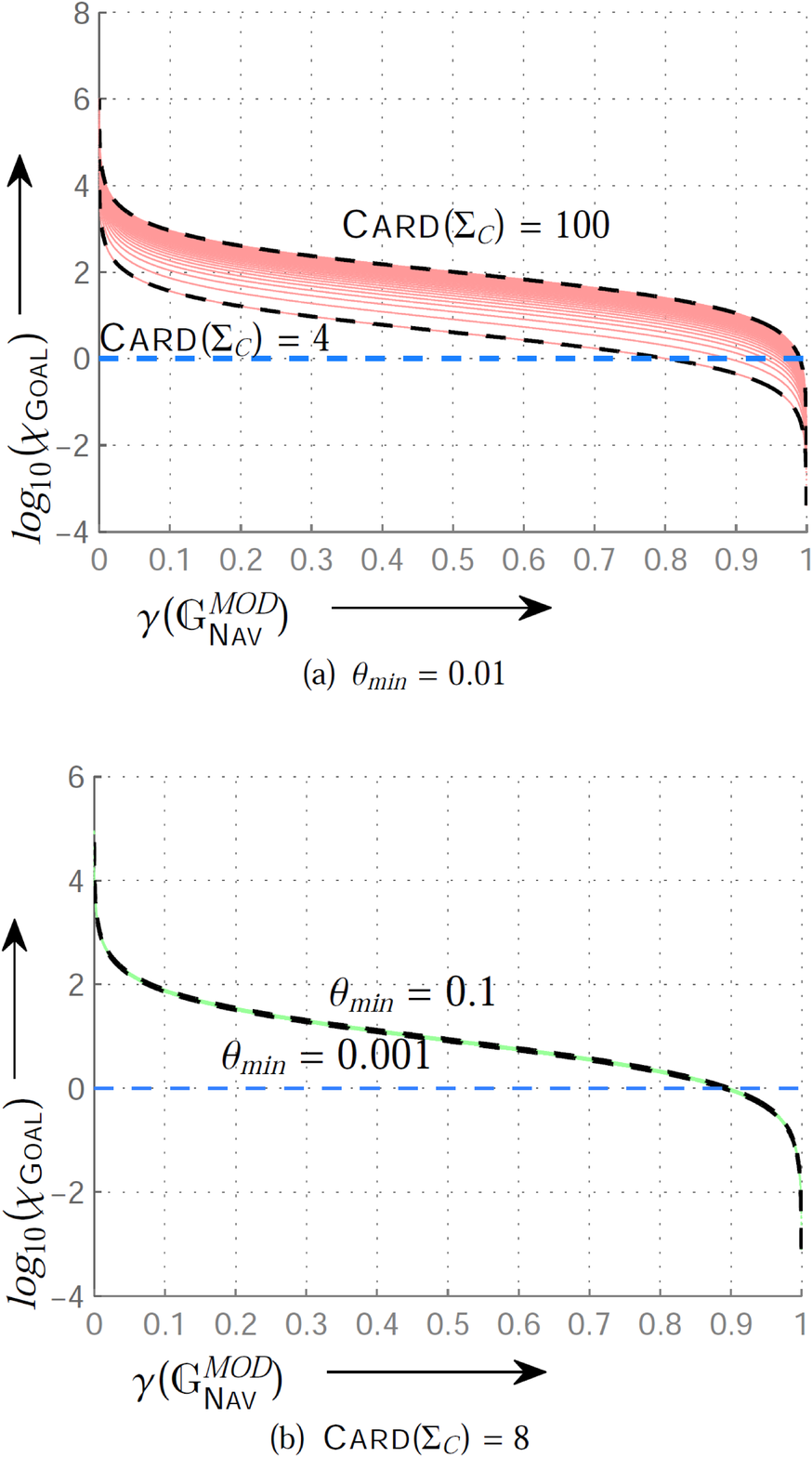}
\caption{Variation of the acceptable lower bound for $\chi_\Gl$ with $\gamma(\Gnm)$. (a) 
The set of controllable moves is expanded from $\Crd(\Sigma_C)=4$
to $\Crd(\Sigma_C)=100$ while holding $\theta_{min}=0.01$ (b) 
The critical termination probability $\theta_{min}$ is varied from $0.001$ to $0.1$ 
while holding $\Crd(\Sigma_C)=8$. Note the lines are almost coincident in this case. }\label{figEE001}
\end{figure}
%##################################################################
%##################################################################
%##################################################################
\begin{algorithm}[!h]
 \small \SetLine
\linesnumbered
  \SetKwInOut{Input}{input}
  \SetKwInOut{Output}{output}
  \caption{Assembly of Plan Vectors}\label{Algorithm_assembly}
\Input{$\boldsymbol{\nu_\#}^{[k]},k=1,\cdots,K$}(Plan Vectors)
\Output{$\boldsymbol{\nu_\#}^{\mathbf{A}}$ (Assembled Plan)} \Begin{ 
% Set $K = \Crd(\{ \nu_\#^{[k]} \})$\tcc*[r]{ Num. of plan vectors}
Set $\boldsymbol{\nu_\#}^{\mathbf{A}}=\boldsymbol{0}$\tcc*[r]{ Zero vector}
\For{$k= 1:K$}{
\For{$i \in Q$}{
$\boldsymbol{\nu_\#}^{tmp}\vert_i = 0$\;
\eIf{$\boldsymbol{\nu_\#}^{[k-1]}\vert_i > 0$}{$\boldsymbol{\nu_\#^{tmp}} \vert_i = 1$\;}{
\If{$\boldsymbol{\nu_\#}^{[k]}\vert_i > 0$}{
$\boldsymbol{\nu_\#^{tmp}}\vert_i = \boldsymbol{\nu_\#}^{[k]}\vert_i$\;
}
}
}
$\boldsymbol{\nu_\#}^{\mathbf{A}} = \boldsymbol{\nu_\#}^{\mathbf{A}} + \boldsymbol{\nu_\#^{tmp}}$\;
}
 }
\end{algorithm}
%##################################################################
%##################################################################
%##################################################################
\subsection{Plan Assembly \& Execution Approach}
The plan vectors $\boldsymbol{\nu_\#}^{[k]}$ (Say, there are $K$ of them, $i.e.$, $k \in \{1,\cdots,K\}$) obtained via the recursive planning
algorithm described above, can be used for subsequent mission execution in two rather distinct ways:
\begin{enumerate}
 \item \textit{(The Direct Approach:)}
\begin{itemize}
 \item At any point during execution, if the current state $q_i \in Q_k$ for some $k \in \{1,\cdots,K\}$, then 
use the gradient defined by the plan vector $\boldsymbol{\nu_\#}^{[k]}$ to decide on the next move, $i.e.$,
$q_j$ is an acceptable next state if $\boldsymbol{\nu_\#}^{[k]}\vert_j > \boldsymbol{\nu_\#}^{[k]}\vert_i$ 
and for states $q_\ell$ that can be reached from
the current state $q_i$ via controllable events, we have $\boldsymbol{\nu_\#}^{[k]}\vert_j \geqq \boldsymbol{\nu_\#}^{[k]}\vert_\ell$.
\item if $\forall k \in \{1,\cdots,K\}, \ q_i \notin Q_k$, then terminate operation because there is no feasible path to the goal.
\item Note that this entails keeping $K$ vectors in memory.
\end{itemize}
\item \textit{(The Assembled Plan Approach:)}
\begin{itemize}
 \item Use $\boldsymbol{\nu_\#}^{[k]}, k \in \{1,\cdots,K\}$ to obtain the assembled plan vector 
$\boldsymbol{\nu_\#}^{\mathbf{A}}$ following Algorithm~\ref{Algorithm_assembly}, which
assigns a real value $\boldsymbol{\nu_\#}^{\mathbf{A}}\vert_i
$ to each state $q_i$ in the workspace. We refer to this map as the assembled plan.
\item Make use of  the gradient defined by $\boldsymbol{\nu_\#}^{\mathbf{A}}$ to reach the goal, by sequentially moving to states with increasing  values specified by the assembled plan, $i.e.$,.
if the current state is  $q_i \in Q$, then 
$q_j$ is an acceptable next state if $\boldsymbol{\nu_\#}^{\mathbf{A}}\vert_j > \boldsymbol{\nu_\#}^{\mathbf{A}}\vert_i$  
and for states $q_\ell$ that can be reached from
the current state $q_i$ via controllable events, we have $\boldsymbol{\nu_\#}^{\mathbf{A}}\vert_j \geqq \boldsymbol{\nu_\#}^{\mathbf{A}}\vert_\ell$.
\item We show in the sequel that if $\boldsymbol{\nu_\#}^{\mathbf{A}}\vert_i < 0$, then no feasible path exists to the goal.
\end{itemize}
\end{enumerate}
% 
% It is important to note that in both approaches, one may have more than one choice for the next state. For example, for the assembled approach, it may turn out that for  a given current state $q_i \in Q$, there exist multiple states
% $q_j,q_k \in Q$ such that both conditions $\boldsymbol{\nu_\#}^{\mathbf{A}}\vert_j > \boldsymbol{\nu_\#}^{\mathbf{A}}\vert_i$ and $\boldsymbol{\nu_\#}^{\mathbf{A}}\vert_k > \boldsymbol{\nu_\#}^{\mathbf{A}}\vert_i$ are satisfied. 
Before we can proceed further,
we need to formally establish some key properties of the assembled plan approach. In particular, we have the following proposition:
\begin{prop}\label{propassembled}
 \begin{enumerate}
 \item For a state $q_i\in Q$, a feasible path to the goal exists from the state $q_i$, if and only if $\boldsymbol{\nu_\#}^{\mathbf{A}}\vert_i > 0$.
\item The assembled plan $\boldsymbol{\nu_\#}^{\mathbf{A}}$ is free from local maxima, $i.e.$, if there exists a $\nustar$-path 
(w.r.t. to $\boldsymbol{\nu_\#}^{\mathbf{A}}$)
from $q_i \in Q$ to $q_j \in Q$ and  a $\nustar$-path from $q_i$ to $q_\Gl$ then there exists a $\nustar$-path from $q_j$ to $q_\Gl$, $i.e.$,
\begin{gather*}
\forall q_i,q_j \in Q \bigg ( \exists \rho_1(q_i,q_\Gl) \bigwedge \exists \rho_2(q_i,q_j) \Rightarrow \exists \rho(q_j,q_\Gl) \bigg )
\end{gather*}
 \item If a feasible path to the goal exists from the state $q_i$, then the agent can reach the goal optimally
by following the gradient of $\boldsymbol{\nu_\#}^{\mathbf{A}}$, where the optimality is to be understood as maximizing the
probability of reaching the goal while simultaneously minimizing the probability of hitting an obstacle ($i.e.$ in the sense stated in Proposition~\ref{propoptinterpret}).
 \end{enumerate}

\end{prop}
\begin{proof}
\textbf{Statement 1:} \\
 Let the plan vectors obtained by the recursive procedure stated in the previous section  be $\boldsymbol{\nu_\#}^{[k]}$ (Say, there are $K$ of them, $i.e.$, $k \in \{1,\cdots,K\}$) and further
let the current state $q_i \in Q_k$ for some $k \in \{1,\cdots,K\}$. 
We observe that on account of Proposition~\ref{propweak}, if $k=1$, then $\boldsymbol{\nu_\#}^{[k]}\vert_i > 0$ is sufficient to guarantee that there exists a $\nustar$-path $\rho(q_i,q_\Gl)$ 
w.r.t the plan vector $\boldsymbol{\nu_\#}^{[1]}$. We further note that $\boldsymbol{\nu_\#}^{[1]}\vert_i <= 0 \Rightarrow q_i \notin Q_1$ (See Eq.~\eqref{eqqk}),
 implying that $\boldsymbol{\nu_\#}^{[1]}\vert_i > 0$ is
also necessary for the existence of $\rho(q_i,q_\Gl)$.
 Extending this argument, we note that,
 for $k>1$, a $\nustar$-path $\rho(q_i,q_j)$ with $q_j \in Q_{k-1}$ exists (with respect to the
plan vector $\boldsymbol{\nu_\#}^{[k]}$) if and only if
$\boldsymbol{\nu_\#}^{[k]}\vert_i > 0$. Noting that $\boldsymbol{\nu_\#}^{[k]}\vert_i > 0 \Leftrightarrow \boldsymbol{\nu_\#}^{\mathbf{A}}\vert_i > 0$, (See Algorithm~\ref{Algorithm_assembly})
we conclude that a $\nustar$-path $\rho(q_i,q_j)$ with $q_j \in Q_{k-1}$ exists (with respect to the
plan vector $\boldsymbol{\nu_\#}^{[k]}$) if and only if $\boldsymbol{\nu_\#}^{\mathbf{A}}\vert_i > 0$.
 Also, since $q_j \in Q_{k-1} \wedge q_i \in Q_k$, it follows
from Algorithm~\ref{Algorithm_assembly}, that $\boldsymbol{\nu_\#}^{\mathbf{A}}\vert_j \geqq 1+\boldsymbol{\nu_\#}^{\mathbf{A}}\vert_i > 0$.
It follows that the same argument can be used recursively to find $\nustar$-paths $\rho(q_j,q_{\ell_1}),\cdots,\rho(q_j,q_\Gl)$ if and only if 
 $\boldsymbol{\nu_\#}^{\mathbf{A}}\vert_i > 0$.

To complete the proof, we still need to show that if there exists a feasible path from a state $q_i$ to the
goal $q_\Gl$, then there exists a $\nustar$-path $\rho(q_i,q_\Gl)$.
We argue as follows:
Let $q_i=q_{r1} \rightarrow q_{r2} \rightarrow \cdots \rightarrow q_{r_{m-1}} \rightarrow q_{r_m}=q_\Gl$ be a feasible path from the state $q_i$ to 
$q_\Gl$. Furthermore, assume if possible that
\begin{gather}\label{eqcontra}
\forall k \  \boldsymbol{\nu_\#}^{[k]}\vert_i \leqq 0 
\end{gather}
$i.e.$, there exists no $\nustar$-path from $q_i$ to $q_\Gl$ w.r.t $\boldsymbol{\nu_\#}^{\mathbf{A}}$.
We observe that since  it is possible to reach $q_\Gl$ from $q_{r_{m-1}}$  in one hop, using Proposition~\ref{proprecurs} we have:
\begin{gather}
 \boldsymbol{\nu_\#}^{[1]}\vert_{r_{m-1}} > 0 \Rightarrow q_{r_{m-1}} \in Q_1
\end{gather}
We further note:
\begin{gather}
 \boldsymbol{\nu_\#}^{[1]}\vert_{r_{m-2}} > 0 \Rightarrow q_{r_{m-2}} \in Q_1 \\
\boldsymbol{\nu_\#}^{[1]}\vert_{r_{m-2}} \leqq 0 \Rightarrow \boldsymbol{\nu_\#}^{[2]}\vert_{r_{m-2}} > 0 \Rightarrow q_{r_{m-2}} \in Q_2
\end{gather}
Hence, we conclude either $q_{r_{m-2}} \in Q_2$ or $q_{r_{m-2}} \in Q_1$. It follows by straightforward induction that either 
$q_{1} \in Q_{m-1}$ or $ q_{1} \in Q_{m-2}$, which contradicts the statement in Eq.~\eqref{eqcontra}.
Therefore, we conclude that if a feasible path to the goal exists from any state $q_i$, then a $\nustar$-path
$\rho(q_i,q_\Gl)$ (w.r.t $\boldsymbol{\nu_\#}^{\mathbf{A}}$) exists as well.
This completes the proof of Statement 1.\\
\textbf{Statement 2:}\\
Given states $q_i,q_j \in Q$, assume that we have the $\nustar$-paths $\rho_1(q_i,q_\Gl)$ and $\rho_2(q_i,q_j)$.
We observe that:
\begin{subequations}
\begin{gather}
 \exists \rho_1(q_i,q_\Gl) \Rightarrow \boldsymbol{\nu_\#}^{\mathbf{A}}\vert_i > 0 \ \mathsf{(See \ Statement \ 1)}\\
 \exists \rho_2(q_i,q_j) \Rightarrow \boldsymbol{\nu_\#}^{\mathbf{A}}\vert_j \geqq \boldsymbol{\nu_\#}^{\mathbf{A}}\vert_i \ 
\mathsf{(See \ Definition~\ref{defnustar})}\\
\Rightarrow \boldsymbol{\nu_\#}^{\mathbf{A}}\vert_j > 0 \Rightarrow \exists \rho(q_j,q_\Gl) \ \mathsf{(See \ Statement \ 1)}
\end{gather}
\end{subequations}
which proves Statement 2.\\
\textbf{Statement 3:}\\
Statements 1 and 2 guarantee that if a feasible path to the goal exists from a state $q_i \in Q$, then an agent can reach the goal
by following a $\nustar$-path (w.r.t $\boldsymbol{\nu_\#}^{\mathbf{A}}$) from $q_i$, $i.e.$, by sequentially 
moving to states which have a better measure as compared to the current state.

We further note that a $\nustar$-path $\omega$ w.r.t $\boldsymbol{\nu_\#}^{\mathbf{A}}$ from any state $q_i$ to $q_\Gl$ can be 
represented as a concatenated sequence $\omega_1\omega_2\cdots \omega_r \cdots \omega_m$ where $\omega_r$ is a $\nustar$-path from some intermediate
state $q_j \in Q_s$, for  some $s \in \{1,\cdots,K\}$, to some state $q_\ell \in Q_{s-1}$. Since the recursive 
procedure optimizes all such intermediate plans, and since the outcome ``reached goal from $q_i$'' can be visualized as the intersection of the
mutually independent outcomes ``reached $Q_s$ from $q_i \in Q_{s-1}$'', ``reached $Q_{s+1}$ from $q_j \in Q_s$'' , $\cdots$ , ``reached $q_\Gl$ from $q_\ell \in Q_1$'',
  the overall path must be optimal as well. This completes the proof.
\end{proof}

We compute the set of acceptable next states from the following definition.
\begin{defn}\label{defqnext}
 Given the current state $q_i \in Q$, $Q_{next}$ is the set of states satisfying the strict inequality:
\begin{gather}
 Q_{next} = \{q_j \in Q: \boldsymbol{\nu_\#}^{\mathbf{A}}\vert_j > \boldsymbol{\nu_\#}^{\mathbf{A}}\vert_i \}
\end{gather}
\end{defn}
We note that Proposition~\ref{propassembled} implies that
$Q_{next}$ is empty if and only if the current state is the goal or if no feasible path to the goal exists from the current state.
%######################################################################
%######################################################################
%######################################################################
% 
\section{Computation of Amortized Uncertainty Parameters}\label{secuncertain}
\begin{figure}[!ht]
\psfrag{U}[cc]{\bf \footnotesize \color{DeepSkyBlue4} \txt{Actuation \& Localization \\ Uncertainty}}
\psfrag{G}[ct]{\bf \scriptsize \DGreen \txt{Grid Decomposition}}
\psfrag{R}[cc]{\bf \scriptsize \color{white} R}
\psfrag{T}[lt]{\bf \footnotesize \Mblue Trajectory}
\psfrag{q1}[cc]{\bf \footnotesize \color{Green4} $\phantom{|}q_1$}
\psfrag{q2}[cc]{\bf \footnotesize \color{Green4} $\phantom{|}q_2$}
\psfrag{q3}[cc]{\bf \footnotesize \color{Green4} $\phantom{|}q_3$}
\psfrag{q4}[cc]{\bf \footnotesize \color{Green4} $\phantom{|}q_4$}
\psfrag{q5}[cc]{\bf \footnotesize \color{Green4} $\phantom{|}q_5$}
\psfrag{q6}[cc]{\bf \footnotesize \color{Green4} $\phantom{|}q_6$}
\psfrag{q7}[cc]{\bf \footnotesize \color{Green4} $\phantom{|}q_7$}
\psfrag{q8}[cc]{\bf \footnotesize \color{Green4} $\phantom{|}q_8$}
\psfrag{Q}[cc]{\bf \footnotesize \color{Green4} $\phantom{|}Q$}
\subfigure[]{ \includegraphics[width=3.5in]{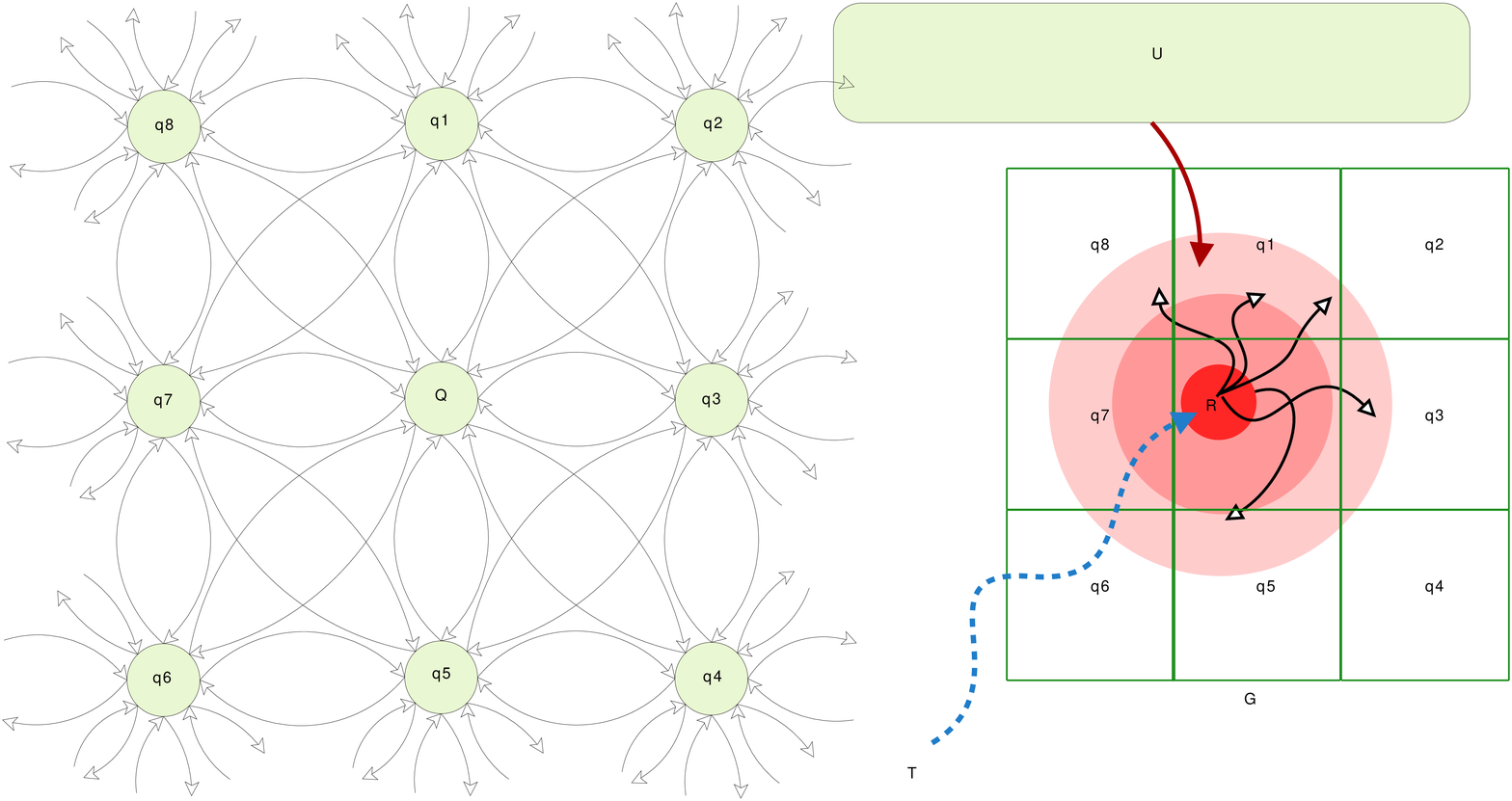}}\\
\centering
\psfrag{X}[rb]{\footnotesize \bf  X}
\psfrag{Y}[rb]{\footnotesize \bf  Y}
\psfrag{A}[lt]{\footnotesize \bf  A}
\psfrag{B}[rb]{\footnotesize \bf  B}
\psfrag{C}[rb]{\footnotesize \bf  C}
\psfrag{D}[lt]{\footnotesize \bf  D}
\psfrag{j0}{\small $J_0$}
\psfrag{j1}{\small $J_1$}
\psfrag{j2}{\small $J_2$}
\psfrag{j3}{\small $J_3$}
\psfrag{j4}[ct]{\small $J_4$}
\psfrag{j5}[ct]{\small $J_5$}
\psfrag{j6}[ct]{\small $J_6$}
\psfrag{j7}{\small $J_7$}
\psfrag{j8}{\small $J_8$}
\psfrag{p}[cb]{\small \txt{$\phantom{X}$Discretization}}
\psfrag{h}[cb]{\small \txt{Deviation\\Contour$(\DST)\rightarrow$}}
\subfigure[]{\includegraphics[width=3in]{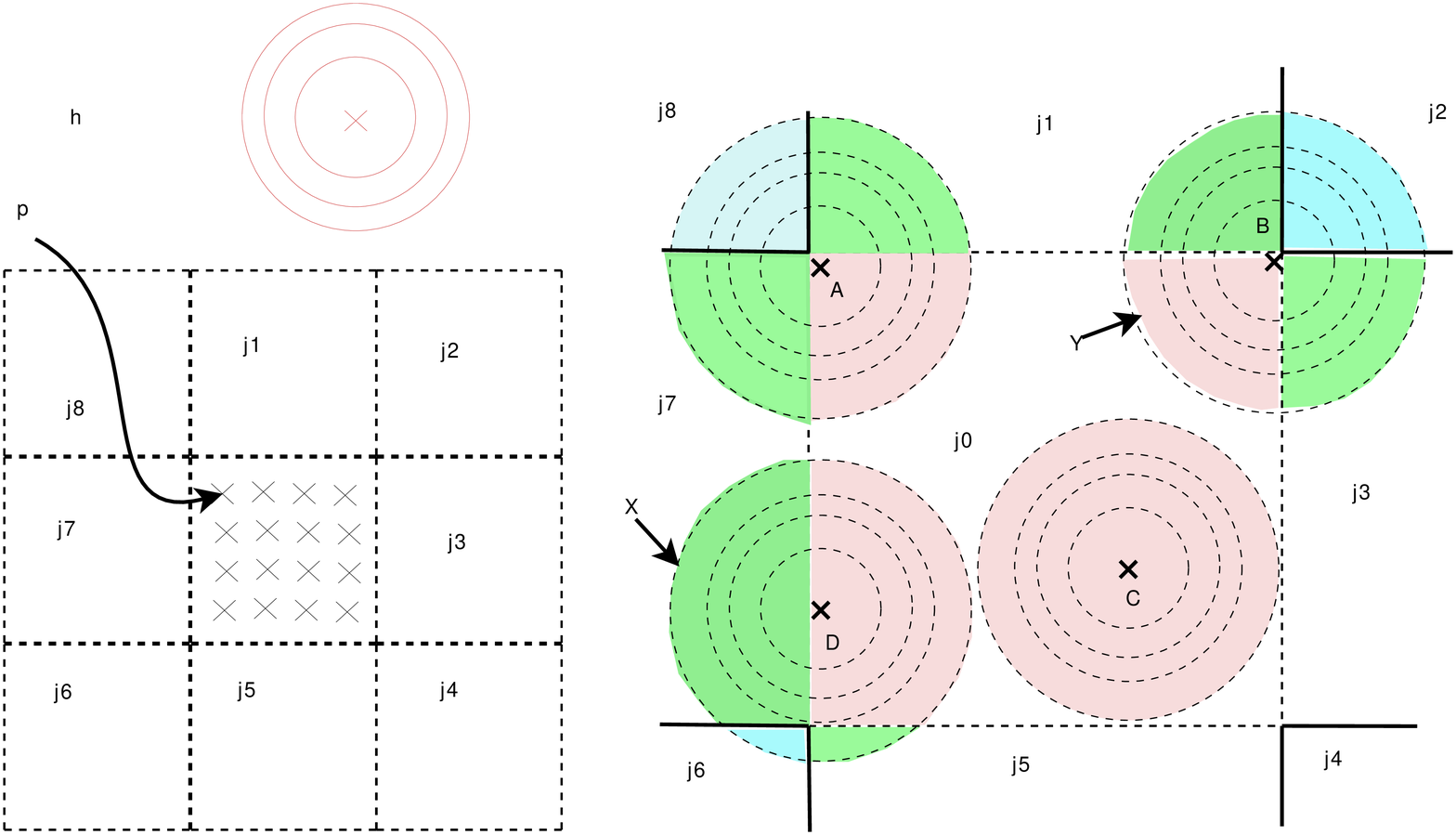}}
\caption{(a) Model for 2D circular robot(b) Numerical integration technique for computing the dynamic parameters for the case of a circular robot model$e.g.$ a SEGWAY RMP 200}\label{fignumerical}
\end{figure}
Specific numerical values of the uncertainty parameters, $i.e.$ the probability of uncontrollable transitions  in the navigation automaton can be computed
from a knowledge of the average uncertainty in the robot localization and actuation in the configuration space.
For simplicity of exposition, we assume a 2D circular robot; however the proposed techniques are applicable 
to more general scenarios. 
The complexity of this identification is related to the dynamic model assumed for the platform ($e.g.$ circular robot in a 2D space, rectangular robot with explicit heading in the configuration etc.), the simplifying assumptions made for the possible errors, and the \textit{degree of averaging} that we are willing to make. Uncertainties arise from two 
key sources:
\begin{enumerate}
 \item Actuation errors: Inability of the robot to execute planned motion primitives exactly, primarily due to the dynamic response of the physical platform.
\item Localization errors: Estimation errors arising from sensor noise, and the limited time available for post-processing exteroceptive data for a moving platform.
Even if we assume that the platform is capable of processing sensor data to eventually localize perfectly for a static robot, the fact that we have to get the estimates while the robot pose is changing in real time, implies that the estimates lag the actual robot configuration. Thus, this effect cannot be neglected even for the best case scenario of a 2D robot with an accurate global positioning system (unless the platform speed is  sufficiently small).
\end{enumerate}
In our approach, we do not distinguish between the different sources of uncertainty, and attempt to represent the overall amortized effect as uncontrollability in the 
navigation automaton. The rationale for this approach is straightforward: we visualize actuation errors as the uncontrollable execution of transitions before the controllable 
planned move can be executed, and for localization errors, we assume that any controllable planned move is followed by an uncontrollable transition to the actual configuration.
Smaller is the probability of the uncontrollable transitions in the navigation automaton, $i.e.$, larger is the coefficient of dynamic deviation for each state, smaller is 
the uncertainty in navigation. From a history of  observed dynamics or from prior knowledge, one can compute  the distribution of the robot pose around the estimated configuration
(in an amortized sense). Then the probability of uncontrollable transitions can be estimated by computing the probabilities with which the robot 
must move to the neighboring cells to approximate  this distribution. The situation for a 2D circular robot is illustrated in Figure~\ref{fignumerical}(a),
% 
% Data from  exteroceptive sensors is processed to obtain estimated localization.
% The need for this post processing implies that the localization estimates are obtained in discrete time intervals, within which the robot may 
% move from its current location. The situation is illustrated in Figure~\ref{fignumerical}(a). Averaging over observed history then leads to 
% a distribution of the expected robot localtion around the localization estimate. Better the localization, and shorter the time interval, sharper is this distribution around the 
% computed estimate, $i.e.$, smaller is the variance of the uncertainty. 
% 
where we assume that averaging over the observations lead to 
a distribution with zero mean-error; $i.e.$, the distribution is centered around the estimated location in the configuration space. For more complex scenarios (as we show in the simulated examples), this assumption can be easily relaxed. We call this distribution the \textit{deviation contour} ($\DST$) in the sequel. The amortization or averaging is involved purely in estimating
the deviation contour from observed dynamics (or from prior knowledge); a simple methodology for which will be presented in the sequel. However, we first formalize the computation of the uncertainty parameters from a knowledge of the deviation contour.

For that purpose, we consider the current state in the navigation automaton to be $q_i$. Recall that $q_i$ maps to a set of possible configurations in the 
workspace. For a 2D circular robot, $q_i$ corresponds to a set of $x-y$ coordinates that the robot can occupy, while for a rectangular robot, $q_i$ maps to a set of $(x,y,\theta)$ coordinates.
The footprint of the navigation automaton states in the configuration space can be specified via the map  $\xi : Q \rightarrow 2^\mathcal{C}$, where $\mathcal{C}$ is the configuration space of the robot.
In general, for a given current state $q_i$, we can identify the set $\mathcal{N}(q_i)\subset Q$ of neighboring states that the robot can transition to in one move. The current state $q_i$ is also
included in $\mathcal{N}(q_i)$ for notational simplicity.
In case of the 2D circular robot model considered in this paper, 
the cardinality of $\mathcal{N}(q_i)$ is 8 (provided of course that $q_i$ is not blocked and is not a boundary state). 
For a position $s \in \xi(q_i)$ of the robot, we denote a neighborhood of radius $r$  of the position $s$ in the configuration space as $\mathcal{B}_{s,r}$. 
The normalized ``volume'' intersections of $\mathcal{B}_{s,r}$ with the footprints of the states included in $\mathcal{N}(q_i)$ 
in the configuration space can be expressed as :
% \begin{gather}
%  F_j(L,r) = \frac{\displaystyle\mathcal{B}_{L,r} \bigcap \xi(q_j)}{\displaystyle\sum_{q_j\in \mathcal{N}(q_i)}\mathcal{B}_{L,r} \bigcap \xi(q_j)}, \forall q_j \in \mathcal{N}(q_i)
% \end{gather}
\begin{gather}
 F_j(s,r) = \frac{\int_{A_j}\mathrm{d}x}{\displaystyle\int_{\displaystyle\cup_{q_j} A_j}\mathrm{d}x}, \forall q_j \in \mathcal{N}(q_i)
\end{gather}
where $A_j = \displaystyle\mathcal{B}_{s,r} \bigcap \xi(q_j)$ and $\mathrm{d}x$ is the appropriate Lebesgue measure for the continuous configuration space.

We observe that the expected or the average  
probability of the  robot deviating to a neighboring state $q_j \in \mathcal{N}(q_i)$ from a location $s \in \xi(q_i)$ is given by:
\begin{gather}
 \int_0^{\infty}F_j(s,r)\DST \mathrm{d}r
\end{gather}
Hence, the probability $\Pi^{uc}_{ij}$ of uncontrollably transitioning to a neighboring state $q_j$ from the current state $q_i$ 
% (denoted by $\Pi^{uc}_{ij}$) 
is obtained by 
considering the integrated effect of all possible positions of the robot within $\xi(q_i)$, $i.e.$ we have:
\begin{gather}
\Pi^{uc}_{ij} = \frac{\displaystyle\int_{\xi(q_i)}\int_0^{\infty}F_j(s,r)\DST \mathrm{d}r\mathrm{d}s}{\displaystyle\sum_{q_j \in \mathcal{N}(q_i)}\int_{\xi(q_i)}\int_0^{\infty}F_j(s,r)\DST \mathrm{d}r\mathrm{d}s}
\end{gather}
where $\mathrm{d}r,\mathrm{d}s$ are appropriate Lebesgue measures on the continuous configuration space of the robot.
It is important to note that the above formulation is completely general and makes no assumption on the structure of the configuration space, $e.g.$, the 
calculations can be carried out for 2D circular robots, rectangular robots or platforms with more complex kinematic constraints equally well.
Figures~\ref{fighistex1}(a)-(c) illustrate the computation for a circular robot with eight controllable moves, $.e.g.$, the situation for a SEGWAY RMP.
The 2D circular case is however the simplest, where any state that can be reached by an uncontrollable transition, can also be reached by a controllable move.
For more complex scenarios, this may not be the case. For example, in the rectangular model, with constraints on minimum turn radius, the robot may not be 
able to move via a controllable transition from $(x,y,h_1)$ to $(x,y,h_2)$, where $h_i, i=1,2$ is the heading in the initial and final configurations. However, there most likely will be 
an uncontrollable transition that causes this change, reflecting uncertainty in the heading estimation (See Figure~\ref{figruc}). Also, one can 
reduce the averaging effect by considering more complex navigation automata. For example, for a 2D circular robot, the configuration state can be 
defined to be $(x_{previous},y_{previous},x_{current},y_{current})$, $i.e.$ essentially considering a 4D problem. The identification of the 
uncertainty parameters on such a model will capture the differences in the uncontrollable transition probabilities arising from arriving at a given state from different directions. 
While the 2D model averages out the differences, the 4D model will make it explicit in the specification of the navigation automaton (See Figure~\ref{fignoncirc}). In the sequel, we will 
present comparative simulation results for these models. Note, in absence of uncertainty, the 4D implementation is superfluous; $(x_{previous},y_{previous},x_{current},y_{current})$ has no more 
information than $(x_{current},y_{current})$ in that case.
\begin{figure}[!ht]
\centering
\psfrag{H}[rc]{$15^\circ$}
\psfrag{A}[cr]{\color{Red4}\txt{Available \\ Controllable $\boldsymbol{\rightarrow}$\\ Transitions$\phantom{XX}$}}
\psfrag{C}[cl]{\txt{Controllable \\ Move}}
\psfrag{U}[cl]{\txt{Uncontrollable \\ Move}}
 \includegraphics[width=1.5in]{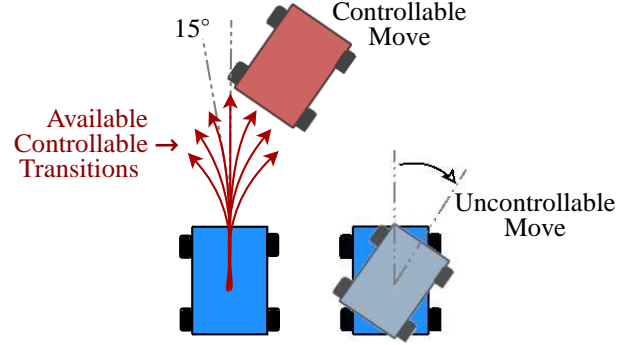}
\caption{A rectangular robot unable to execute zero-radius turns. There exists uncontrollable transitions that alter heading in place, which reflects 
uncertainties in heading estimation, although there are no controllable moves that can achieve this transition}\label{figruc}
\end{figure}

Next, we present a methodology for computing
the relevant uncertainty parameters as a function of the robot dynamics.
We assume a modular plan execution framework, in which the 
low-level continuous controller on-board the 
robotic platform is
sequentially given a target cell (neighboring to the current cell)
 to go to, as it executes the plan. The robot may be able to reach the
cell and subsequently receives the next target, or may end up in a  different cell due to dynamic
constraints, when it receives the next target from this deviated cell as dictated by the computed plan.
The inherent dynamical response of the particular robot determines how well
the patform is able to stick to the plan. We formulate a framework to compute 
the probabilities  of uncontrollable transitions that best describe these deviations.
%##############################################################
\begin{figure}[!h]
\centering
% \begin{minipage}{3.25in}
\vspace{25pt}
\centering
\psfrag{J0}{\small $J_0$}
\psfrag{J1}{\small $J_1$}
\psfrag{J2}{\small $J_2$}
\psfrag{J3}{\small $J_3$}
\psfrag{J4}{\small $J_4$}
\psfrag{J5}{\small $J_5$}
\psfrag{J6}{\small $J_6$}
\psfrag{J7}{\small $J_7$}
\psfrag{J8}{\small $J_8$}
\psfrag{U}[bl]{\BRed $\boldsymbol{\sigma_u}$}
\psfrag{A}[rc]{\red  \textbf{A}}
\psfrag{B}[rc]{\red   \textbf{B}}
\psfrag{C}[cr]{\red   \textbf{C}}
\psfrag{D}[tc]{\red   \textbf{D}}
\psfrag{D0}{\red \small $\Delta_R(t_0)$}
\psfrag{Dt}{\red \small $\Delta_R(t)$}
\psfrag{T}[cc]{\Mblue \small \txt{Robot \\ Trajectory}}
\psfrag{T1}[tl]{\Mblue \small Target Cell}
\psfrag{T2}[cr]{\BRed \small \txt{Current Location $\phantom{.}$\\ (time = $t_0$)}}
\subfigure[]{ \includegraphics[width=2in]{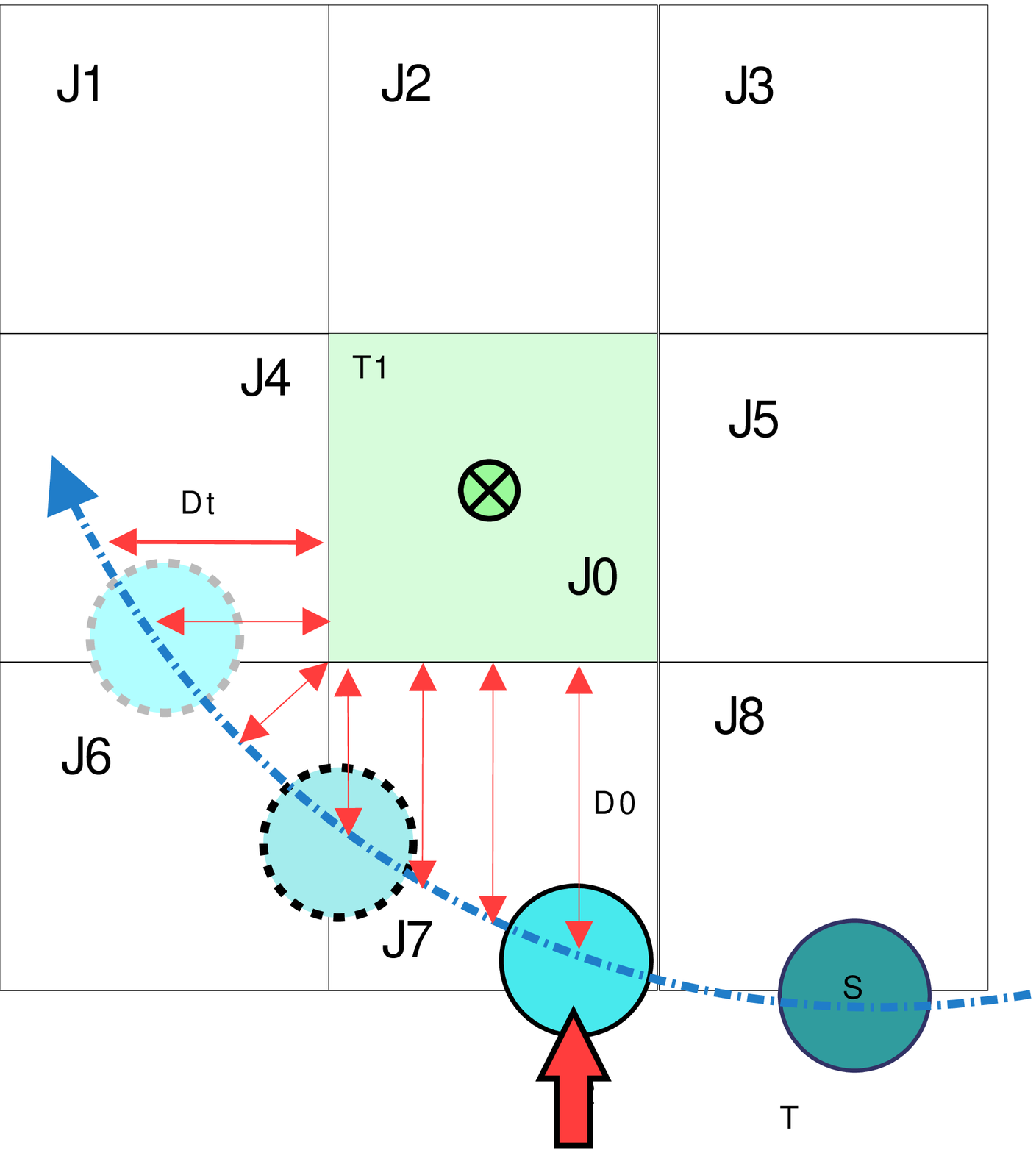}}
\subfigure[]{ \includegraphics[width=2in]{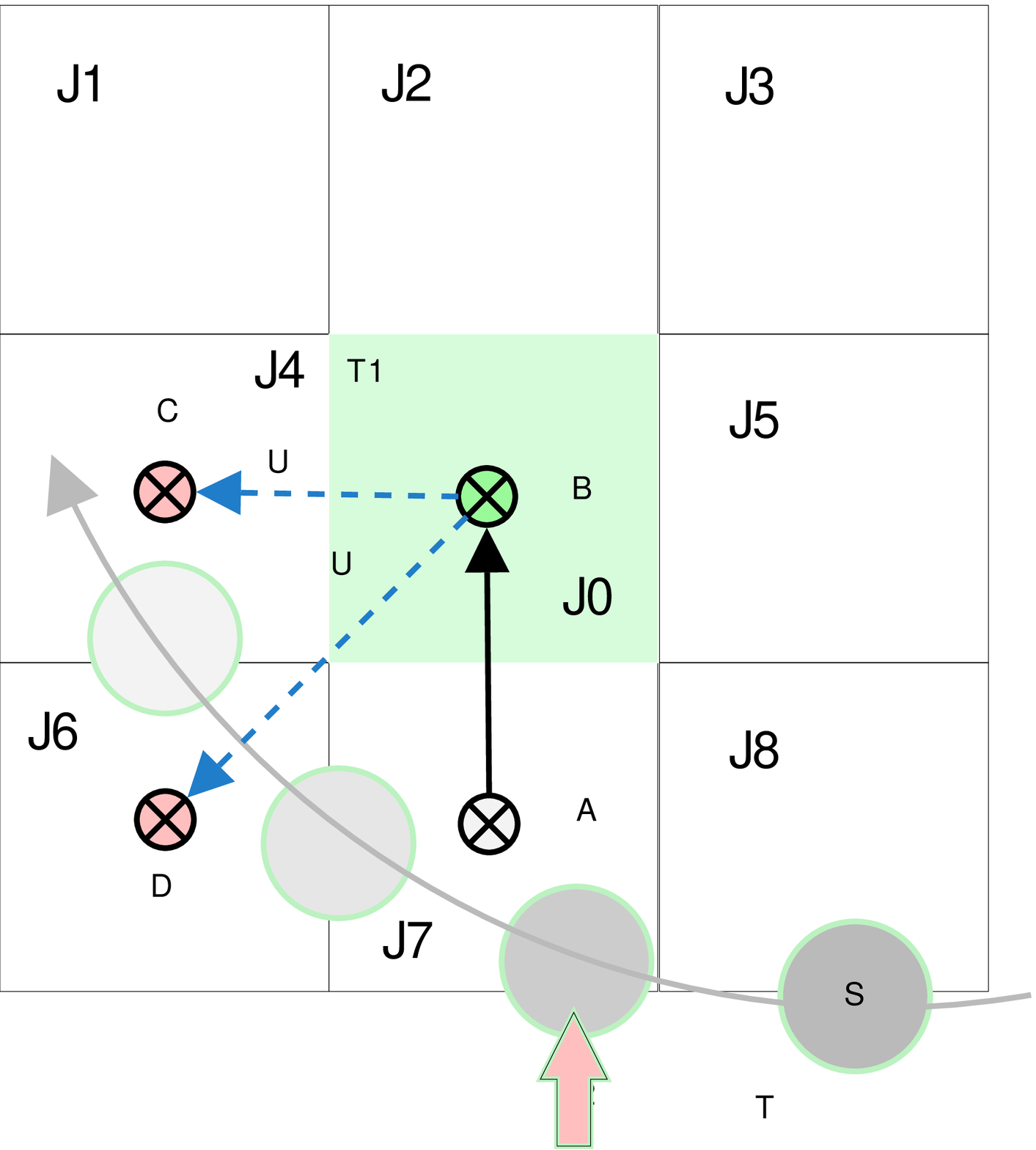}}
% \end{minipage}
% 
\caption{Illustration for computation of amortized dynamic uncertainty parameters}\label{figdev1}
\end{figure}
% 
% 
%##############################################################
\begin{defn}
 The raw deviation $\Delta_R(t)$ as a function of the operation time $t$ is defined as 
follows:
\begin{gather}
 \Delta_R(t) = \Theta(  p(t) , \zeta(t) )
\end{gather}
where $p(t)$ is the current location of the robot in the workspace coordinates, 
$\zeta(t)$ is the location of the point within the current target cell which is nearest to the
robot position $p(t)$ (See Figure~\ref{figdev1}), and $\Theta(\cdot,\cdot)$ is an appropriate 
distance metric in the configuration space.
\end{defn}
The robot will obviously take some time to reach the target cell, assuming it is actually able to do so.
We wish to eliminate the effect of this delay from our calculations, since a platform that
is able to sequentially reach each target cell, albeit with some delay, does not need the 
plan to be modified. Furthermore, unless velocity states are incorporated in the navigation automata, 
the plan cannot be improved for reducing this delay. We note that the raw deviation $\Delta_R(t)$ incorporates the effect of this possibly 
variable delay and needs to be corrected for. We do so by introducing the delay corrected deviation $\Delta(t)$ as follows:
% 
%##############################################################
%##############################################################
% \begin{figure}[!h]
% \centering
% \psfrag{X}[rb]{\footnotesize \bf  X}
% \psfrag{Y}[rb]{\footnotesize \bf  Y}
% \psfrag{A}[lt]{\footnotesize \bf  A}
% \psfrag{B}[rb]{\footnotesize \bf  B}
% \psfrag{C}[rb]{\footnotesize \bf  C}
% \psfrag{D}[lt]{\footnotesize \bf  D}
% \psfrag{j0}{\small $J_0$}
% \psfrag{j1}{\small $J_1$}
% \psfrag{j2}{\small $J_2$}
% \psfrag{j3}{\small $J_3$}
% \psfrag{j4}[ct]{\small $J_4$}
% \psfrag{j5}[ct]{\small $J_5$}
% \psfrag{j6}[ct]{\small $J_6$}
% \psfrag{j7}{\small $J_7$}
% \psfrag{j8}{\small $J_8$}
% \psfrag{p}[cb]{\small \txt{Discretization}}
% \psfrag{h}[cb]{\small \txt{Deviation\\Contour$\rightarrow$}}
% \includegraphics[width=3in]{numerical}
% \caption{Numerical integration technique for computing the dynamic parameters for the case of a circular robot model$e.g.$ a SEGWAY RMP 200}\label{fignumerical}
% \end{figure}
%##############################################################
%##############################################################
\begin{defn}
 The delayed deviation $\Delta_d(t,\eta)$ is defined as:
\begin{gather}
 \Delta_d(t,\eta) = \Theta(  p(t+\eta(t)) , \zeta(t) )
\end{gather}
where $\eta(t)$ is some delay function satisfying $\forall \tau \in \mathbb{R}, \eta(\tau) \geq 0$.
\end{defn}

\begin{defn}\label{defcdel}
 The delay corrected deviation $\Delta(t)$ as a function of the operation time $t$ is defined as:
\begin{gather}
 \forall t \in \mathbb{R} , \ \Delta(t) = \mathop{arginf}_{\displaystyle \Delta_d(t,\eta) : \forall \tau \in \mathbb{R}, \eta(\tau) \geq 0} \big \vert  \big \vert \Delta_d(t,\eta) \big \vert  \big \vert
\end{gather}
\end{defn}
Note that Definition~\ref{defcdel} incorporates the possibility that the delay may vary in the course of mission execution.
\textit{We will make the assumption that although the delay may vary, it does so slowly enough to be approximated as a constant function over relatively short intervals.}

If we further assume that we can make observations only at discrete intervals, we can approximately compute $\Delta(t)$ over a short interval $I = [t_{init},t_{final}]$ as follows:
\begin{gather}
\forall t \in I, \Delta(t) = \mathop{argmin}_{ \Theta(  p(t+\eta) , \zeta(t) ) :\eta \in \mathbb{N}\cup \{0\}} \big \vert  \big \vert\Theta(  p(t+\eta) , \zeta(t) )\big \vert  \big \vert
\end{gather}
Furthermore, the approximately constant average delay $\eta^\star$ over the interval $I$ can be expressed as:
\begin{gather}
 \eta^\star = \mathop{argmin}_{\eta \in \mathbb{N}\cup \{0\}}\big \vert  \big \vert\Theta(  p(t+\eta) , \zeta(t) )\big \vert  \big \vert
\end{gather}
Since the delay may  vary slowly, the computed value of $\eta^\star$ may vary from one observation interval to another. For each interval $I_k \in \{I_1,\cdots,I_M\}$, one can obtain
the approximate probability distribution of the delay corrected deviation $\Delta(t)$, which is denoted as $\DST^{[k]}$.
Therefore, from a computational point of view, $\DST^{[k]}$ is just a histogram constructed from the $\Delta(t)$ values for the interval $I_k$ (for a set of appropriately chosen histogram bins or intervals).
% Rigorously speaking, $\DST^k : \mathop{range}(\Delta(t)) \rightarrow 
For a sufficiently large number of observation intervals $\{I_1,\cdots,I_M\}$, one can capture the deviation dynamics of the
robotic platform by computing the expected distribution of $\Delta(t)$, $i.e.$ computing $\DST$, which can be estimated simply by:
\begin{gather}
\DST = \frac{1}{M} \sum_{k=1}^{M} \DST^{[k]}
\end{gather}
Once the distribution for the delay corrected deviation has been computed, we can proceed to estimate the probabilities of the uncontrollable transitions, as described before.
Determination of the uncertainty parameters in the navigation model then allows us to use the proposed optimization to compute 
optimal plans which the robot can execute. We summarize the sequential steps in the next section.
%##############################################################
%######################################################################
%##############################################################
\section{Summarizing $\nustar$ Planning \& Subsequent Execution}\label{secsummalgo}
The complete approach is summarized in Algorithm~\ref{AlgorithmH}.
The planning and plan assembly steps (Lines 2 \& 3) are to be done either offline
or from time to time when deemed  necessary in the course of mission execution. Replanning may be necessary
if the dynamic model changes either due to change in the environment or due to variation in the 
operational parameters of the robot itself, $e.g.$, unforeseen faults that alter or restrict the kinematic or dynamic degrees of freedom.
Onwards from Line 4 in Algorithm~\ref{AlgorithmH} is the set of instructions needed for mission execution.
Line 5 computes the set of states to which the robot can possibly move from the current state. 
% 
%######################################################################
%######################################################################
\begin{algorithm}[!hb]
 \footnotesize \SetLine
\linesnumbered
% \dontprintsemicolon
  \SetKwInOut{Input}{input}
  \SetKwInOut{Output}{output}
  \SetKw{Tr}{true}
   \SetKw{Tf}{false}
  \caption{Summarized Planning \& Mission Execution}\label{AlgorithmH}
\Input{Model $\Gnm$}
% \Output{Plan} 
\Begin{ 
\tcc*[h]{{\red \bf $\longleftarrow$ Planning \& Plan Assembly $\longrightarrow$ }}\;
Compute decomposed plans $\boldsymbol{\nu_\#}^{[k]}$ \tcc*[l]{{\color{DeepSkyBlue4} Algorithm~\ref{figflow}} }
Compute assembled plan $\boldsymbol{\nu_\#}^{\mathbf{A}}$ \tcc*[l]{ {\color{DeepSkyBlue4} Algorithm~\ref{Algorithm_assembly}} }
\tcc*[h]{{\bf \color{red} $\longleftarrow$ Mission Execution $\longrightarrow$} }\;
\While{\Tr }{
Find set of neighbors: $N(q_i) = \big \{q \in Q: \exists \sigma \in \Sigma_C \ \mathrm{s.t.} \  q_i \xrightarrow{\sigma} q \big  \}$\;
Compute $Q_{next} = \big \{ q_j \in N(q_i): \forall q_k \in N(q_i), \  \boldsymbol{\nu_\#}^{\mathbf{A}}\vert_j > \boldsymbol{\nu_\#}^{\mathbf{A}}\vert_k \big \}$\;
{\color{Green4} Choose one state $q_{next}$ from set $Q_{next}$}\;
Attempt to move to $q_{next}$\tcc*[l]{{\color{DeepSkyBlue4} May be unsuccessful }}
Read current state $q_i$\tcc*[l]{ {\color{DeepSkyBlue4}Possibly $q_i\neq q_{next}$ }}
\eIf{$q_i == q_\Gl$}
{
Mission Successful \textbf{\color{Green4} Terminate Loop}\;
}
{
\If{$\boldsymbol{\nu_\#}^{\mathbf{A}}\vert_i \leqq 0$}
{
Mission Failed \textbf{\red Terminate Loop}\;
}
}
}
 }
\end{algorithm}
%######################################################################
%######################################################################
%##############################################################%######################################################################
%######################################################################
\begin{figure*}[!ht]
\centering
\psfrag{gamma = 0.75}[lc]{$\gamma = 0.75$}
\psfrag{5}[cc]{\small $5$}
\psfrag{10}[cc]{\small $10$}
\psfrag{15}[cc]{\small $15$}
\psfrag{20}[cc]{\small $20$}
\psfrag{25}[cc]{\small $25$}
\psfrag{30}[cc]{\small $30$}
\psfrag{35}[cc]{\small $35$}
\psfrag{40}[cc]{\small $40$}
\psfrag{45}[cc]{\small $45$}
\psfrag{50}[cc]{\small $50$}
\subfigure[$\beta=0$]{ \includegraphics[width=1.7in]{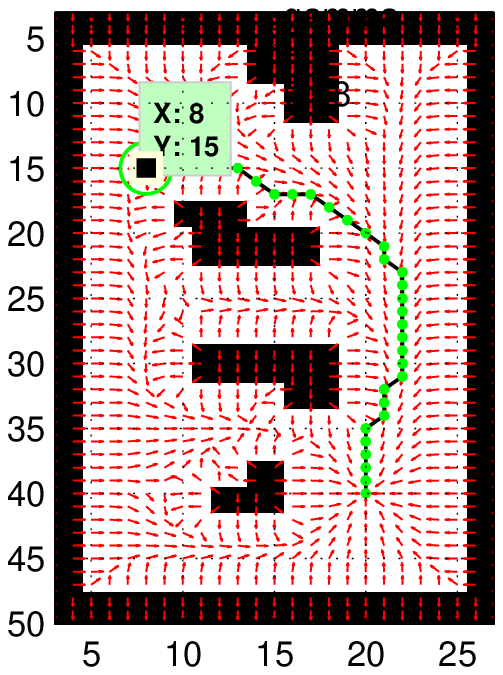}}
\subfigure[$\beta=0.35$]{ \includegraphics[width=1.7in]{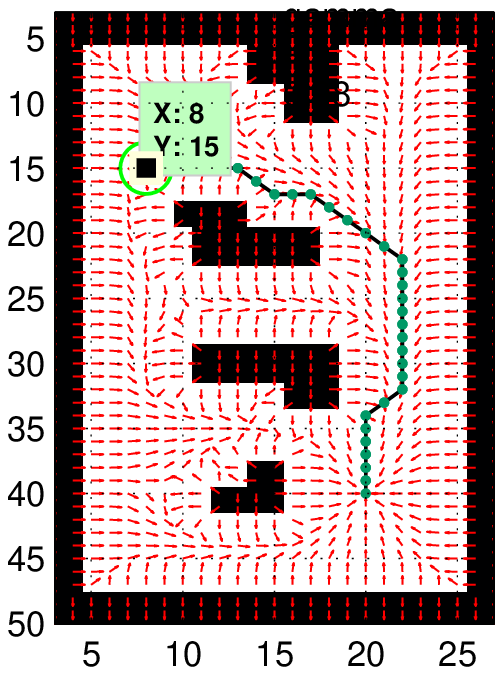}}
\subfigure[$\beta=0.75$]{ \includegraphics[width=1.7in]{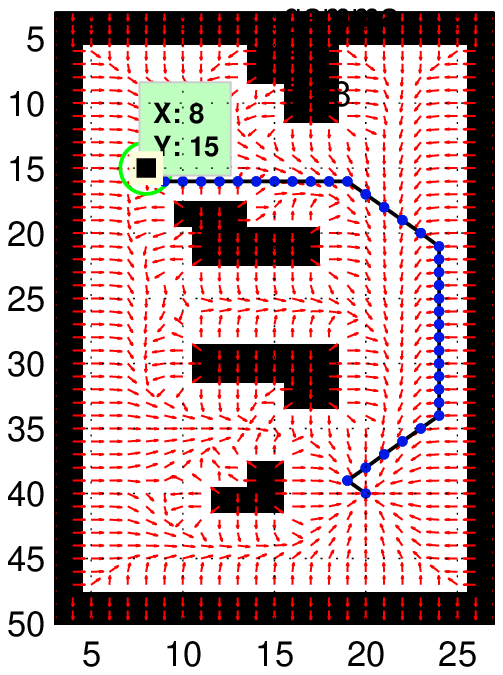}}
\subfigure[$\beta=1$]{ \includegraphics[width=1.7in]{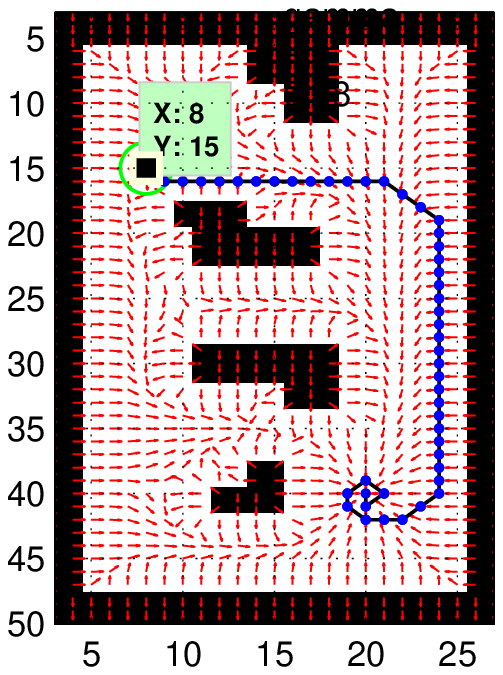}}
% \animategraphics[width=3in,autoplay,loop]{10}{EpsfilesA/F}{0}{10}
\caption{Effect of increasing turning penalty $\beta$ on making turns in choice of local transitions post $\nustar$-optimization.
Note that the \textit{potential gradient} is identical in all four cases. The start state is $(8,15)$ as shown. The gradient shown (by the arrows) is for the computed measure vector 
$\boldsymbol{\nu_\#}^{\mathbf{A}}$, and hence is identical for all four cases
 }\label{figpenalty0}
\end{figure*}
%######################################################################
%##############################################################%######################################################################
We select one state from the
set of possible next states which have a strictly higher measure compared to the current state in the computed plan. It is possible that the set of such states $Q_{next}$ (See Line 6) has more than one entry. Choice of any element in $Q_{next}$ as the next desired state  is optimal in the sense of Proposition~\ref{propoptinterpret} . However, one may use some additional criteria to aid this choice to yield plans suited to the application at hand, $e.g.$, to minimize change of current direction of travel.
For example one may choose the state from $Q_{next}$ that requires least deviation from the current direction of movement, to minimize control effort. In general, we can penalize turning 
using a specified penalty $\beta \in [0,1]$ as follows:
\begin{defn}\label{defnturnpenalty}
Given a turning penalty $\beta \in [0,1]$, the turn penalized measure values on the set of possible next states $Q_{next}$  is computed as follows:
\begin{gather}
 \forall q\in Q_{next},  \nu^\beta(q) = (1-\beta) \boldsymbol{\nu_\#}^{\mathbf{A}}(q) + \beta\cos(h(q))
\end{gather}
where $h(q)$ is the heading correction required for transitioning to $q \in Q_{next}$, which for 2D circular robots is calculated as the angular correction 
between the line joining the center of  the current state to that of $q$, and the line joining the center of the last state with that of the current one.
the direction of the last state.
The robot then chooses $q_{next}$ as the state  $q \in Q_{next}$ which has the maximum value for $\nu^\beta(q)$.
\end{defn}

The effect of penalizing turns is shown in the Figure~\ref{figpenalty0}. Note that for maximum turn penalty, the computed plan
is almost completely free from kinks. Also, note that the $\nustar$ optimization ensures that all these plans have the same probability of 
success and collision.

As stated in Line 8, the robot may not be successful to actually transition to the next chosen state due to dynamic effects. In particular, if the state that the robot actually ends up in has a
non-positive measure (due to uncertainties and execution errors), then execution is terminated and the mission is declared unfeasible from that point (See Line 14).

It is important to note that if a particular configuration maps to a navigation state with non-positive measure, then no feasible path to goal exists from that configuration, \textit{irrespective 
of uncertainty effects}. This underscores the property of the proposed algorithm that it finds  optimal feasible paths; even if the only feasible path is very unsafe, it still is the 
\textit{only} feasible path; and is therefore the optimal course of action (See Proposition \ref{propassembled} Statement 3).

%##############################################################%######################################################################
%##############################################################%######################################################################
%##############################################################%######################################################################
\section{Verification \& Validation}\label{secvv}
%##############################################################
In this section we validate the proposed planning algorithm via detailed high fidelity simulation studies
and in experimental runs on a heavily instrumented SEGWAY RMP 200 at the robotic testbed at the Networked Robotics \& Systems Laboratory (NRSL), Pennstate.
The results of these experiments adequately verify the theoretical formulations and the key
claims made in the preceding sections.

\begin{rem}\label{rempath}
In depicting simulation results in the sequel, we refer to ``computed paths/plans''. It is important to clarify, what we mean by such a computed or simulated path. The computed path 
is the sequence of configuration states that the robot would enter, if the uncertainties do not force it to deviate, $i.e.$, the path depicts the 
best case scenario under the uncertainty model. Thus, the depictions merely give us a feel for the kind of paths the robot would take; in actual implementation, the trajectories would differ between runs. Also, when we refer to lengths of the computed paths, we are referring to the lengths of the paths in the best case scenario, $i.e.$, the tight lower bound on the 
path length that will actually be encountered.
\end{rem}

%##############################################################
\subsection{Simulation Results for Circular Robots}
%##############################################################%######################################################################
%##############################################################%######################################################################
%######################################################################
\begin{figure*}[!ht]
\centering
\psfrag{Probability}[c]{\footnotesize  Probability}
\psfrag{0.4}[c]{\scriptsize  0.4}
\psfrag{0.35}[c]{\scriptsize  0.35}
\psfrag{0.3}[c]{\scriptsize  0.3}
\psfrag{0.25}[c]{\scriptsize  0.25}
\psfrag{0.2}[c]{\scriptsize  0.2}
\psfrag{0.15}[c]{\scriptsize  0.15}
\psfrag{0.1}[c]{\scriptsize  0.1}
\psfrag{0.05}[c]{\scriptsize  0.05}
\psfrag{0}[c]{\scriptsize  0}
\psfrag{R1}[lc]{\footnotesize  Run 1}
\psfrag{R2}[lc]{\footnotesize  Run 2}
\psfrag{R3}[lc]{\footnotesize  Run 3}
\psfrag{D}[rb]{\footnotesize  Deviation $\Delta$}
\subfigure[]{ \includegraphics[width=2.3in]{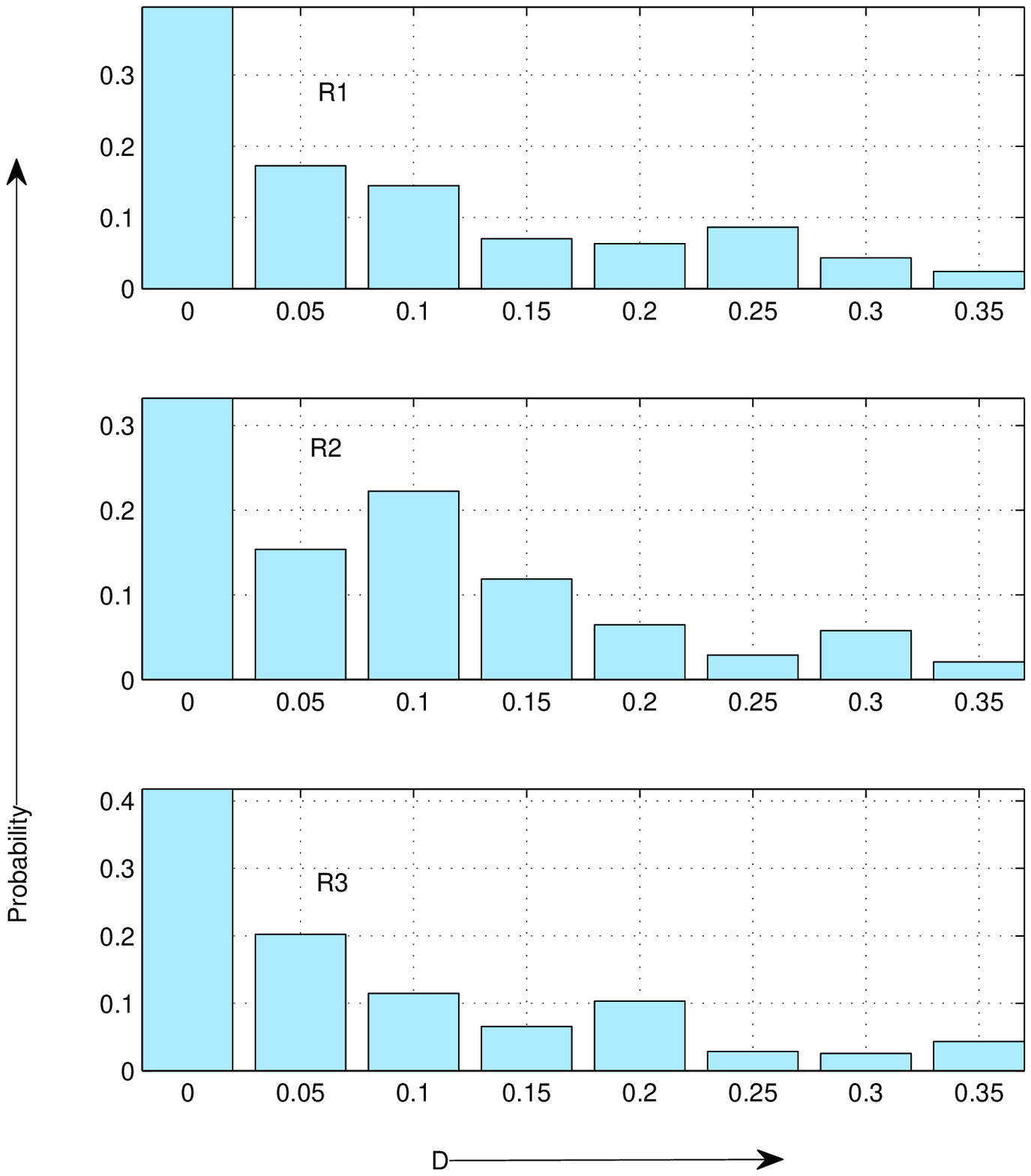}}
% +++++++++++++++++++++++++++++++++++++++++++++++++
\psfrag{D}[rt]{\footnotesize  $\eta^\star$}
\psfrag{R}[rb]{\footnotesize  Run Number}
\psfrag{N  }{\footnotesize $\eta^\star\phantom{x}$ }
\psfrag{Probability}[c]{\footnotesize  Probability}
\psfrag{0.5}[c]{\footnotesize  0.5}
\psfrag{0.4}[c]{\footnotesize  0.4}
\psfrag{0.3}[c]{\footnotesize  0.3}
\psfrag{0.2}[c]{\footnotesize  0.2}
\psfrag{0.1}[c]{\footnotesize  0.1}
\psfrag{0}[c]{\footnotesize  0}
\psfrag{16}[c]{\footnotesize  16}
\psfrag{17}[c]{\footnotesize  17}
\psfrag{18}[c]{\footnotesize  18}
\psfrag{19}[c]{\footnotesize  19}
\psfrag{25}[c]{\footnotesize  25}
\psfrag{5}[c]{\footnotesize  5}
\psfrag{10}[c]{\footnotesize  10}
\psfrag{15}[c]{\footnotesize  15}
\psfrag{20}[c]{\footnotesize  20}
\subfigure[]{ \includegraphics[width=2.1in]{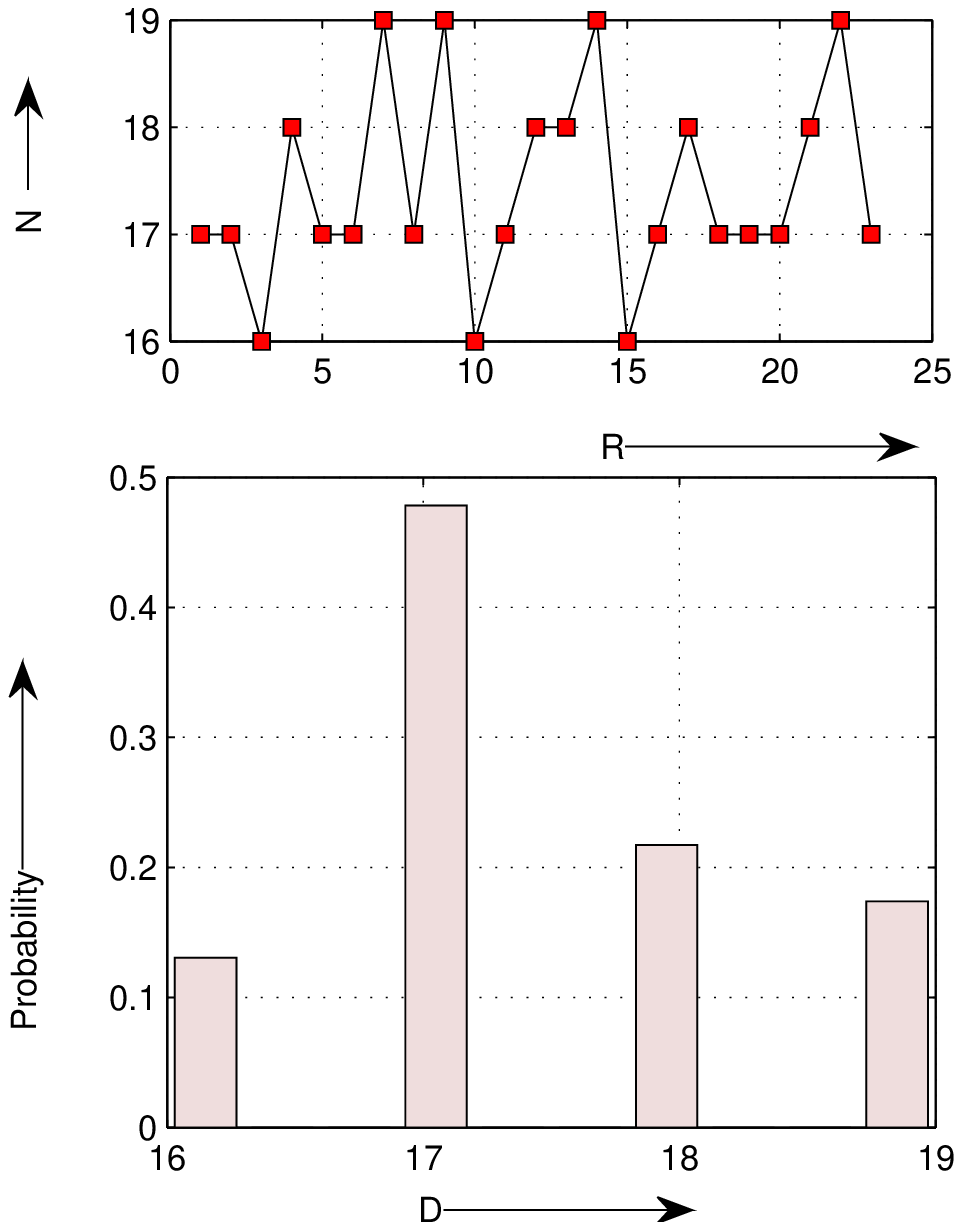}}
% +++++++++++++++++++++++++++++++++++++++++++++++++
\psfrag{p}[lb]{\Mblue \footnotesize  \txt{$\gamma=$\\$.93$}}
\psfrag{0.004}[lb]{\Mblue \footnotesize  $.003$}
\psfrag{0.035}[lb]{\Mblue \footnotesize  $.015$}
\psfrag{D}[r][b]{\footnotesize  $\overline{\Delta}$}
\psfrag{0.4}[c]{\footnotesize  0.4}
\psfrag{0.35}[c]{\footnotesize  0.35}
\psfrag{0.3}[c]{\footnotesize  0.3}
\psfrag{0.25}[c]{\footnotesize  0.25}
\psfrag{0.2}[c]{\footnotesize  0.2}
\psfrag{0.15}[c]{\footnotesize  0.15}
\psfrag{0.1}[c]{\footnotesize  0.1}
\psfrag{0.05}[c]{\footnotesize  0.05}
\psfrag{0}[c]{\footnotesize  0}
\psfrag{Probability }[c]{\footnotesize  Probability}
\psfrag{Expected Deviation Histogram}[c][c]{\footnotesize  $\phantom{-}$ $\mathbf{E}(\DST)$ Histogram}
\subfigure[]{\includegraphics[width=2.3in]{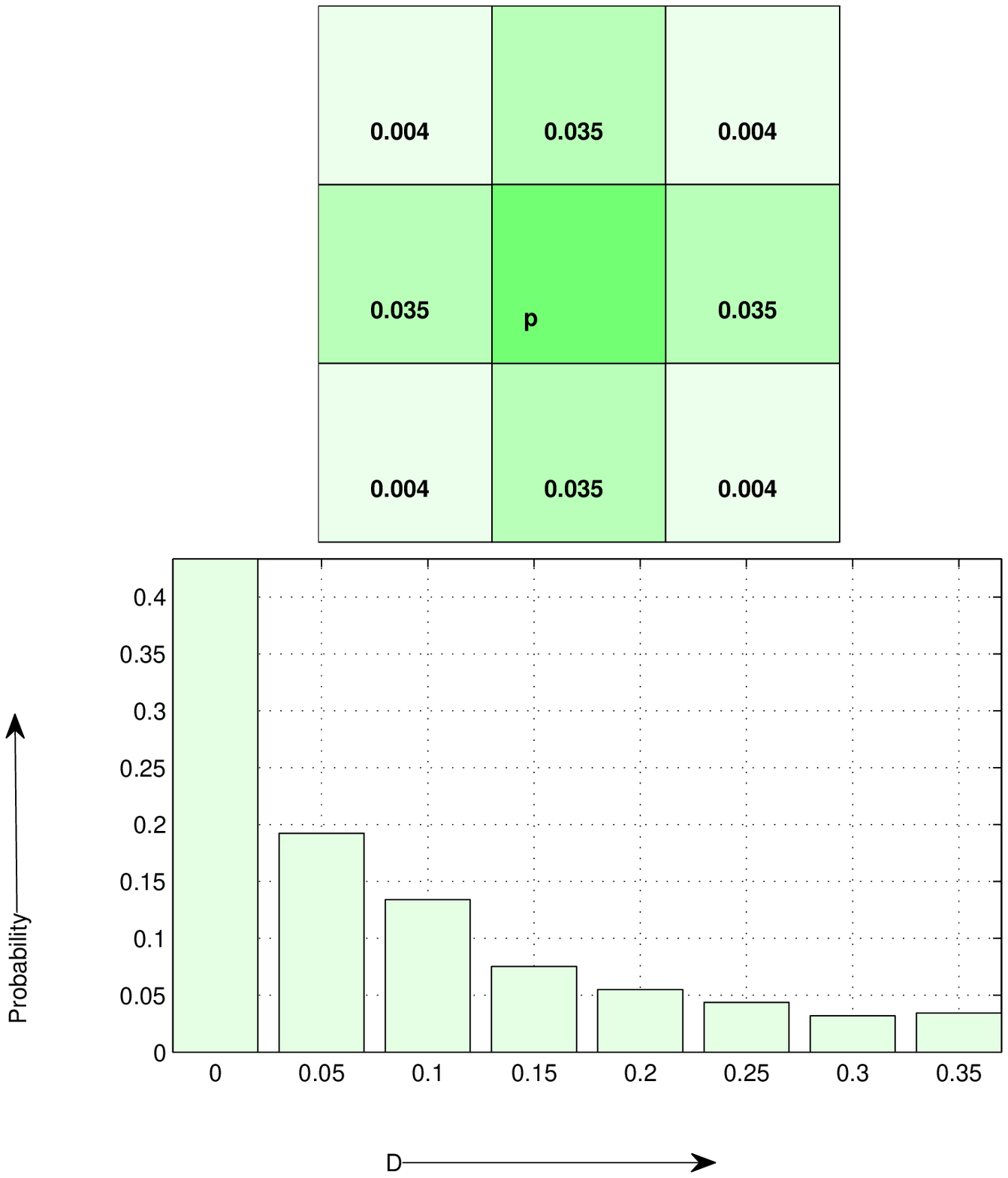}}
\caption{Computation of dynamic uncertainty parameters from observed dynamics for SEGWAY RMP 200 at NRSL, Pennstate: (a) shows the 
observed distribution $ \DST$ of the deviation $\Delta$ for different runs, (b) shows the distribution of the delay $\eta^\star$ for the various runs, Lower  plate in (c) 
illustrates the expected distribution $\mathbf{E}(\DST)$ for deviation $\Delta$ while the upper plate in (c) enumerates the probabilities of the uncontrollable transitions to the neighboring cells, and the coefficient of dynamic deviation $\gamma(\Gnm)$}\label{fighistex1}
\end{figure*}
%##############################################################
%##############################################################
%##############################################################
\begin{figure*}[!ht]
\centering
\psfrag{gamma = 1}[cb]{}
\psfrag{gamma = 0.98}[cb]{}
\psfrag{gamma = 0.9}[cb]{}
\psfrag{gamma = 0.8}[cb]{}
\psfrag{5}[rt]{\small $5$}
\psfrag{10}[rt]{\small $10$}
\psfrag{15}[rt]{\small $15$}
\psfrag{20}[rt]{\small $20$}
\psfrag{25}[rt]{\small $25$}
\psfrag{30}[rt]{\small $30$}
\psfrag{35}[rt]{\small $35$}
\psfrag{40}[rt]{\small $40$}
\psfrag{45}[rt]{\small $45$}
\psfrag{50}[rt]{\small $50$}\vspace{5pt}
\subfigure[$\gamma(\Gnm)=1$]{\includegraphics[width=2.25in]{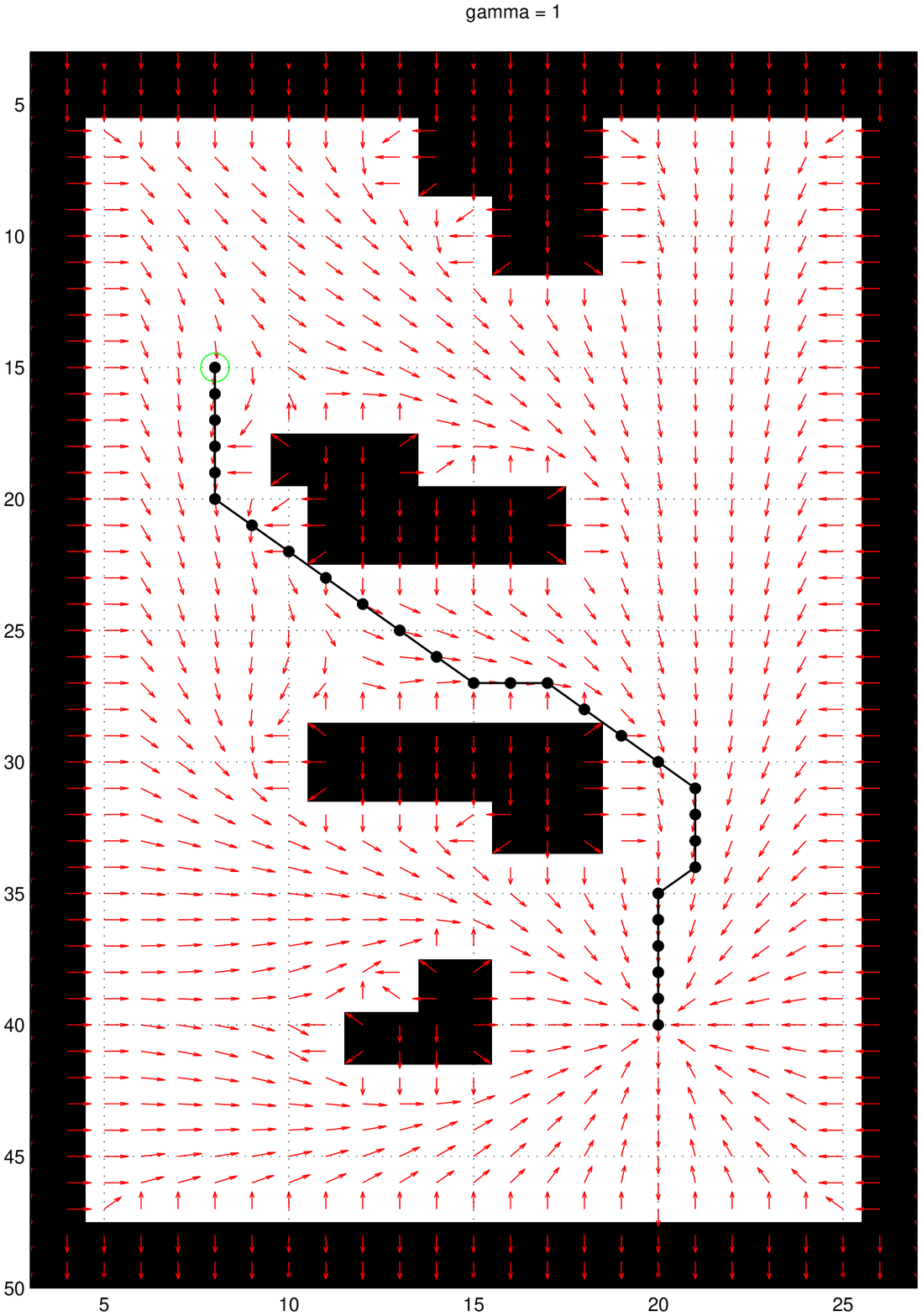}}
\subfigure[$\gamma(\Gnm)=0.98$]{\includegraphics[width=2.25in]{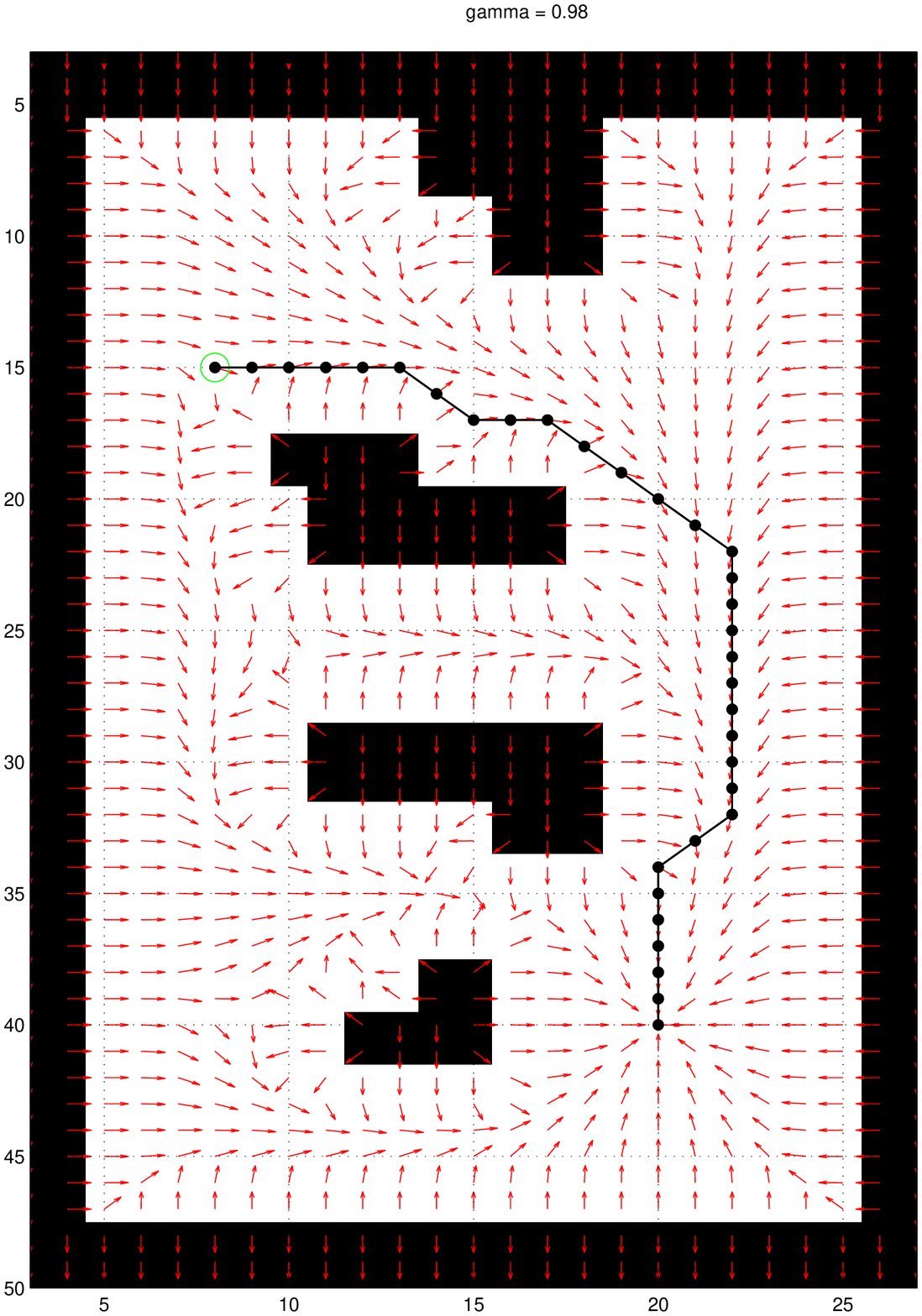}}
\subfigure[$\gamma(\Gnm)=0.9$]{\includegraphics[width=2.25in]{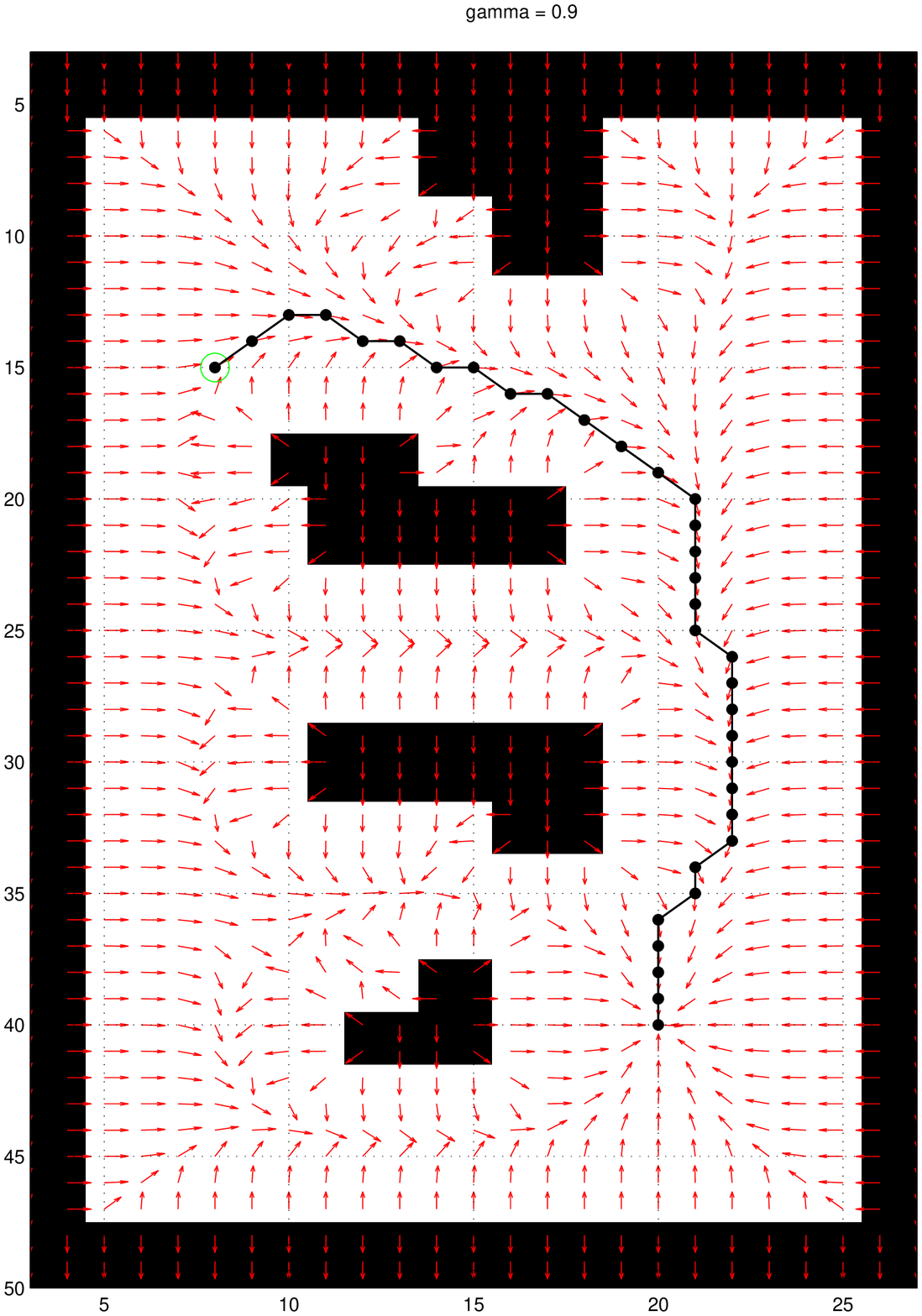}}
% \subfigure[$\gamma(\Gnm)=0.8$]{\includegraphics[width=1.7in]{ZZZZ3}}
\caption{Simulation results with a circular robot model and different values of the coefficient of dynamic deviation $\gamma(\Gnm)$.
Note the response of the navigation gradient as $\gamma(\Gnm)$ is decreased.
}\label{figsimul1}
\end{figure*}
%##############################################################%######################################################################
%##############################################################%######################################################################
% 
The recursive version of the modified $\nustar$ planning algorithm presented in this paper (See Algorithm~\ref{AlgorithmH})
 is first validated in a detailed simulation example as illustrated in Figures~\ref{fighistex1}(a-c) and \ref{figsimul1}(a-c).
The workspace is chosen to be a $50 \times 30$ grid, with obstacles placed at shaded locations, as illustrated.
The size of the workspace is chosen to correspond with the 
size of the actual test-bed, where  experimental runs on the robotic platform would be performed subsequently.
Plates (a)-(c) in Figure~\ref{figsimul1} illustrates the 
gradient of the 
optimized measure vector (by short arrows) and a sample optimal path from location $(15, 10)$ (upper,left) to the 
goal $(40,20)$ (down,right). We note that the ``potential field'' defined by the measure gradient
converges ($i.e.$ has an unique sink) at the goal.  Also, 
note that the coefficient of dynamic deviation $\gamma(\Gnm)$ is decreased, the 
algorithm responds by altering the optimal route to the goal. In particular, the optimal path for smaller values of $\gamma(\Gnm)$
stay further away from the obstacles. The key point to note here is that the the proposed algorithm guarantees that 
this lengthening of the route to account for dynamic uncertainty is \textit{optimal}, $i.e.$, further lengthening by staying even further 
from the obstacles yields no advantage in a probabilistic sense. This point has a direct practical implication; one that 
can be verified experimentally as follows. Let us assume that we have  a real-world robot equipped with on-board reactive collision avoidance, by which we can ensure that  
the platform does not \textit{actually collide} with obstacles under dynamic uncertainty, but executes corrections dictated by  local reactive avoidance. Then, the preceding result would imply that 
a using the correct value of dynamic deviation (for the specific platform) in the planning algorithm would result in routes that
require the least number of local corrections; which in turn ensures minimum time route traversals on the average.

It is important to note that the assumption of a circular robot poses no critical restrictions. Similar results can be obtained for more complex 
models as well, $i.e.$ rectangular platform with constrained turn radius. However, extension to multi-body motion planning would require 
addressing the algorithmic complexity issues that become important even for the recursive $\nustar$ for very large configuration spaces, and is a topic of future work.
% 
%######################################################################
\subsection{Simulation Results for Non-symmetric Uncertainty}
%######################################################################
As stated in the course of the theoretical development, it is possible to choose the degree of amortization or averaging 
that one is willing to allow in the specification of the navigation automaton. As a specific example, 
one may choose to compute the probabilities of uncontrollable transitions with respect to some length of 
trajectory history; the simplest case is using the previous state information to yield non-symmetric 
deviation contours (See Figure~\ref{fignoncirc}). The particular type of 
deviation contours illustrated in Figure~\ref{fignoncirc} is obtained if the platform 
has a large stopping distance and inertia, and the heading and positional estimates are more or less accurate, $i.e.$, the 
uncontrollability in the model is a stronger function of the dynamic response, rather than the estimation errors.
A typical scenario is the SEGWAY RMP 200 with good global positioning capability, in which factoring in the dynamic response is 
important due to the inverted-pendulum two-wheel kinematics. For this simulation, 
we use a navigation automaton obtained from discretizing an essentially 4D underlying continuous 
configuration space. Each state (except the obstacle state) in the navigation automaton maps to a discretized pair of locations, reflecting the 
current robot location and the one from which it transitioned to the current location. We call this the \textit{4D Model} to distinguish it from the 
the significantly smaller and simpler 2D model. Note that the 2D model can be obtained from the 4D model by merging  states with the same current location via averaging the 
probabilities of uncontrollable transitions over all possible previous states. Also note that (as stated earlier), in the absence of uncertainty, the 4D formulation 
adds nothing new; explicitly encoding the previous location in the automaton state gives us no new information. Table ~\ref{tabcomp} enumerates the comparative model sizes.
%######################################################################
%######################################################################
%######################################################################
\begin{table}[!hb]
\centering
\caption{Comparison of 4D and 2D Models}\label{tabcomp}
 \begin{tabular}{||c||c|c|c|}
  \hline
& Map Size & No. of States & Alphabet Size \\\hline
2D Model & $40\times 40$ & $1600$ & $8$ \\\hline
4D Model & $40\times 40$ & $256\times 10^4$ & $8$ \\\hline
 \end{tabular}

\end{table}

%######################################################################
%######################################################################
%######################################################################
\begin{figure}[!ht]
\centering
\begin{minipage}{3.5in}
\centering
\psfrag{U}[cc]{\bf \footnotesize \DGreen \txt{Deviation Contour}}
\psfrag{T}[lt]{\bf \footnotesize \Mblue Trajectory}
\psfrag{A}[lt]{ \small  (a)}
\psfrag{B}[lt]{ \small  (b)}
\psfrag{C}[lt]{\small  (c)}
 \includegraphics[width=3.25in]{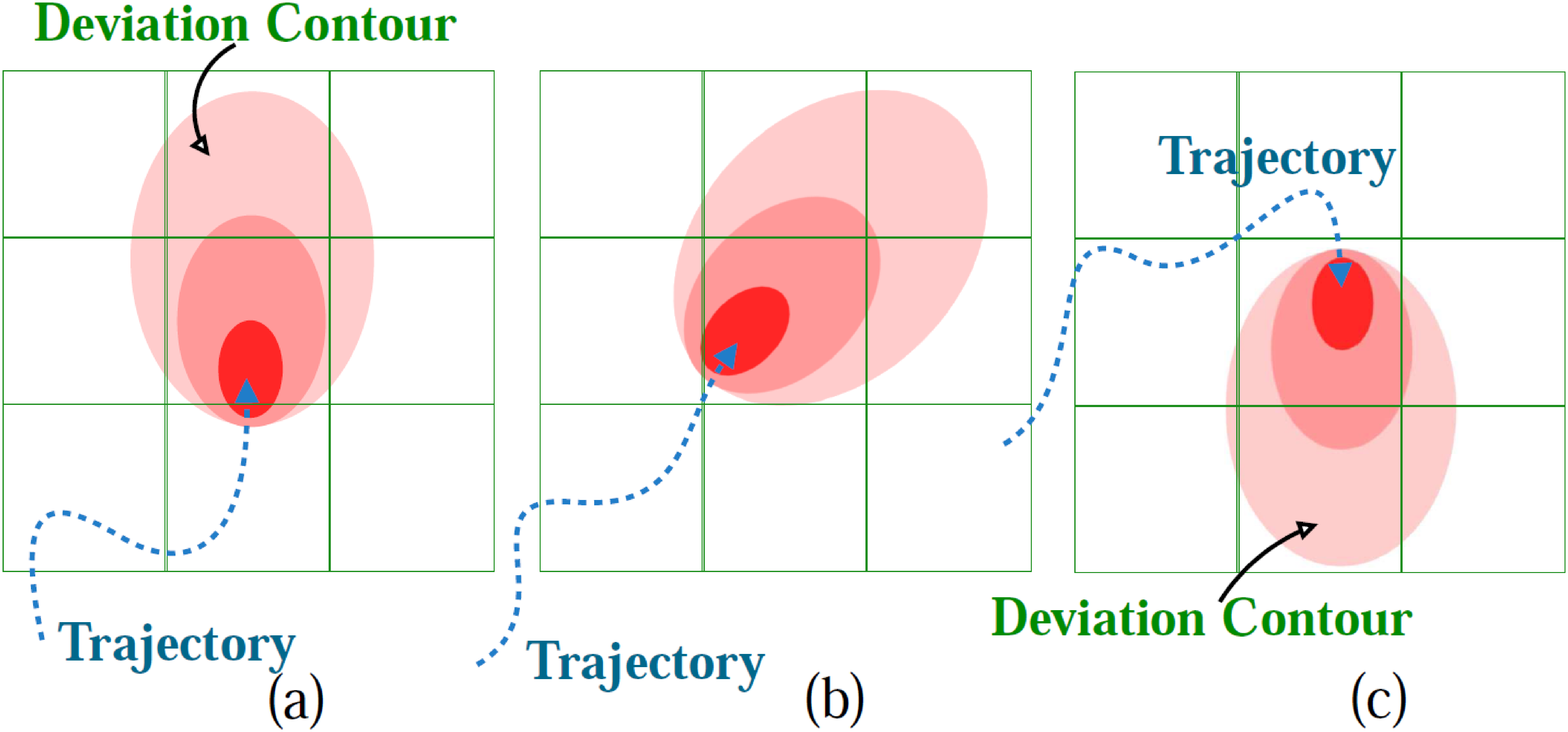}
\caption{Non-symmetric deviation contours arising from explicit dependence on the 
last discrete position for the robot}\label{fignoncirc}
\end{minipage}
%######################################################################
%######################################################################
\begin{minipage}{3.5in}
\centering
\psfrag{G}[cb]{\footnotesize GOAL}
\psfrag{I}[cl]{\footnotesize INIT}
\psfrag{A}[cl]{\footnotesize \color{Red4}\bf 4D model}
\psfrag{S}[cl]{\footnotesize \color{Green4} \bf 2D model}
\psfrag{123456789123456P}[cl]{\footnotesize \color{blue} No uncertainty}
 \includegraphics[width=3in]{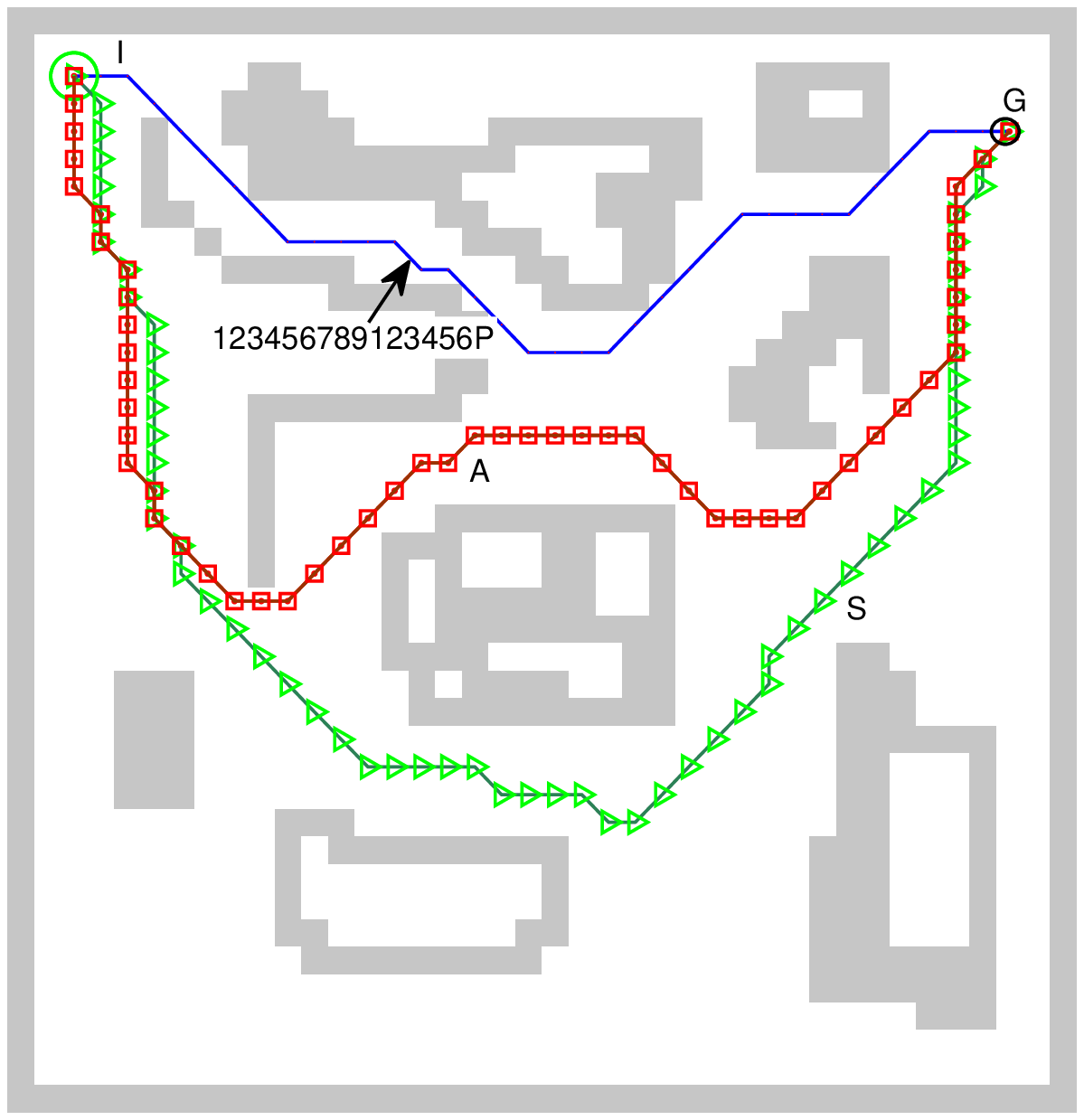}
\caption{Comparative pathlengths for the chosen initial and goal configurations: 65 steps for 2D model,  51 steps for 4D model, 36 steps for
the case with no uncertainty}\label{figcompass}
\end{minipage}
\end{figure}
%######################################################################
%######################################################################
%######################################################################

Comparisons of computed plans for a particular set of initial and goal configurations is shown in Figure~\ref{figcompass}.
To accentuate the differences in the computed plans, the deviation contours were chosen so that the probability of uncontrollable transition in the current
direction of travel is significantly more compared to that of deviating to left or right, $i.e.$, the contours are really narrow ellipses. Under such a scenario, the platform is 
more capable of navigating narrow corridors as compared to the amortized 2D counterpart. This is reflected in the paths shown in Figure~\ref{figcompass}, where the path for the 4D model 
is shorter, and goes through some of the narrow bottlenecks, while the path for the 2D model takes a safer path. Note, the path for the no-uncertainty case is even more aggressive, and shorter.
In practical implementation, when the uncontrollable probabilities are identified from observed dynamics or pre-existing continuous models, the 
differences in the two cases are often significantly less.
% 
%######################################################################
\subsection{Simulation Results for Rectangular Robots}
%######################################################################
The proposed planning algorithm is next applied to the case of rectangular robots, specifically 
ones that have a minimum non-zero turn radius. We further impose the constraint that the platform cannot travel backwards, which is 
a good assumption for robots that have no ranging devices in the rear, and also for 
aerial vehicles (UAVs). Even assuming  planar operation, this problem is essentially 3D, with the 
navigation automaton reflecting the underlying configuration states of the form $(x,y,h)$ where $h$ is the current heading, which can no longer be neglected 
due to the inability of the patform to turn in place. A visual comparison of the models for the circular and rectangular cases is shown in Figure~\ref{figrectang}(e).
The heading is discretized at $15^\circ$ intervals, implying we have $24$ discrete headings. This also means that for the same planar workspace, the number of states in the rectangular model is
about $24$ times larger the number of states for the circular model. Also, while in the circular case, we had $8$ neighbors, the number of neighboring configurations increases to $8\times 24 + 24 = 216$ However, not all neighbors can be reached via controllable transitions due to the restriction on the turn radius; we assume a maximum turn of $\pm 45^\circ$ in the model considered 
for the simulation. As explained earlier, all the neighbors may be reachable via uncontrollable transitions, which reflects uncertainties in estimation (See Figure~\ref{figruc}).

We test the algorithm with different values of $\gamma(\Gnm)$ as illustrated in Figure~\ref{figrectang}(a-d). Note the trajectories become more rounded and less aggressive (as expected) as the uncertainty is increased. Also note that the heading at the goal is different for the different cases. This is because, in the model , we specified as goals any state that maps to the 
goal location in the planar grid irrespective of the heading, $i.e.$, the problem was solved with essentially $24$ goals. Although for simplicity, the theoretical  development was presented assuming a single goal, the results can be trivially shown to extend to such scenarios. The trajectories however, will be significantly different if we insist on having a particular heading 
at the goal (See Figure~\ref{figcompdir}).
%######################################################################
%######################################################################
%######################################################################
\begin{figure*}[!ht]
\centering
\subfigure[$\gamma(\Gnm)=1$]{\includegraphics[width=2.7in]{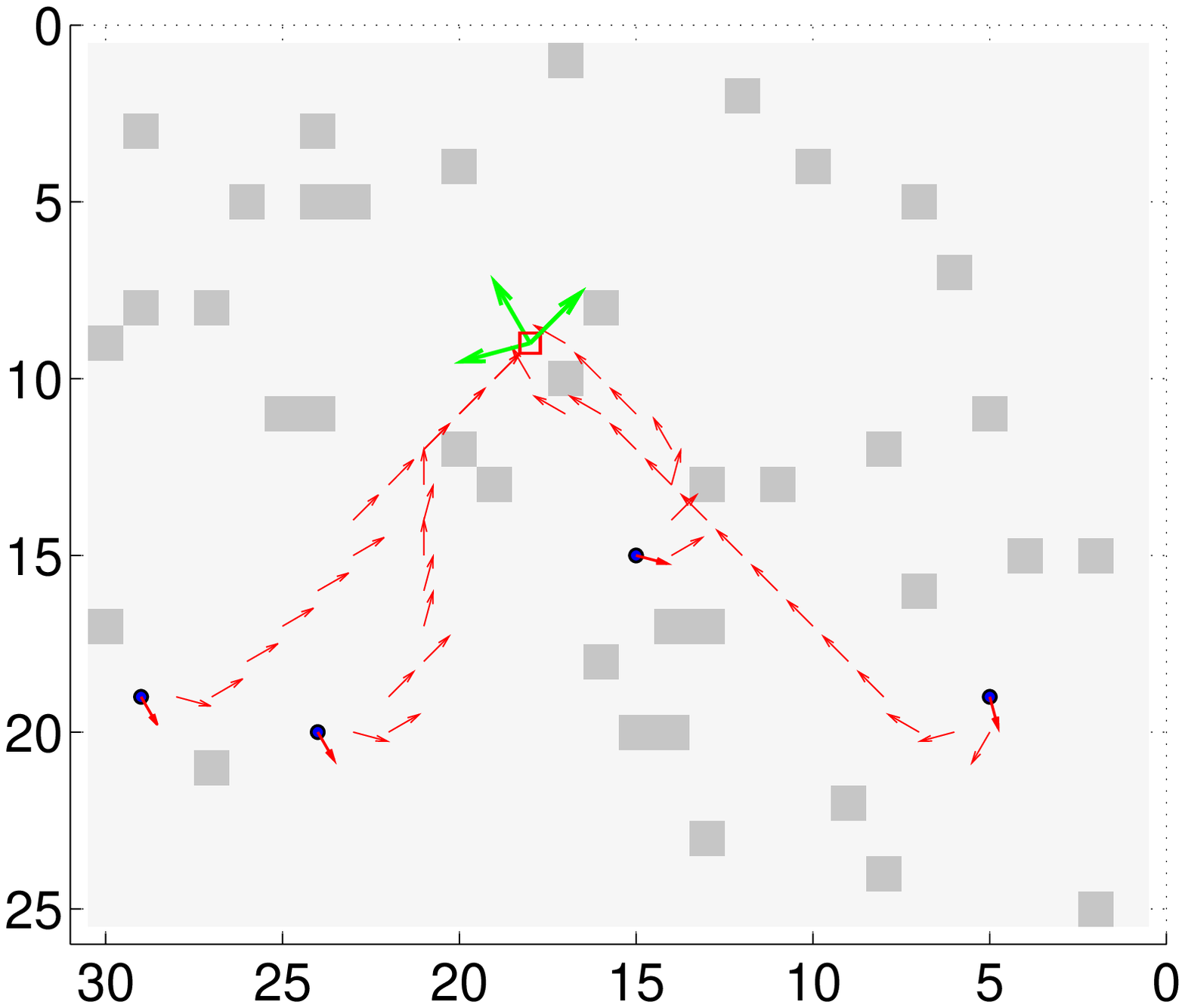}}
\subfigure[$\gamma(\Gnm)=0.8$]{\includegraphics[width=2.7in]{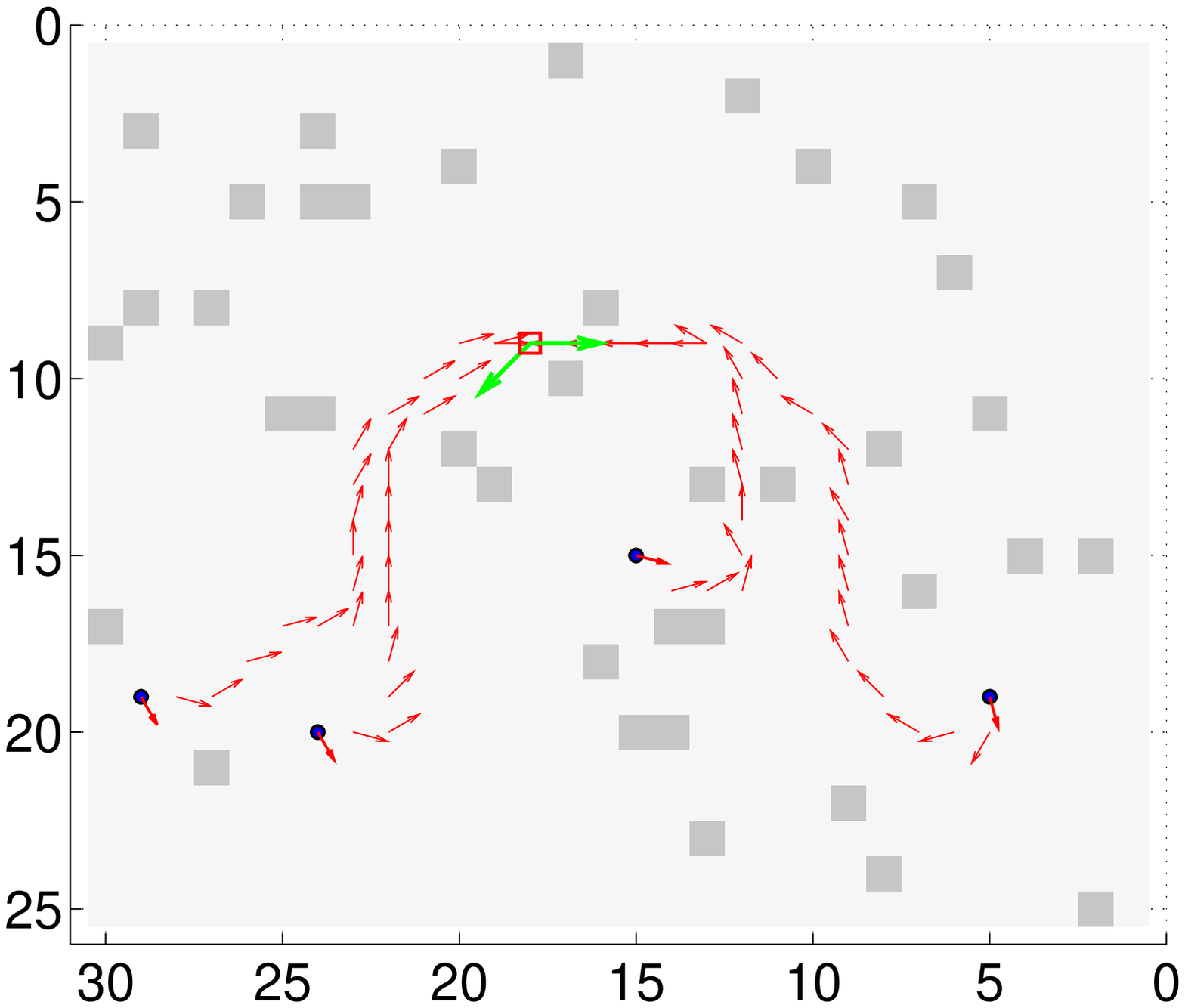}}
\subfigure[$\gamma(\Gnm)=0.5$]{\includegraphics[width=2.7in]{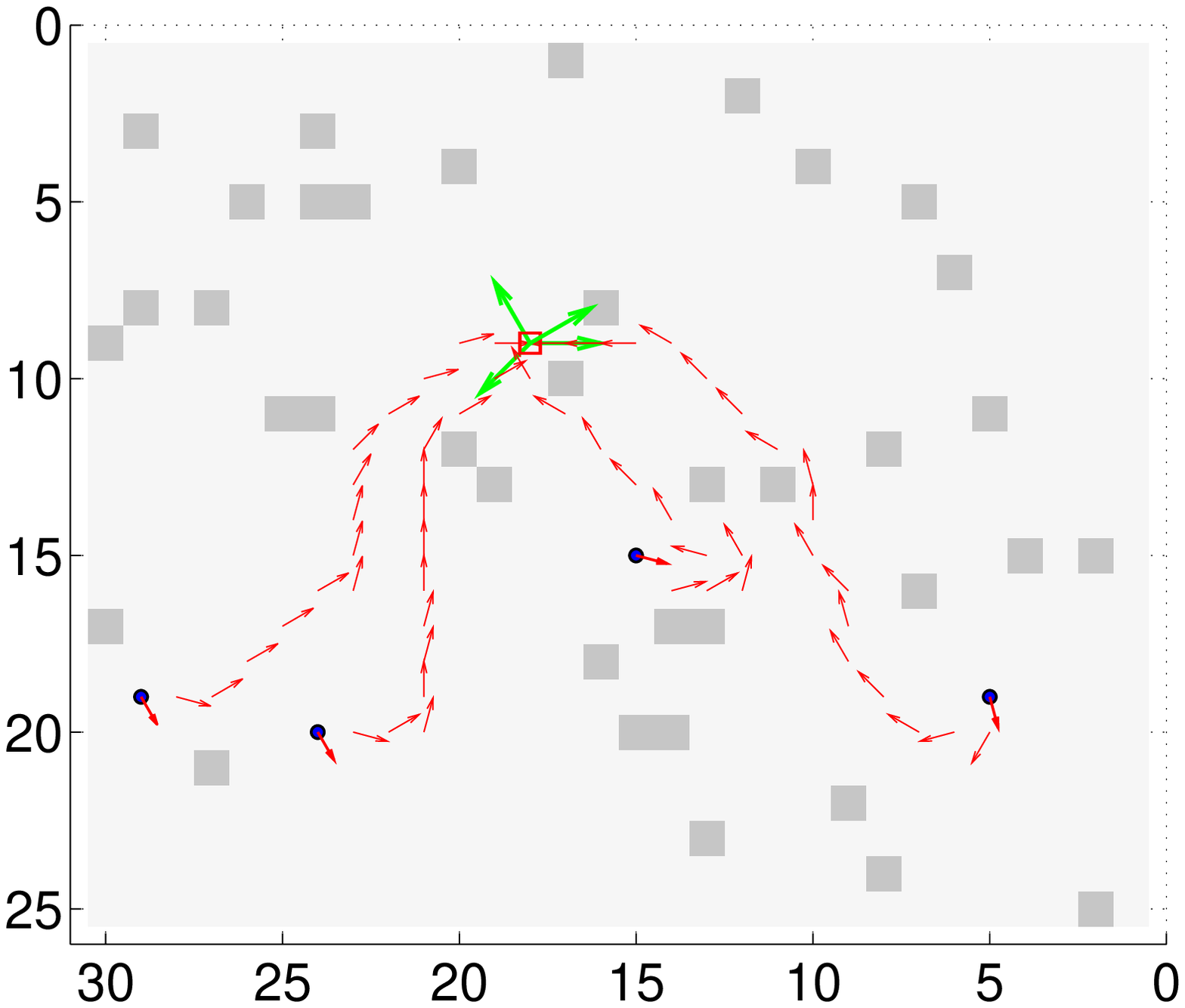}}
\subfigure[$\gamma(\Gnm)=0.1$]{\includegraphics[width=2.7in]{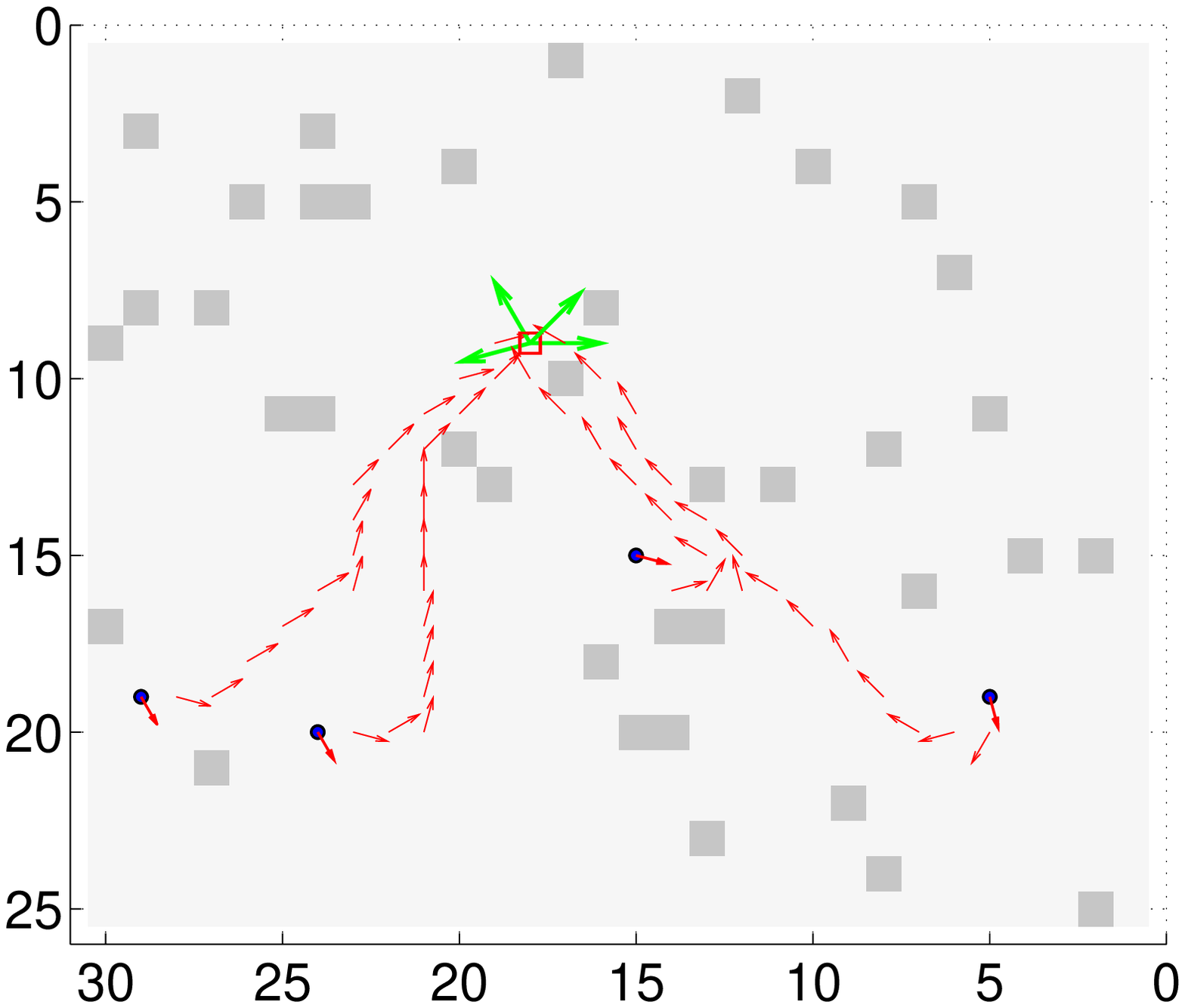}}
\hspace{10pt}
\psfrag{d}[lt]{ \small  $15^\circ$}
\psfrag{L}[lt]{ }
% \psfrag{L}[lt]{ \footnotesize  \txt{Local Dynamical\\ Structure}}
\psfrag{D}[cb]{ \color{Red4} \footnotesize  \txt{Discretized \\ Neighboring \\ Configurations}}
\psfrag{M1}[ct]{ \small   Circular Model}
\psfrag{M2}[ct]{ \small   Rectangular Model}
\subfigure[]{ \includegraphics[width=3in]{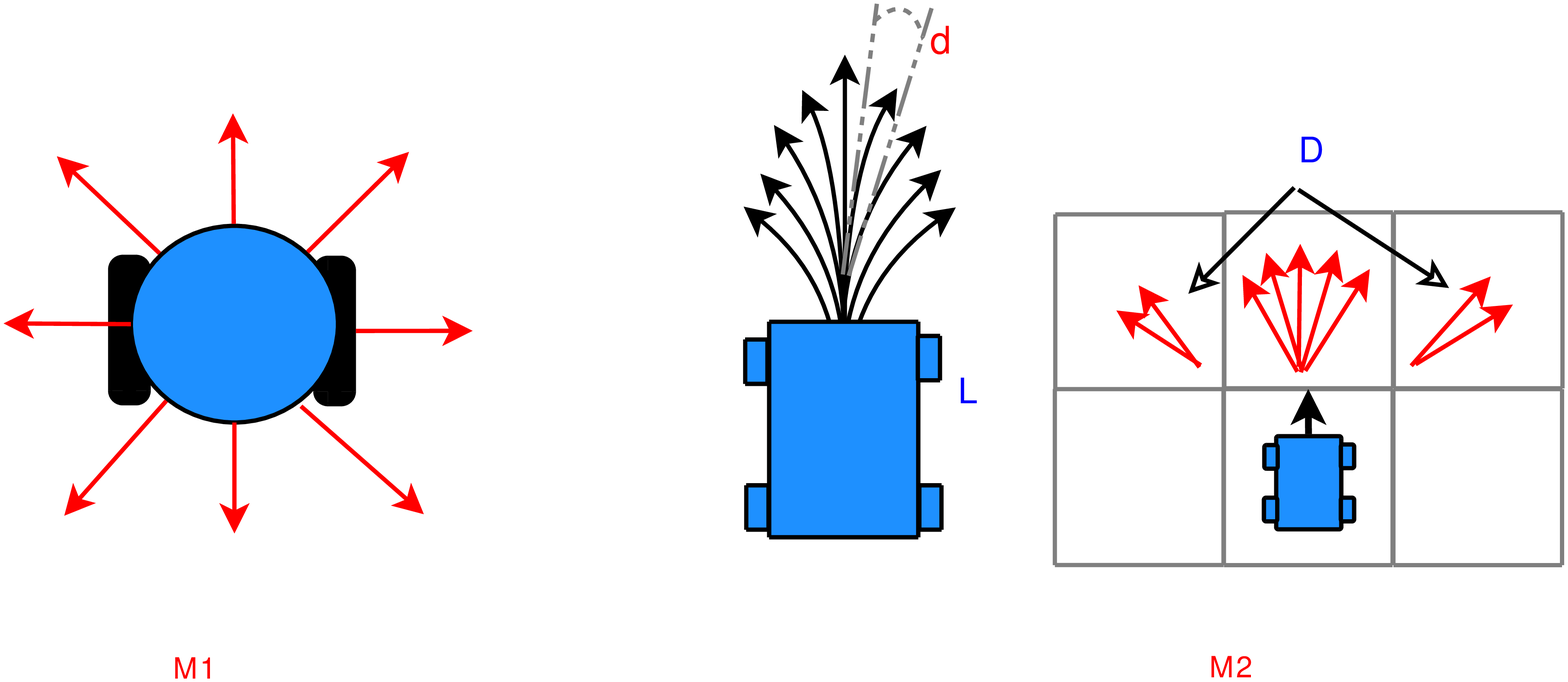}} 
\caption{Comparative trajectories for rectangular model}\label{figrectang}
\end{figure*}
%######################################################################
%######################################################################
%######################################################################
\begin{figure*}[!ht]
\centering
\psfrag{G}[cb]{\bf \footnotesize \color{Red4} Goal Location}
\psfrag{D}[cr]{\bf \footnotesize \color{Green4} \txt{Goal Heading\\ $\phantom{X}$}}
\psfrag{N}[cr]{\bf \footnotesize \color{DeepSkyBlue4} Note Difference}
\subfigure[Strict Goal Heading Requirement]{\includegraphics[width=2.7in]{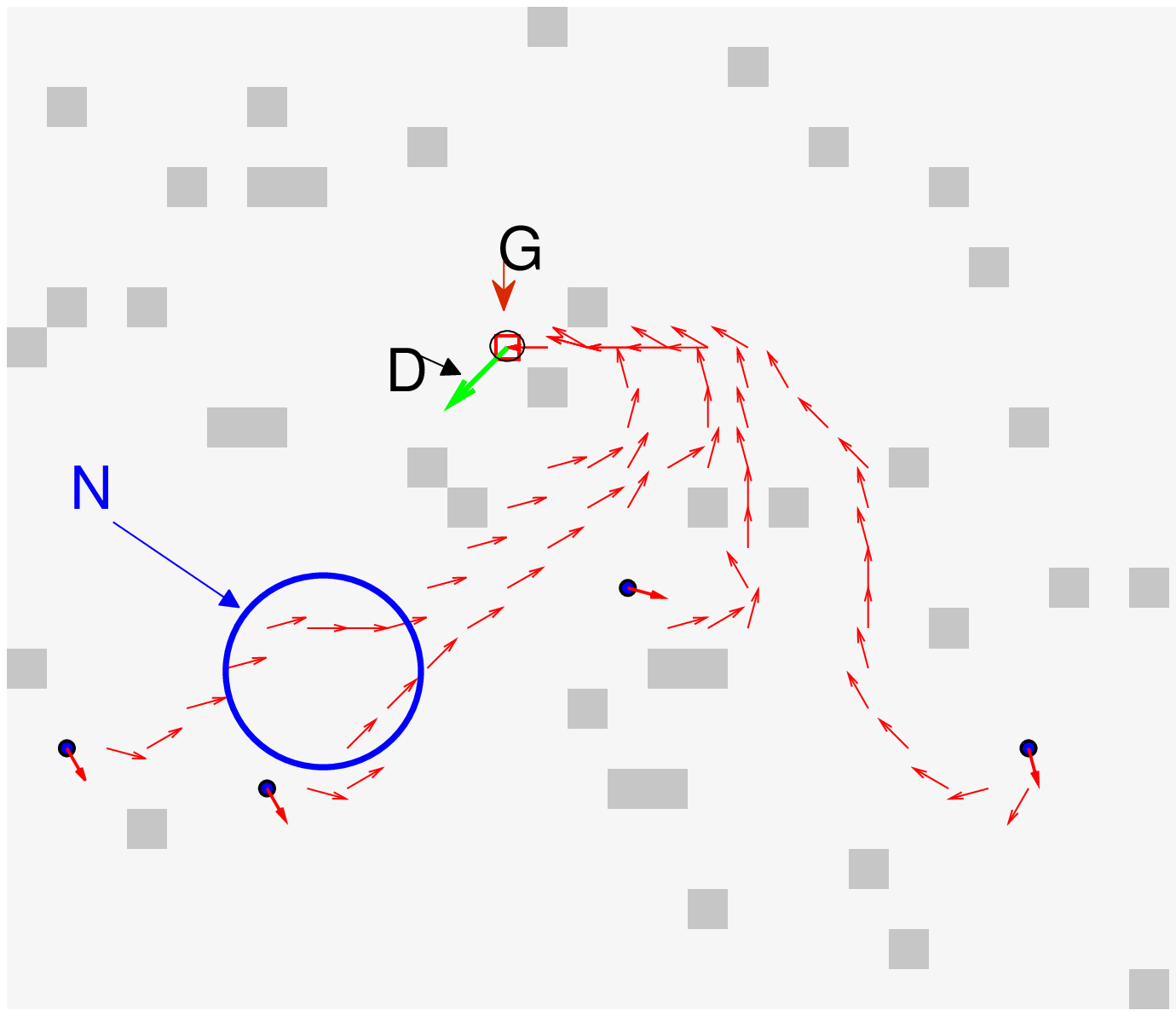}}\hspace{20pt}
\subfigure[No Heading Requirement]{\includegraphics[width=2.7in]{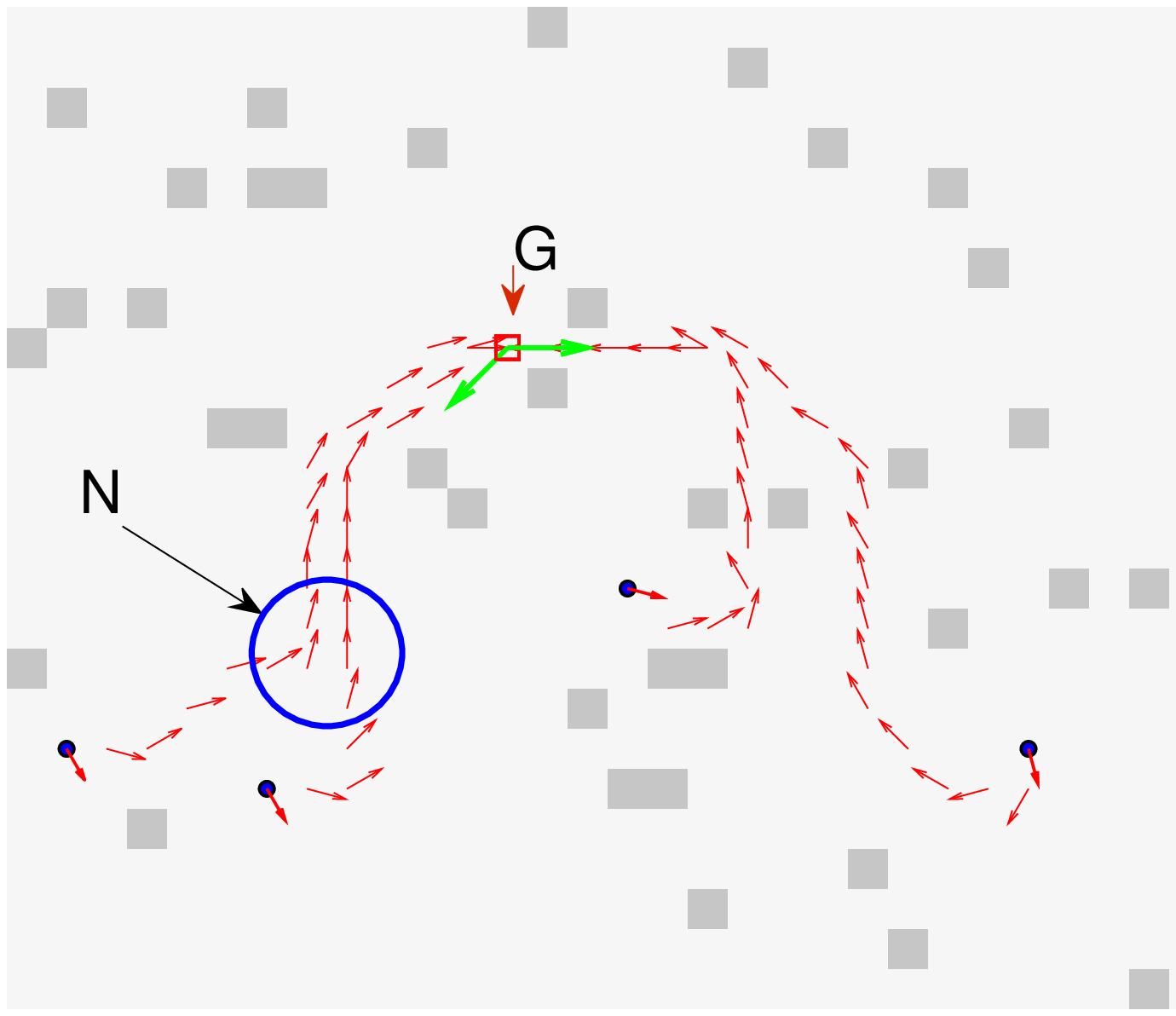}}
\caption{Comparative trajectories for rectangular model for $\gamma(\Gnm)=0.8$: Case (a) we demand that the robot reach the goal 
with a specified heading ($-150^\circ$). Case (b): Any heading at the goal is acceptable}\label{figcompdir}
\end{figure*}
%######################################################################
%######################################################################
%######################################################################
\begin{figure*}[!ht]
\centering
\psfrag{0}[ct]{ }
\psfrag{5}[ct]{ }
\psfrag{10}[ct]{ }
\psfrag{15}[ct]{ }
\psfrag{20}[ct]{ }
\psfrag{25}[ct]{ }
\psfrag{30}[ct]{ }
\psfrag{C}[cr]{\bf \footnotesize \color{DeepSkyBlue4}Circular Robot}
\psfrag{R}[cl]{\bf \footnotesize \color{Red4} Rectangular Robot}
\psfrag{N2}[bl]{\bf \footnotesize \color{blue} Note }
\psfrag{N}[br]{\bf \footnotesize \color{Red4} Note }
\subfigure[$\gamma(\Gnm)=1$]{\includegraphics[width=2.7in]{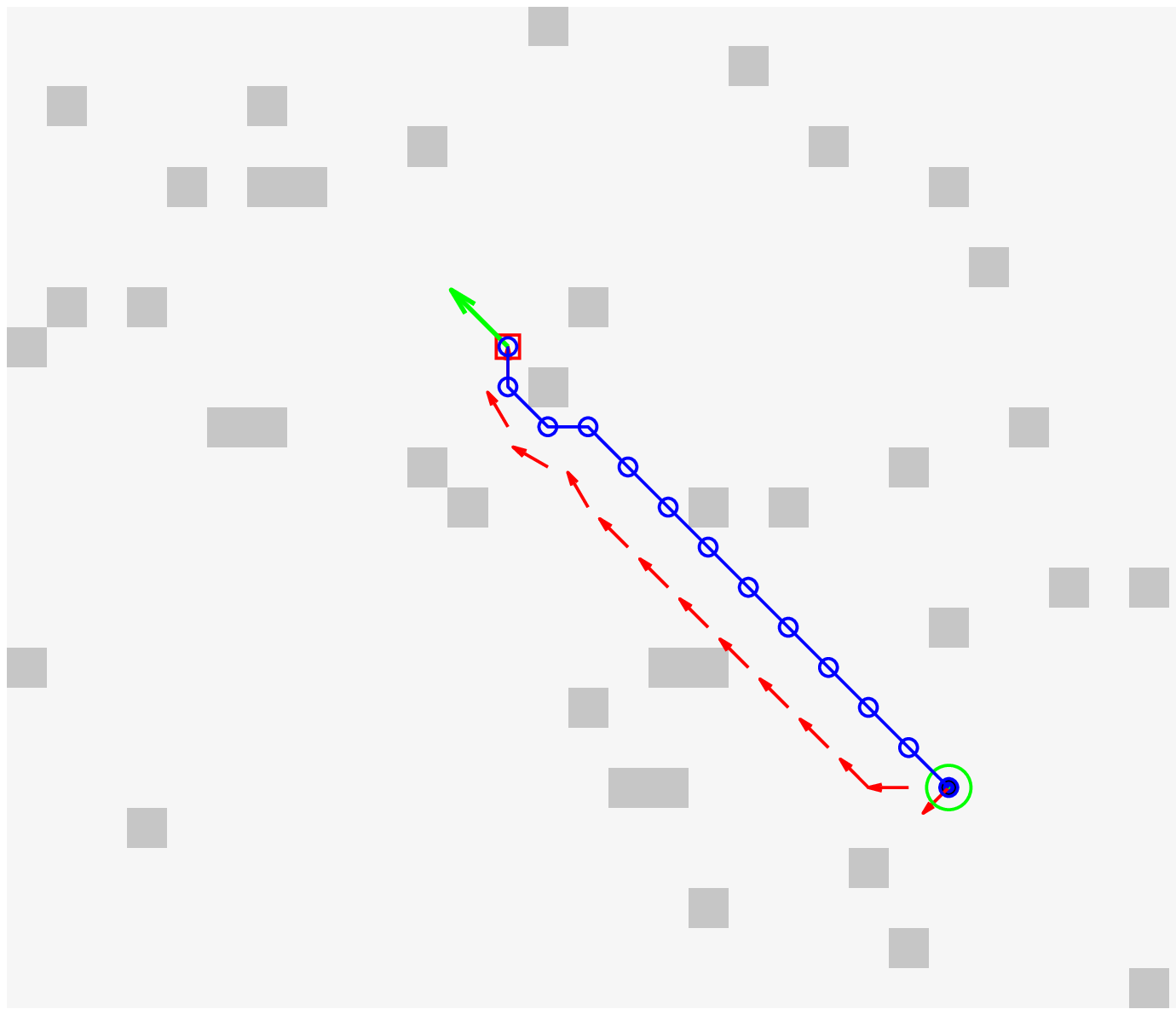}}\hspace{20pt}
\subfigure[$\gamma(\Gnm)=0.8$]{\includegraphics[width=2.7in]{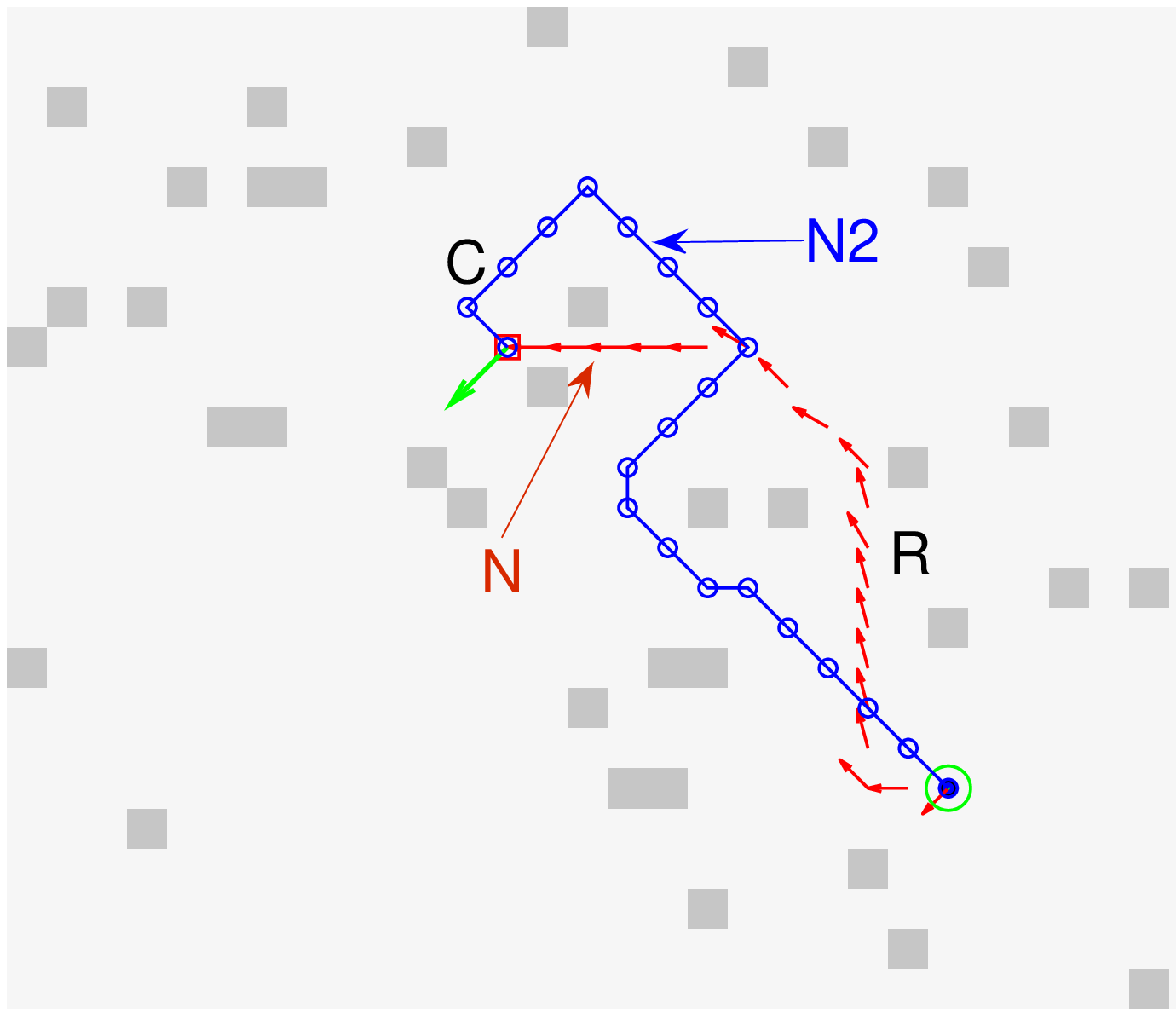}}
% \subfigure[$\gamma(\Gnm)=0.1$]{\includegraphics[width=2.25in]{recttraj3}}
\caption{Comparative trajectories for rectangular model and circular model to illustrate the effect of different uncertainty assumptions: (a) Trajectories in the absence of uncertainty (b) Trajectories with $\gamma(\Gnm)=0.8$.}\label{figrect2}
\end{figure*}

In the simulations for the rectangular model, we deliberately assume that any neighbor that cannot be reached via a controlled move is not reachable by an uncontrollable transition as well.
Although this is not what we expect to encounter in field, the purpose of this assumption is to bring out an interesting consequence that we illustrate in Figures~\ref{figrect2}(a-b). As explained above, this assumption implies that we have little or no uncertainty in local heading estimations (since the robot cannot turn in place, so there is no uncontrollable transition that alters 
heading in place). It therefore follows, that under this scenario, the platform would find it relatively safe to navigate narrow passages. This is exactly what we see in Figure~\ref{figrect2}(b), where the circular robot with same coefficient of dynamic uncertainty, really goes out of way to avoid the narrow passage, while the rectangular robot goes through.
% 
%##############################################################
%##############################################################
\subsection{Simulation Results for Mazes}\label{subsecmaze}
%######################################################################
We simulate planning in a maze of randomly placed static obstacles.
A sample case with optimal paths computed for different coefficients of dynamic deviation is illustrated 
in Figure~\ref{figpenalty}(a). A key point to note is that the optimal path is lengthens first as $\gamma(\Gnm)$ is decreased, and then 
starts shortening again, which may seem paradoxical at first sight. However, this is exactly what we expect. Recall that
the proposed algorithm minimizes the probability of collision. Also note that 
there are two opposing effects in play here; while a longer path that stays away from the 
obstacles influences to decrease the collision probability, the very fact that the path is longer has an increasing influence arising from the increased
 probability that an uncontrollable sequence would execute 
from some point in the path that leads to a collision. 
At relatively high values of $\gamma(\Gnm)$, the first effect dominates and  we can effectively
decrease the collision probability by staying away from the obstacles thereby increasing the path length.
 However, at low  values of $\gamma(\Gnm)$, the latter effect dominates, implying that  increased path lengths are no longer advantageous.
This interesting phenomenon is illustrated in Figure~\ref{figpenalty}(b), where we clearly see that the path lengths peak in the
  $\gamma(\Gnm)=0.72$ to $\gamma(\Gnm)=0.85$ range (for the maze considered in Figure~\ref{figpenalty}(a)). Also note that the 
configuration space has to be sufficiently complex to actually see this effect; which is why we do not see this phenomenon in the 
simulation results presented in Figures~\ref{figsimul1}(a-c).
% 
% 
% 
% 
%######################################################################
%##############################################################
\begin{figure*}[!ht]
\centering
\psfrag{gamma = 0.99}[cb]{$\gamma = 1$}
\psfrag{gamma = 0.9}[cb]{$\gamma = 0.9$}
\psfrag{gamma = 0.8}[cb]{$\gamma = 0.8$}
\psfrag{gamma = 0.6}[cb]{$\gamma = 0.6$}
\psfrag{5}[cc]{\small $5$}
\psfrag{10}[cc]{\small $10$}
\psfrag{15}[cc]{\small $15$}
\psfrag{20}[cc]{\small $20$}
\psfrag{25}[cc]{\small $25$}
\psfrag{30}[cc]{\small $30$}
\psfrag{35}[cc]{\small $35$}
\psfrag{40}[cc]{\small $40$}
\psfrag{45}[cc]{\small $45$}
\psfrag{50}[cc]{\small $50$}
\psfrag{60}[cc]{\small $60$}
\psfrag{70}[cc]{\small $70$}
\psfrag{80}[cc]{\small $80$}
\psfrag{G}[cc]{\bf \large\hspace{1pt}G}
\psfrag{AAAAA}[cc]{\bf \scriptsize $\phantom{XX}\boldsymbol{1}$}
\psfrag{BBBBB}[cc]{\bf \scriptsize $\phantom{XX}\boldsymbol{0.8}$}
\psfrag{CCCCC}[cc]{\bf \scriptsize $\phantom{XX}\boldsymbol{0.4}$}
\psfrag{DDDDD}[cc]{\bf \scriptsize $\phantom{XX}\boldsymbol{0.1}$}
\psfrag{X: 15}[lc]{\color{Green4} \bf S}
\psfrag{Y: 80}{}
\psfrag{X: 80}[lc]{\color{OrangeRed4}\bf G}
\psfrag{Y: 20}{}
\subfigure[]{ \includegraphics[width=2.5in]{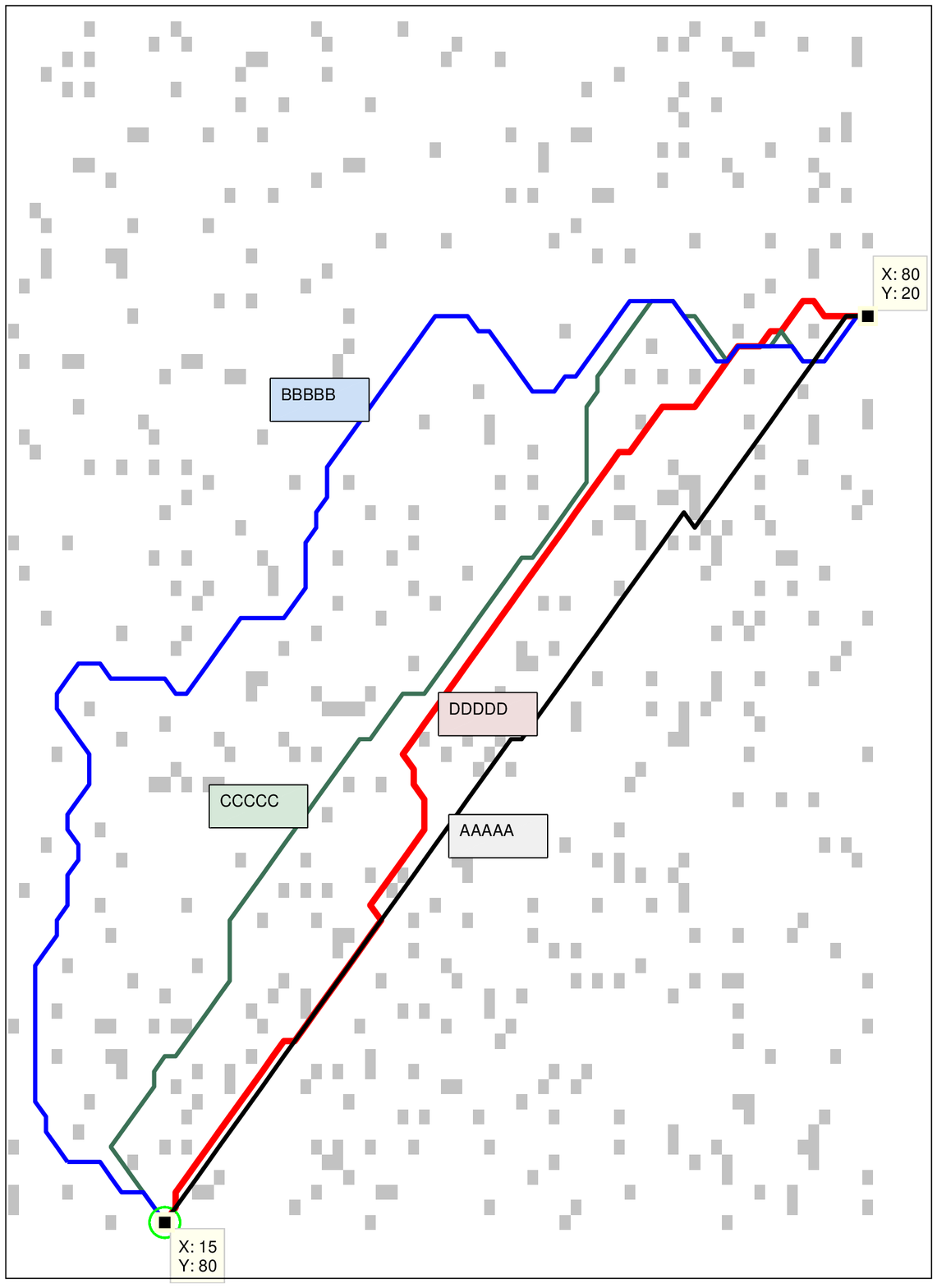}\vspace{10pt}}
\hspace{20pt}
\psfrag{g}[rc]{\Mblue \footnotesize  \txt{$\gamma(\Gnm)$}}
\psfrag{P}[rc]{\Mblue \footnotesize  \txt{Path Lengths}}
\psfrag{S}[lc]{\Mblue \footnotesize   \txt{Simulated Data}}
\psfrag{Fit}[cl]{\Mblue \footnotesize  \txt{\\ Polynomial Fit}}
\psfrag{                                                  A}[cl]{}
\psfrag{B}[cl]{}
% \psfrag{Path Length  }[lb]{\Mblue \footnotesize  \txt{Path Length$\phantom{00}$}}
\psfrag{0}[c]{\footnotesize  }
\psfrag{1}[c]{\footnotesize  $1$}
\psfrag{0.95}[c]{\footnotesize  $0.95$}
\psfrag{0.9}[c]{\footnotesize  $0.9$}
\psfrag{0.85}[c]{\footnotesize  $\phantom{0}$.85}
\psfrag{0.8}[c]{\footnotesize  $0.8$}
\psfrag{0.75}[c]{\footnotesize  0.75}
\psfrag{0.7}[c]{\footnotesize  $0.7$}
\psfrag{0.2}[c]{\footnotesize  $0.2$}
\psfrag{0.6}[c]{\footnotesize  $0.6$}
\psfrag{0.1}[c]{\footnotesize  $0.1$}
\psfrag{0.5}[c]{\footnotesize $0.5$ }
\psfrag{0.4}[c]{\footnotesize $0.4$ }
\psfrag{0.3}[c]{\footnotesize $0.3$  }
\psfrag{90}[c]{\footnotesize  $90$}
\psfrag{70}[c]{\footnotesize  $70$}
\psfrag{100}[c]{\footnotesize  $100$}
\psfrag{105}[c]{\footnotesize  $105$}
\psfrag{110}[c]{\footnotesize  $110$}
\psfrag{80}[c]{\footnotesize  $80$}
\psfrag{120}[c]{\footnotesize  $120$}
\subfigure[]{\includegraphics[width=3in]{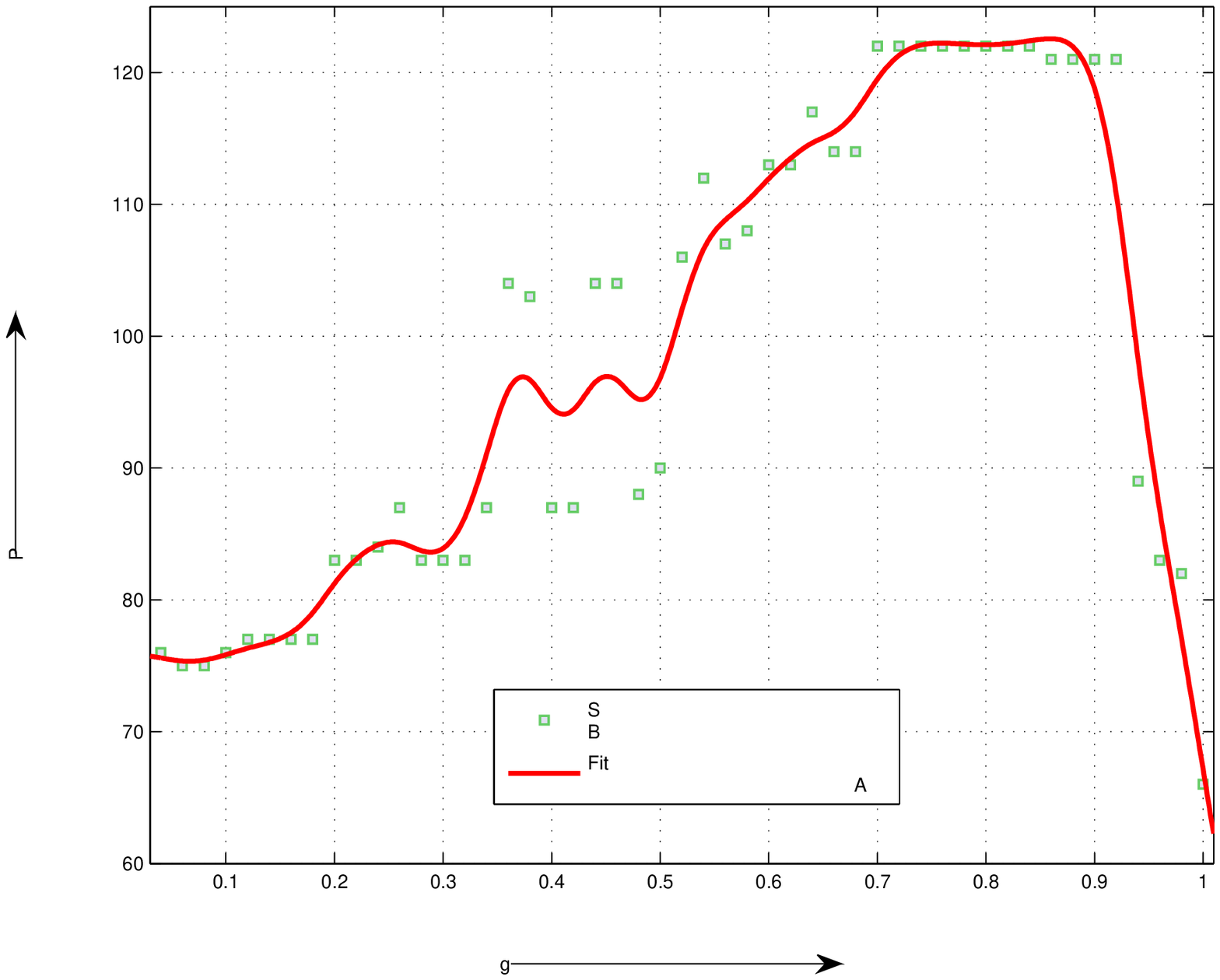}}
% \animategraphics[width=3in,autoplay,loop]{10}{EpsfilesA/F}{0}{10}
\caption{Effect of dynamic uncertainty on the optimal path lengths computed by $\nustar$. Plate (a) illustrates the optimal paths for
$\gamma(\Gnm) = 1.0,0.8,0.4,0.1$ from the start location marked by $S$ and the goal marked by $G$. Plate (b) illustrates the variation of the length of the optimal paths 
as a function of $\gamma(\Gnm)$ for the maze illustrated in (a). }\label{figpenalty}
\end{figure*}
%##############################################################
%##############################################################
\subsection{Experimental Runs on SEGWAY RMP 200}
% \subsubsection{Computation of Dynamic Parameters for SEGWAY RMP 200 Runs}
The proposed algorithm  is validated on a SEGWAY RMP 200 which is a two-wheeled robot
with significant dynamic uncertainty. In particular, the inverted-pendulum dynamics prevents
the platform from halting instantaneously, and making sharp turns at higher velocities. At low velocities, however, the platform can make zero radius turns.
The global positional fix is provided via an (in-house developed) over-head multi-camera vision system, which identifies the position and orientation of the 
robot in the testbed. The vision system yields a positional accuracy of $\pm 7.5 \ \mathsf{cm}$, and a heading accuracy or $\pm 0.1 \ \mathsf{rad}$ for a stationary robot. The accuracy deteriorates 
significantly for a mobile target, but noise correction is intentionally not applied to simulate a high noise uncertain work environment.
Furthermore, the cameras communicate over a shared  wireless network and randomly suffers from  communication delays from time to time, leading to delayed positional 
 updates to the platform. In the experimental runs conducted
at NRSL  the workspace discretized into
a $53 \times 29$ grid. Each grid location is about $4$ sq. ft. allowing the SEGWAY to fit completely  inside
each such discretized positional state which justifies the simplified \textit{circular robot} modeling. The runs are illustrated in Figure~\ref{figsimres}.
The robot was run at various allowed top speeds ($v_{max}$) ranging from $0.5 \ \mathsf{mph}$ to over $2.25 \ \mathsf{mph}$. Only the extreme cases are illustrated in the figure.
For each speed, the uncertainty parameters were estimated using the formulation presented in Section~\ref{secuncertain}. The sequence of computational steps 
for the low velocity case ($v_{max} = 0.5 \ \mathsf{mph}$) are shown in Figure~\ref{fighistex1}. Note the coefficient of dynamic deviation for the low velocity case 
turns out to be $\gamma^{low}=0.973$. For the high velocity case, ($v_{max} = 2.25 \ \mathsf{mph}$), the coefficient is computed to have a value of $\gamma^{high}=0.93$ (calculation not shown).
Also, the robot is equipped with an on-board low-level reactive collision avoidance algorithm, which ensures that the platform does not \textit{actually} collide due 
path deviations; but executes local reactive corrections when faced with such situations. The platform is equipped with multiple high frequency sonars, infra-red range finders and a high-precision SICK LMS 200 laser range finder. The data from these multiple ranging devices, mounted at various key locations on the platform, must be fused to obtain correct situational awareness. In this 
paper, we skip the details of this on-board information processing for the sake of brevity. The overall scheme is illustrated in Figure~\ref{figexptsetup}

In the experimental runs, we choose two waypoints (marked $A$ and $B$) in Figure~\ref{figsimres} (plates a,b,d,e), and the mission is to  plan and execute the optimal routes in sequence 
from $A$ to $B$, back  to $A$ and repeat the sequence a specified number of times (thirty). 
This particular mission is executed for each top speed for a range of $\gamma$ values, namely with $\gamma \in [0.75, 1]$ with increments of  $0.1$. The expectation is that using the correct coefficient of dynamic deviation (as computed above for the two chosen speeds) for the given speed,
would result in the minimum number of local corrections, leading to minimum average traversal times over thirty laps.

The results are summarized in plates (c) (for the low velocity case) and (f) (for the high velocity case) in Figure~\ref{figsimres}. Note the fitted curve in both cases attain the minimal point very close to the corresponding computed $\gamma$ values, namely, $0.97$ for $v_{max} = 0.5 \ \mathsf{mph}$ and $0.92$ for $v_{max} = 2.25 \ \mathsf{mph}$.
A visual comparison of the trajectories in the plates (a) and (d) clearly reveal that the path execution has significantly more uncertainties in the high velocity case. Also
note, that the higher average speed leads to repeated loss of position fix information in locations around $(row=20,column=35)$. Plates (b) and (e) illustrate the 
sequence of waypoints invoked by the robot in the two cases, being the centers of the states in the navigation automaton that the robot visits during mission execution. Note that in the high velocity case, the variance of the trajectories is higher leading to a larger set of waypoints been invoked.
Note that three distinct zones (denoted as Zone A, Zone B and Zone C) can be identified in the plates (e) and (f) of Figure~\ref{figsimres}. Zone A reflects the 
operation when $\gamma$ is (incorrectly assumed to be) too large, leading to too many corrections, and hence execution time can be reduced by reducing $\gamma$.
In Zone B, reducing $\gamma$ increases execution time, since now the trajectories becomes unnecessarily safe, $i.e.$ stays away from obstacles way more than necessary leading to longer than required paths and hence increased execution time. Zone C represents a sort of saturation zone where reducing $\gamma$ has no significant effect, arising from the fact that 
the paths cannot be made arbitrarily safe by increasing path lengths. Although the experimental runs were not done for smaller values of $\gamma$, we can say from the experience with maze simulations (See Section~\ref{subsecmaze}), that the execution times will start reducing again as $\gamma$ is further reduced.

These results clearly show that the approach presented in this paper successfully integrates amortized dynamical uncertainty  with autonomous planning, and 
establishes a computationally efficient framework to cyber-physical motion planning.

%##############################################################################
%##############################################################################
\begin{figure}[!ht]
\centering
\includegraphics[width=3in,angle=0]{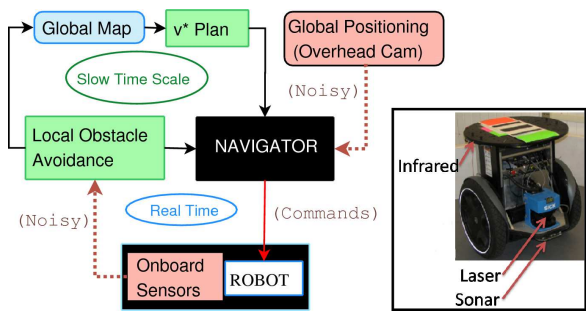}
\caption{Autonomous navigation scheme with $\nustar$ implementation  on heavily instrumented Segway RMP (shown in inset). 
% The GUI is developed on QT and is designed to run remotely as a $\Pl$ client.
}
\label{figexptsetup}
\end{figure}
%##############################################################################
%##############################################################################
%##############################################################
%##############################################################
\begin{figure*}[!ht]
\centering
\psfrag{A}[bc]{\bf \color{blue}A}
\psfrag{B}[cc]{\bf \color{blue}B}
\psfrag{5}[cc]{\small $5$}
\psfrag{10}[cc]{\small $10$}
\psfrag{15}[cc]{\small $15$}
\psfrag{20}[cc]{\small $20$}
\psfrag{25}[cc]{\small $25$}
\psfrag{30}[cc]{\small $30$}
\psfrag{35}[cc]{\small $35$}
\psfrag{40}[cc]{\small $40$}
\psfrag{45}[cc]{\small $45$}
\psfrag{50}[cc]{\small $50$}
\psfrag{60}[cc]{\small $60$}
\psfrag{70}[cc]{\small $70$}
\psfrag{80}[cc]{\small $80$}
\subfigure[]{\includegraphics[width=1.5in]{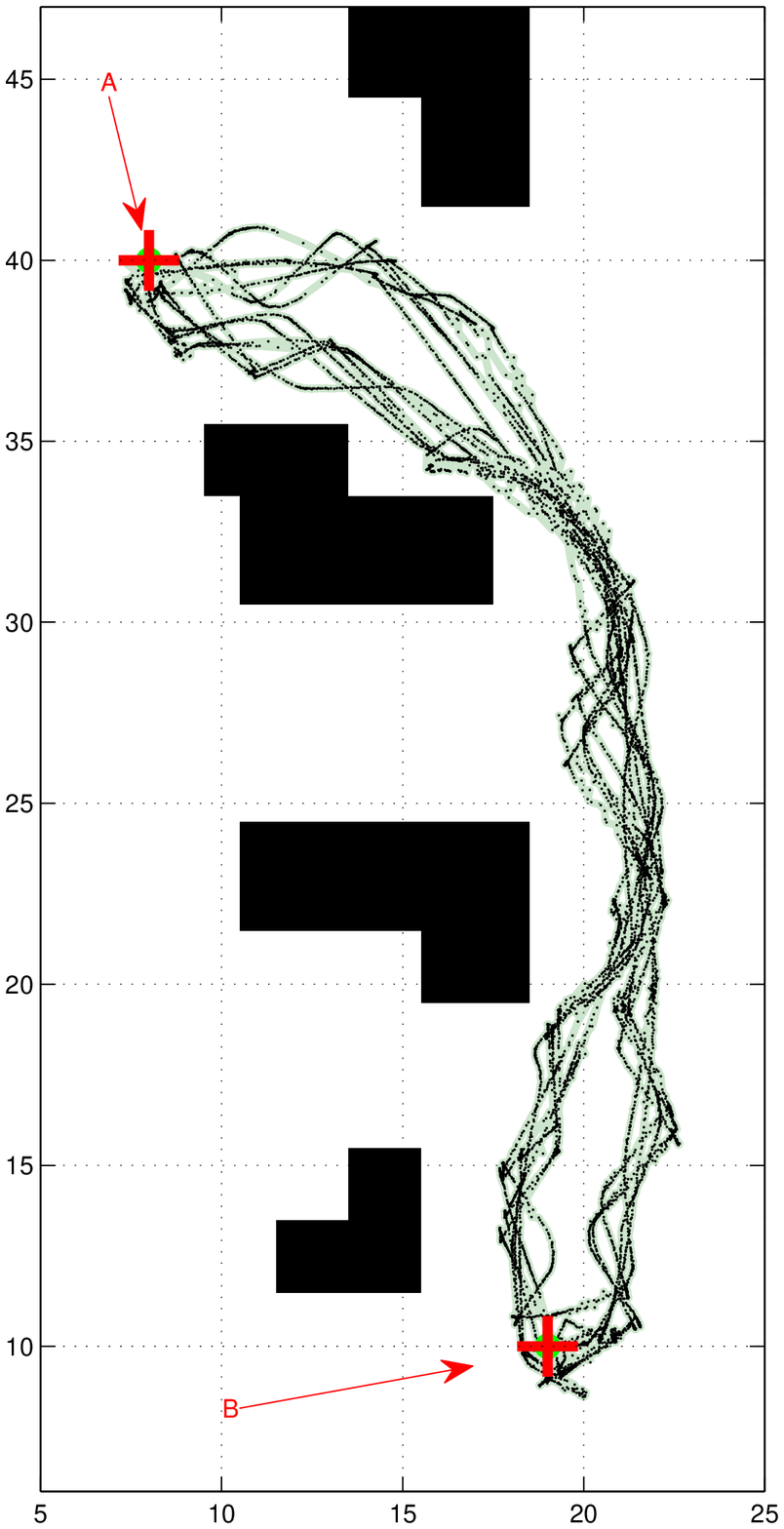}}
\subfigure[]{\includegraphics[width=1.5in]{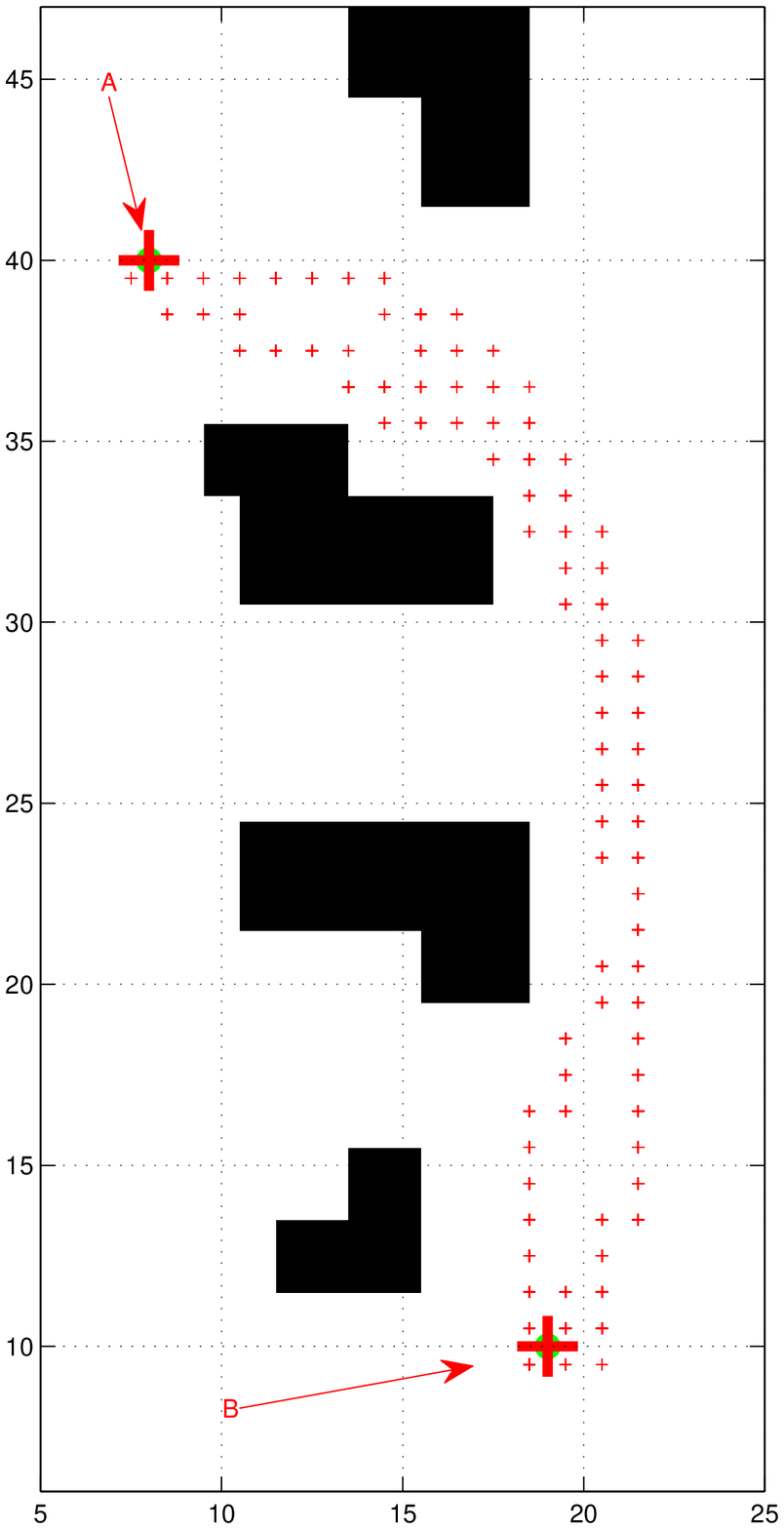}}
\hspace{20pt}
\subfigure[]{\includegraphics[width=1.5in]{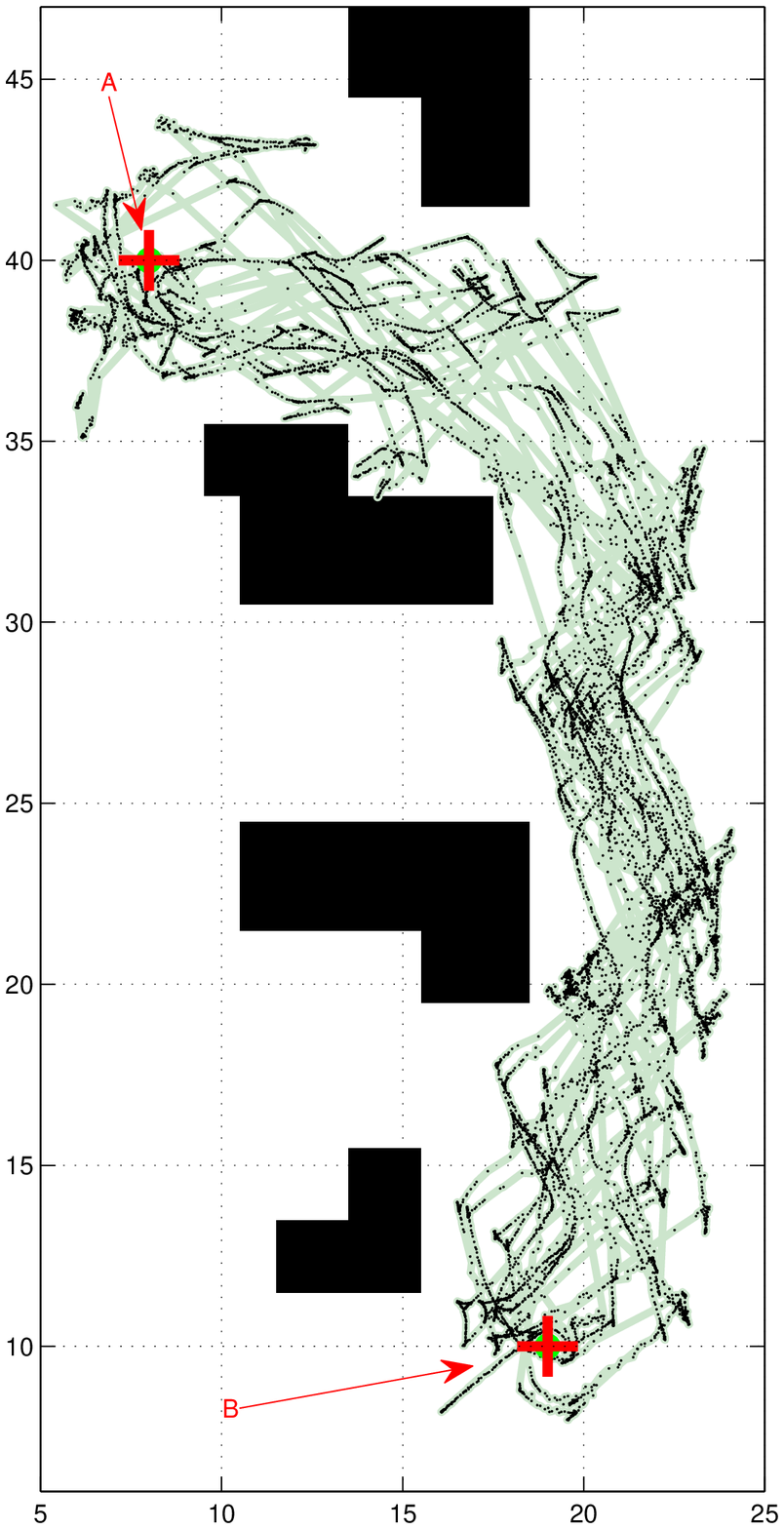}}
\subfigure[]{\includegraphics[width=1.5in]{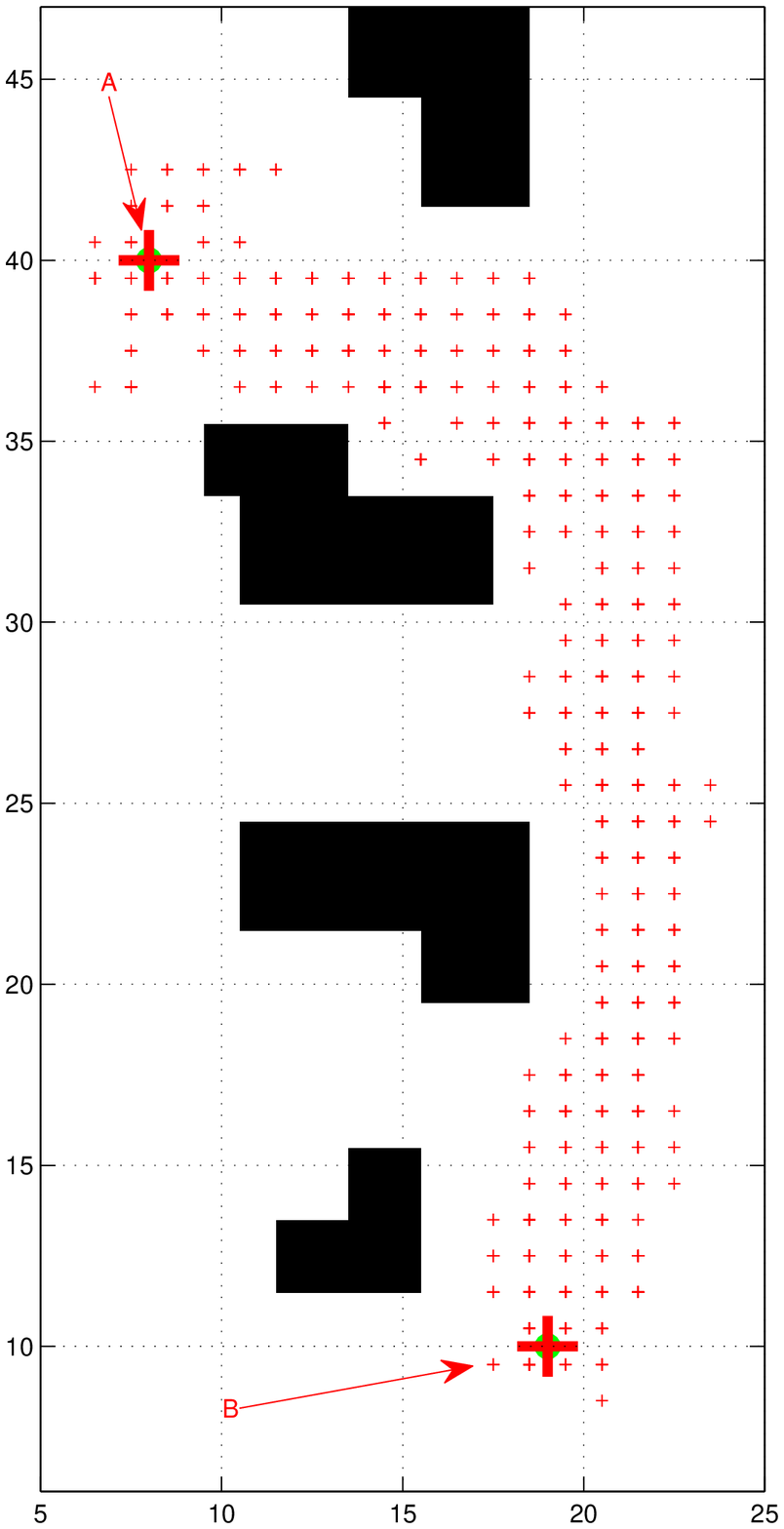}}\\
\psfrag{0.75}[lc]{\small $0.75$}
\psfrag{0.8}[lc]{\small $0.8$}
\psfrag{0.85}[lc]{\small $0.85$}
\psfrag{0.9}[lc]{\small }
\psfrag{0.91}[lc]{\small $0.91$}
\psfrag{0.92}[lc]{\small }
\psfrag{0.93}[lc]{\small $0.93$}
\psfrag{0.94}[lc]{\small }
\psfrag{0.95}[lc]{\small $0.95$}
\psfrag{0.96}[lc]{\small }
\psfrag{0.97}[lc]{\small $0.97$}
\psfrag{0.98}[lc]{\small }
\psfrag{0.99}[lc]{\small $0.99$}
\psfrag{145}[cc]{\small $145$}
\psfrag{150}[cc]{\small $150$}
\psfrag{155}[cc]{\small $155$}
\psfrag{160}[cc]{\small $160$}
\psfrag{165}[cc]{\small $165$}
\psfrag{170}[cc]{\small $170$}
\psfrag{175}[cc]{\small $175$}
\psfrag{180}[cc]{\small $180$}
\psfrag{185}[cc]{\small $185$}
\psfrag{190}[cc]{\small $190$}
\psfrag{195}[cc]{\small $195$}
\psfrag{200}[cc]{\small $200$}
\psfrag{205}[cc]{\small $205$}
\psfrag{210}[cc]{\small $210$}
\psfrag{Z1}[lc]{\BRed\bf\txt{Zone\\A}}
\psfrag{Z2}[cc]{\BRed\bf\txt{Zone\\B}}
\psfrag{Z3}[lc]{\BRed\bf\txt{Zone\\C}}
\psfrag{               D}[bc]{\Mblue \scriptsize Mean Runtime}
\psfrag{               F                    }[bc]{\BRed \scriptsize  Smooth Spline}
\psfrag{D               }[bl]{\Mblue \scriptsize Mean Runtime}
\psfrag{F                         }[bl]{\BRed \scriptsize Smooth Spline}
\psfrag{D}[rc]{\Mblue \footnotesize  \txt{$\gamma(\Gnm)$}}
\psfrag{T}[rb]{\Mblue \footnotesize  \txt{Time (sec)}}
\hspace{10pt}
 \subfigure[]{\includegraphics[width=2.65in]{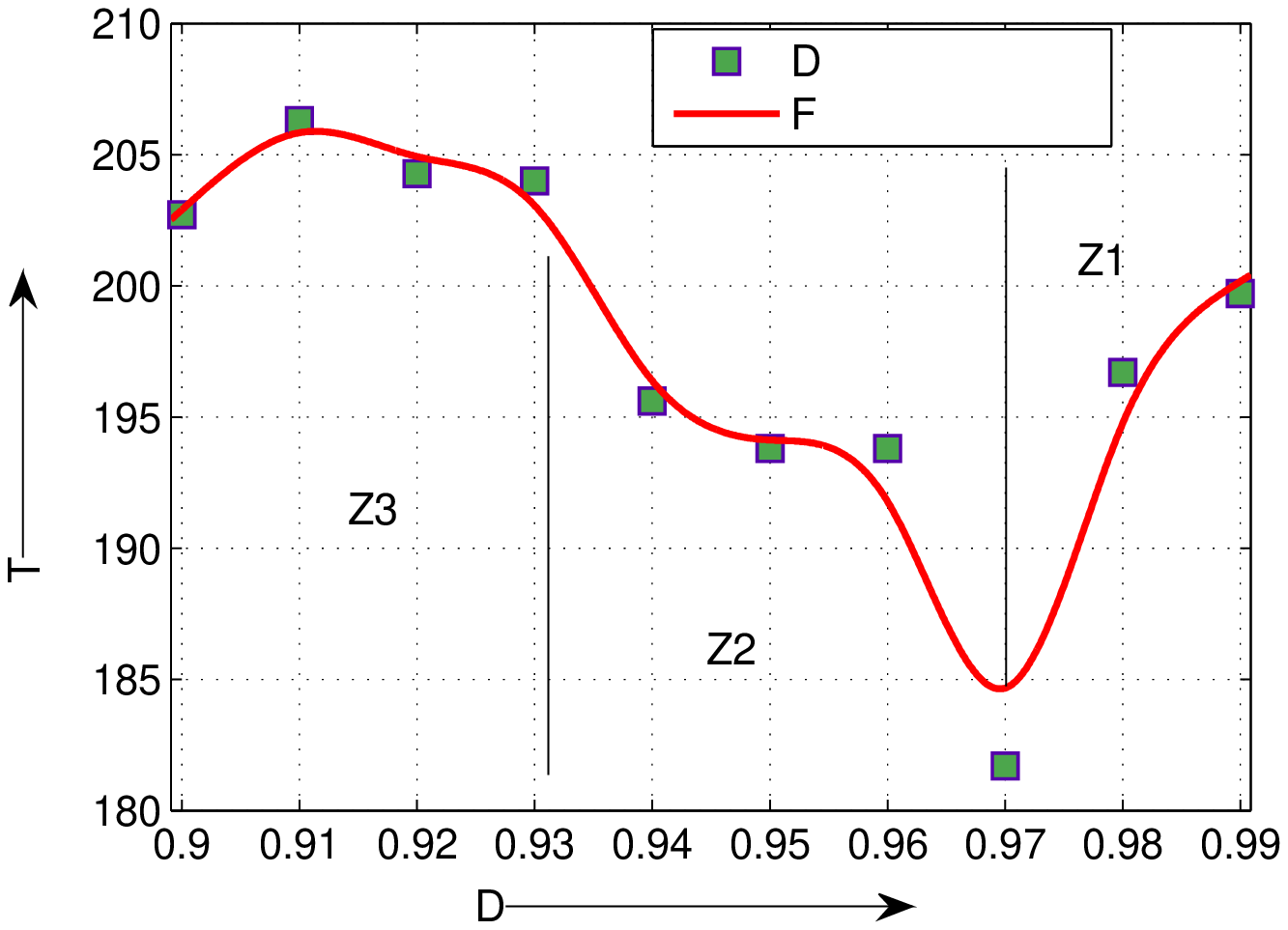}}
\psfrag{0.9}[lc]{\small $0.9$}
 \hspace{20pt}
\subfigure[]{\includegraphics[width=2.65in]{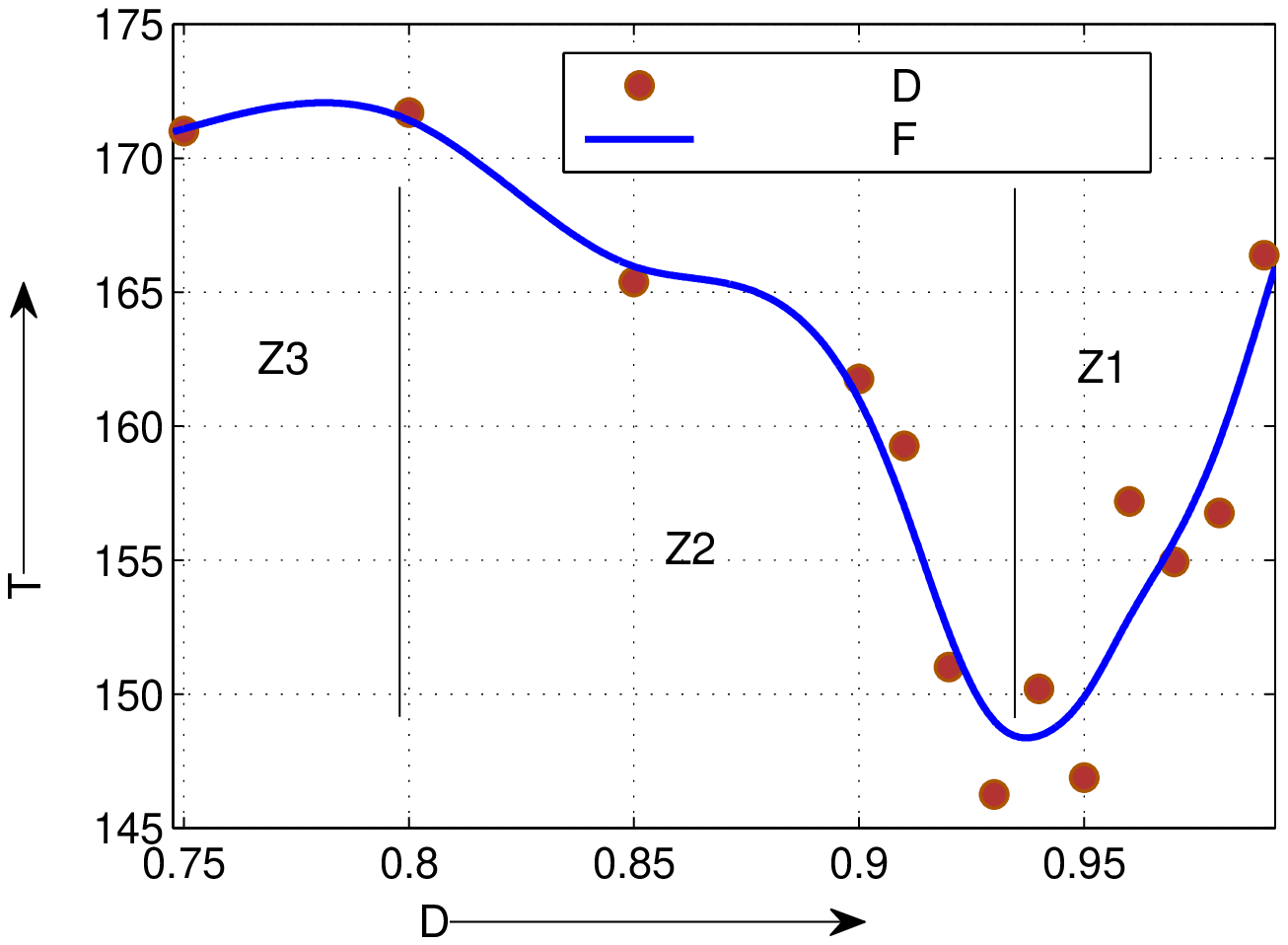}}
\caption{Experimental runs on SEGWAY RMP 200: (a)-(c) Low speed runs and (d)-(f) High speed runs:
 Plates (a) and (d) shows the trace of the robot positions as read by the overhead vision system at NRSL for
the low and high speed runs respectively. (b) and (e) shows the waypoints invoked by the robot in course 
of executing the specified mission in the low and high speed cases respectively. Plates (c) and (f) illustrate the 
variation of the mean mission execution times with the coefficient of dynamic deviation used fro planning in the 
low and high speed cases respectively.
}\label{figsimres}
\end{figure*}
%######################################################################
%######################################################################
\section{Summary \& Future Research}\label{secsummary}
The recently proposed PFSA-based path planning algorithm $\nustar$ is 
generalized to handle amortized dynamic uncertainties in plan execution, arising from the physical limitations of 
sensing and actuation, and the inherent dynamic response of the physical platforms. The key to this generalization is the 
introduction of uncontrollable transitions in the modified navigation automaton, and showing that $\nustar$ can be implemented in a recursive fashion to 
guarantee plan optimality under such circumstances. The theoretical algorithmic results is verified in detailed high-fidelity simulations and subsequently validated 
in experimental runs on the SEGWAY RMP 200 at NRSL, Pennstate.
\vspace{0pt}
\subsection{Future Work}
Future work will extend the language-measure theoretic planning algorithm to address the following problems:
\begin{enumerate}
\item \textbf{Multi-robot coordinated planning: } Run-time complexity grows exponentially with the number of agents if one attempts to solve the full Cartesian product problem. However $\nustar$ can be potentially used to plan individually followed by an intelligent assembly of the plans to take interaction into account.
\item \textbf{Hierarchical implementation to handle very large workspaces:} Large workspaces can be solved more efficiently if planning is done when needed  rather than solving the whole problem at once; however care must be taken to ensure that the computed solution is not too far from the optimal one.
One the areas of current research is an algorithmic decomposition of the configuration space such that individual blocks are solved in parallel on communicating processors, with
the interprocessor communication ensuring close-to-global optimality. We envision such an approach to be ideaally suited to scenarios involving multiple agents distributed over a large workspace which
cooperatively solve the global planning problem in an efficient resource-constrained manner.
\item \textbf{Handling partially observable dynamic events:} In this paper all uncontrollable transitions are assumed to be perfectly observable.
 Physical errors and onboard sensor failures may need to be modeled as unobservable transitions and will be addressed in future publications. A generalization of the measure-theoretic optimization technique under partial observation has been already reported~\cite{CR09p}. The future goal in this direction is to incorporate the modifications to allow $\nustar$ handle loss 
of observation and feedback information.
\end{enumerate}
%\nocite{*}
% \useRomanappendicesfalse
%##############################################################################
%##################################################
}       %####### DO NOT REMOVE ###################
%##################################################
%\clearpage

\section*{References}
\bibliographystyle{elsarticle-num}
\bibliography{BibLib1}

\begin{thebibliography}{10}
\expandafter\ifx\csname url\endcsname\relax
  \def\url#1{\texttt{#1}}\fi
\expandafter\ifx\csname urlprefix\endcsname\relax\def\urlprefix{URL }\fi
\expandafter\ifx\csname href\endcsname\relax
  \def\href#1#2{#2} \def\path#1{#1}\fi

\bibitem{Lat91}
J.-C. Latombe, Robot Motion Planning, International Series in Engineering and
  Computer Science; Robotics: Vision, Manipulation and Sensors, Kluwer Academic
  Publishers, Boston, MA, U.S.A., 1991, 651 pages.

\bibitem{Lav06}
S.~M. LaValle, Planning Algorithms, Cambridge University Press, Cambridge,
  U.K., 2006, available at http://planning.cs.uiuc.edu/.

\bibitem{KK92}
K.~Kondo, Motion planning with six degrees of freedom by
  multistrategicbidirectional heuristic free-space enumeration, IEEE
  Transactions on Robotics and Automation 7~(3) (1991) 267--277.

\bibitem{BK02}
J.~Borenstein, Y.~Koren, \href{http://dx.doi.org/10.1109/70.88137}{The vector
  field histogram-fast obstacle avoidance for mobile robots}, Robotics and
  Automation, IEEE Transactions on 7~(3) (2002) 278--288.
\newblock \href {http://dx.doi.org/10.1109/70.88137}
  {\path{doi:10.1109/70.88137}}.
\newline\urlprefix\url{http://dx.doi.org/10.1109/70.88137}

\bibitem{Lp87}
T.~Lozano-Perez, A simple motion-planning algorithm for general robot
  manipulators, IEEE Transactions on Robotics and Automation 3~(3) (1987)
  224--238.

\bibitem{AniHamHu03}
D.~A. Anisi, J.~Hamberg, X.~Hu, Nearly time-optimal paths for a ground vehicle,
  Journal of Control Theory and Applications.

\bibitem{BLL90}
J.~Barraquand, B.~Langlois, J.-C. Latombe, Robot motion planning with many
  degrees of freedom and dynamic constraints, MIT Press, Cambridge, MA, USA,
  1990.

\bibitem{Lan08}
J.~Langelaan, Tree-based trajectory planning to exploit atmospheric energy, in:
  American Control Conference, 2008, 2008, pp. 2328--2333.
\newblock \href {http://dx.doi.org/10.1109/ACC.2008.4586839}
  {\path{doi:10.1109/ACC.2008.4586839}}.

\bibitem{LOZRC09}
S.~Lahouar, E.~Ottaviano, S.~Zeghoul, L.~Romdhane, M.~Ceccarelli, Collision
  free path-planning for cable-driven parallel robots, Robotics and Autonomous
  Systems 57~(11) (2009) 1083 -- 1093.
\newblock \href {http://dx.doi.org/DOI: 10.1016/j.robot.2009.07.006}
  {\path{doi:DOI: 10.1016/j.robot.2009.07.006}}.

\bibitem{Ort09}
L.~M. Ortega, A.~J. Rueda, F.~R. Feito, A solution to the path planning problem
  using angle preprocessing, Robotics and Autonomous Systems In Press,
  Corrected Proof (2009) --.
\newblock \href {http://dx.doi.org/DOI: 10.1016/j.robot.2009.07.028}
  {\path{doi:DOI: 10.1016/j.robot.2009.07.028}}.

\bibitem{AH83}
J.~R. Andrews, N.~Hogan, Impedance Control as a Framework for Implementing
  Obstacle Avoidance in a Manipulator, ASME, Boston, MA, 1983, pp. 243--251.

\bibitem{Kh85}
O.~Khatib, Real-time obstacle avoidance for manipulators and mobile robots, in:
  IEEE International Conference on Robotics and Automation, Vol.~2, St. Louis,
  MI, 1985, pp. 500--505.

\bibitem{Kr84}
B.~H. Krogh, A generalized potential field approach to obstacle avoidance
  control, in: International Robotics Research Conference, Bethlehem, 1984.

\bibitem{KGZ07}
M.~Kumar, D.~Garg, R.~Zachery, Multiple mobile agents control via artificial
  potential functions and random motion, in: Proceedings of the ASME
  International Mechanical Engineering Congress and Exposition, ASME, Seattle,
  WA, 2007.
\newblock \href {http://dx.doi.org/Paper No. IMECE2007-41521} {\path{doi:Paper
  No. IMECE2007-41521}}.

\bibitem{SHK08}
S.~Sarkar, E.~Halland, M.~Kumar, Mobile robot path planning using support
  vector machines, in: ASME Dynamic Systems and Control Conference, ASME, Ann
  Arbor, Michigan, 2008.
\newblock \href {http://dx.doi.org/Paper No. DSCC2008-2200} {\path{doi:Paper
  No. DSCC2008-2200}}.

\bibitem{BK91-1}
J.~Borenstein, Y.~Koren, Potential field methods and their inherent limitations
  for mobile robot navigation, in: Proceedings of the 1991 IEEE International
  Conference on Robotics and Automation, 1991, pp. 1398--1404.

\bibitem{Ti90}
R.~Tilove, Local obstacle avoidance for mobile robots based on the method of
  artificial potentials, Robotics and Automation, 1990. Proceedings., 1990 IEEE
  International Conference on (1990) 566--571 vol.1\href
  {http://dx.doi.org/10.1109/ROBOT.1990.126041}
  {\path{doi:10.1109/ROBOT.1990.126041}}.

\bibitem{CMR08}
I.~Chattopadhyay, G.~Mallapragada, A.~Ray, $\nu^{\star}:$ a robot path planning
  algorithm based on renormalized measure of probabilistic regular languages,
  International Journal of Control 82~(5) (2008) 849--867.

\bibitem{CR06}
I.~Chattopadhyay, A.~Ray, Renormalized measure of regular languages, Int. J.
  Control 79~(9) (2006) 1107--1117.

\bibitem{CR07}
I.~Chattopadhyay, A.~Ray, Language-measure-theoretic optimal control of
  probabilistic finite-state systems, Int. J. Control.

\bibitem{JTCL05}
J.~M. O'kane, B.~Tovar, P.~Cheng, S.~M. Lavalle, Algorithms for planning under
  uncertainty in prediction and sensing, in: Chapter 18 in Autonomous Mobile
  Robots: Sensing, Control, Decision-Making, and Applications, Marcel Dekker,
  2005, pp. 501--547.

\bibitem{LP84}
T.~Lozano-Perez, M.~T. Mason, R.~H. Taylor, {Automatic Synthesis of Fine-Motion
  Strategies for Robots}, The International Journal of Robotics Research 3~(1)
  (1984) 3--24.
\newblock \href {http://dx.doi.org/10.1177/027836498400300101}
  {\path{doi:10.1177/027836498400300101}}.

\bibitem{ll92}
A.~Lazanas, J.~Latombe, Landmark-based robot navigation, Vol.~92, AAAI Press,
  San Jose, California, 1992, pp. 816--822.

\bibitem{fm98}
T.~Fraichard, R.~Mermond, Path planning with uncertainty for car- like robots,
  in Proc. of the IEEE Intl. conf. on Robotics \& Automation (1998) 27--32.

\bibitem{Ni80}
N.~J. Nilsson, Principles of Artificial Intelligence, Tioga, 1980.

\bibitem{tl92}
H.~Takeda, J.-C. Latombe, Sensory uncertainty field for mobile robot
  navigation, in Proc. of the IEEE Intl. conf. on Robotics \& Automation (1992)
  2465--2472.

\bibitem{tk96}
P.~E. Trahanias, Y.~Komninos, Robot motion planning: Multi- sensory uncertainty
  fields enhanced with obstacle avoidance, in: Proc. of the IEEE/RSJ Intl.
  conf. on Intelligent Robots and Systems, 1996.

\bibitem{vt98}
N.~A. Vlassis, P.~Tsanakas, A sensory uncertainty field model for unknown and
  non-stationary mobile robot environments, in: Proceedings of the IEEE Intl.
  conf. on Robotics \& Automation, 1998.

\bibitem{rt99}
N.~Roy, S.~Thrun, Coastal navigation with mobile robots, in: Advances in Neural
  Information Processing, Systems (NIPS, 1999.

\bibitem{as94}
R.~Alami, T.~Simeon, Planning robust motion strategies for a mobile robot, in:
  Proc. of the IEEE Intl. conf. on Robotics \& Automation, 1994.

\bibitem{bsa95}
B.~Bouilly, T.~Simeon, R.~Alami, A numerical technique for plan- ning motion
  strategies of a mobile robot in presence of uncertainty, in: Proc. of the
  IEEE Intl. conf. on Robotics \& Automation, 1995.

\bibitem{kbsc97}
M.~Khatib, B.~Bouilly, T.~Simeon, R.~Chatila, Indoor navigation with
  uncertainty using sensor-based motions, in Proc. of the IEEE Intl. conf. on
  Robotics \& Automation 4 (1997) 3379--3384.

\bibitem{BF95}
J.~Barraquand, P.~Ferbach, Motion planning with uncertainty: The information
  space approach, in: Proc. of the IEEE Intl. conf. on Robotics \& Automation,
  1995.

\bibitem{PS95}
L.~A. Page, A.~C. Sanderson, Robot motion planning for sensor-based control
  with uncertainties, Vol.~2, Nagoya, Japan, 1995, pp. 1333--1340.

\bibitem{blw06}
L.~Blackmore, H.~Li, B.~Williams, A probabilistic approach to optimal robust
  path planning with obstacles, in: Proceedings of the AIAA Guidance,
  Navigation and Control ConferenceNavigation and Control Conference, 2006.

\bibitem{b06}
L.~Blackmore, A probabilistic particle control approach to optimal, robust
  predictive control, in: Proceedings of the AIAA Guidance, Navigation and
  Control ConferenceNavigation and Control Conference, 2006.

\bibitem{lf00}
A.~Lambert, N.~L. Fort-Piat, Safe task planning integrating uncertainties and
  local maps federations, International Journal of Robotics Research, volume 19
  (2000) 597--611.

\bibitem{LG03}
A.~Lambert, D.~Gruyer, Safe path planning in an uncertain-configuration space,
  in: Robotics and Automation, 2003. Proceedings. ICRA '03. IEEE International
  Conference on, Vol.~3, 2003, pp. 4185--4190.
\newblock \href {http://dx.doi.org/10.1109/ROBOT.2003.1242246}
  {\path{doi:10.1109/ROBOT.2003.1242246}}.

\bibitem{gs05}
J.~P. Gonzalez, A.~T. Stentz, Planning with uncertainty in position: An optimal
  and efficient planner, in Proc. of the IEEE/RSJ Intl. conf. on Intelligent
  Robots and Systems (2005) 2435--2442.

\bibitem{gs07}
J.~P. Gonzalez, A.~Stentz, Planning with uncertainty in position using
  high-resolution maps, in: Proc, of the IEEE Intl. conf. on Robotics \&
  Automation, Rome, Italy, 2007.

\bibitem{ASG07}
R.~Alterovitz, T.~Sim{\'e}on, K.~Y. Goldberg, The stochastic motion roadmap: A
  sampling framework for planning with markov motion uncertainty, in: Robotics:
  Science and Systems, 2007.

\bibitem{SS08}
P.~Singla, T.~Singh, A novel coordinate transformation for obstacle avoidance
  and optimal trajectory planning, in: 2008 AAS/AIAA Astrodynamics Specialist
  Conference and Exhibit, 2008.

\bibitem{R05}
A.~Ray, Signed real measure of regular languages for discrete-event supervisory
  control, Int. J. Control 78~(12) (2005) 949--967.

\bibitem{G92}
V.~Garg, An algebraic approach to modeling probabilistic discrete event
  systems, Proceedings of 1992 IEEE Conference on Decision and Control (Tucson,
  AZ, December 1992) 2348--2353.

\bibitem{G92-2}
V.~Garg, Probabilistic lnaguages for modeling of \textsc{DED}s, Proceedings of
  1992 IEEE Conference on Information and Sciences (Princeton, NJ, March 1992)
  198--203.

\bibitem{R88}
W.~Rudin, Real and Complex Analysis, 3rd ed., McGraw Hill, New York, 1988.

\bibitem{C-PhD}
I.~Chattopadhyay, Quantitative control of probabilistic discrete event systems,
  PhD Dissertation, Dept. of Mech. Engg. Pennsylvania State University,
  {http{://} etda.libraries.psu.edu {/} theses / approved / WorldWideIndex /
  ETD-1443}.

\bibitem{CR09p}
I.~Chattopadhyay, A.~Ray, Optimal control of infinite horizon partially
  observable decision processes modeled as generatorsof probabilistic regular
  languages, International Journal of Control In Press.

\end{thebibliography}
%##############################################################################
%##############################################################################
%##############################################################################
\end{document}